\documentclass[10pt]{article} 

\usepackage[accepted]{rlj} 

\usepackage{amssymb}
\usepackage{mathtools}
\usepackage{mathrsfs}
\usepackage{graphicx}
\usepackage{subcaption}
\usepackage[space]{grffile}
\usepackage{url}
\usepackage{algorithm}
\usepackage{algorithmic}
\usepackage{multirow}
\usepackage{wrapfig}

\usepackage{amsthm}             
\theoremstyle{remark}

\theoremstyle{definition}
\newtheorem{example}{Example}
\newtheorem{definition}{Definition}
\newtheorem{assumption}{Assumption}
\newtheorem{theorem}{Theorem}
\newtheorem{proposition}{Proposition}

\usepackage[on]{editing}

\title{Scalable Causal Imitation Learning}

\setrunningtitle{Scalable Causal Imitation Learning}

\author{Eylam Tagor\textsuperscript{1}, Mingxuan Li\textsuperscript{1}, Elias Bareinboim\textsuperscript{1}}

\emails{eylam.tagor@columbia.edu, ml@cs.columbia.edu, eb@cs.columbia.edu}

\affiliations{
$^{1}$\textbf{Department of Computer Science, Columbia University}\\
\par
}

\contribution{
    We introduce Causal SQIL and Causal IQ-Learn, two CIL algorithms for continuous-control settings with unobserved confounders. Both algorithms apply state-of-the-art inverse soft Q-learning objectives to causally-adjusted state representations.
    }
    {
    Prior CIL algorithms \citep{zhang2020causal, kumor2021sequential, ruan2023causal, ruan2024partial} are only suitable for short-horizon, low-dimensional domains. Standard (non-causal) SQIL \citep{reddy2020sqil} and IQ-Learn \citep{garg2021iqlearn} assume no confounding.
    }

\contribution{
    We develop an approximation of the sequential $\pi$-backdoor criterion that reduces the adjustment set computation from the full horizon to a fixed-size sliding window, making causal adjustment tractable for environments with horizons of a thousand or more steps.
    }
    {
    The sequential $\pi$-backdoor criterion \citep{kumor2021sequential} provides a complete graphical characterization of imitability given knowledge of the causal diagram but is infeasible to compute and bloats dimensionality; our approximate adjustment is correct assuming bounded temporal influence and endogenous time-homogeneity.
    }

\contribution{
    We provide an empirical evaluation across confounded high-dimensional long-horizon control tasks, demonstrating that Causal SQIL and Causal IQ-Learn achieve competitive success rates on tasks where Causal BC and Causal GAIL degrade and all causally unaware methods fail to learn meaningful policies.
    }
    {
    These environments are augmentations of OGBench \citep{park2025ogbench} tasks with latent confounders and partial observability modeled by structural causal models to mimic real-world imitation tasks. Prior empirical evaluations of CIL use short-horizon discrete MDPs or low-dimensional continuous environments.
    }

\keywords{Causality, Imitation Learning, Inverse Reinforcement learning, Unobserved Confounders}

\summary{Imitation learning enables learning a policy in an unknown environment with a latent reward signal using expert demonstrations, but it struggles when the imitator's and expert's observations are mismatched and unobserved confounders are present in expert demonstrations. By identifying appropriate adjustment sets via the sequential $\pi$-backdoor criterion, causal imitation learning (CIL) provides a framework for approximating the expert's policy from confounded data. However, existing CIL methods, Causal Behavioral Cloning (Causal BC) and Causal Generative Adversarial Imitation Learning (Causal GAIL), are designed for short-horizon, low-dimensional settings. When applied to continuous control tasks with long horizons and high-dimensional state-action spaces, these methods exhibit poor performance: Causal BC suffers from compounding errors, Causal GAIL is unstable and sample-inefficient, and sequential $\pi$-backdoor adjustment becomes impractical.

We introduce Causal Soft Q Imitation Learning (SQIL) and Causal Inverse soft-Q Learning (IQ-Learn), two off-policy causal imitation learning algorithms that combine the causal adjustment framework with state-of-the-art inverse reinforcement learning objectives. Both algorithms operate on causally-adjusted state representations produced by an efficient approximation of the sequential $\pi$-backdoor criterion, exploiting the causal structure of continuous control environments to reduce the full-horizon adjustment to a fixed-size sliding window. We evaluate all methods in a suite of confounded environments and find that Causal SQIL and Causal IQ-Learn substantially outperform prior CIL algorithms on long-horizon tasks, sometimes surpassing the expert, whereas all causally unaware imitation methods fail to learn meaningful behavior.}

\begin{document}

\makeCover
\maketitle

\begin{abstract}
Imitation learning enables learning a policy in an unknown environment with a latent reward signal using expert demonstrations, but it struggles when the imitator's and expert's observations are mismatched and unobserved confounders are present in expert demonstrations. By identifying appropriate adjustment sets via the sequential $\pi$-backdoor criterion, causal imitation learning (CIL) provides a framework for approximating the expert's policy from confounded data. However, existing CIL methods, Causal Behavioral Cloning (Causal BC) and Causal Generative Adversarial Imitation Learning (Causal GAIL), are designed for short-horizon, low-dimensional settings. When applied to continuous control tasks with long horizons and high-dimensional state-action spaces, these methods exhibit poor performance: Causal BC suffers from compounding errors, Causal GAIL is unstable and sample-inefficient, and sequential $\pi$-backdoor adjustment becomes impractical. We introduce Causal Soft Q Imitation Learning (SQIL) and Causal Inverse soft-Q Learning (IQ-Learn), two off-policy causal imitation learning algorithms that combine the causal adjustment framework with state-of-the-art inverse reinforcement learning objectives. Both algorithms operate on causally-adjusted state representations produced by an efficient approximation of the sequential $\pi$-backdoor criterion, exploiting the causal structure of continuous control environments to reduce the full-horizon adjustment to a fixed-size sliding window. We evaluate all methods in a suite of confounded environments and find that Causal SQIL and Causal IQ-Learn substantially outperform prior CIL algorithms on long-horizon tasks, sometimes surpassing the expert, whereas all causally unaware imitation methods fail to learn meaningful behavior.

Repository:~\href{https://github.com/CausalAILab/Scalable-Causal-Imitation-Learning}{https://github.com/CausalAILab/Scalable-Causal-Imitation-Learning}
\end{abstract}

\section{Introduction}
\label{sec:intro}

Imitation learning (IL) has become a central paradigm in robotics and control tasks as an alternative to reinforcement learning (RL) in domains where reward signals are unavailable, sparse, or difficult to engineer~\citep{osa2018, zare2023survey}. Rather than optimizing a task-specific reward, IL seeks to learn from demonstrations collected as state-action trajectories from an expert deployed in the environment. This framework underlies a large body of work in behavioral cloning, dataset aggregation, and inverse RL \citep{ross2011reduction, ziebart2008maxent, ho2016generative, fu2018learning, chi2023diffusion, zhao2023aloha}, and is widely used in offline RL and robotics applications \citep{levine2020offline, fu2021d4rl, prudencio2023survey, brohan2023rt2}.

Traditionally, IL methods assume that the expert and imitator operate with the same sensory capabilities, meaning every variable that the expert can observe is also observable to the imitator. Under this No Unobserved Confounders (NUC) assumption, the expert policy is identifiable from observational data and standard IL can recover it given sufficient demonstrations. However, realistic decision-making systems rarely satisfy NUC. In practice, experts often have access to additional sensors or context unavailable to the imitator, and latent conditions such as wind, friction, and payload may unpredictably change their distribution. These phenomena introduce unobserved confounding to the imitation task that may jointly affect the state transition, expert's action and rewards.

\begin{figure*}[t]
    \centering
    \begin{subfigure}[t]{0.38\textwidth}
        \centering
        \includegraphics[width=\linewidth]{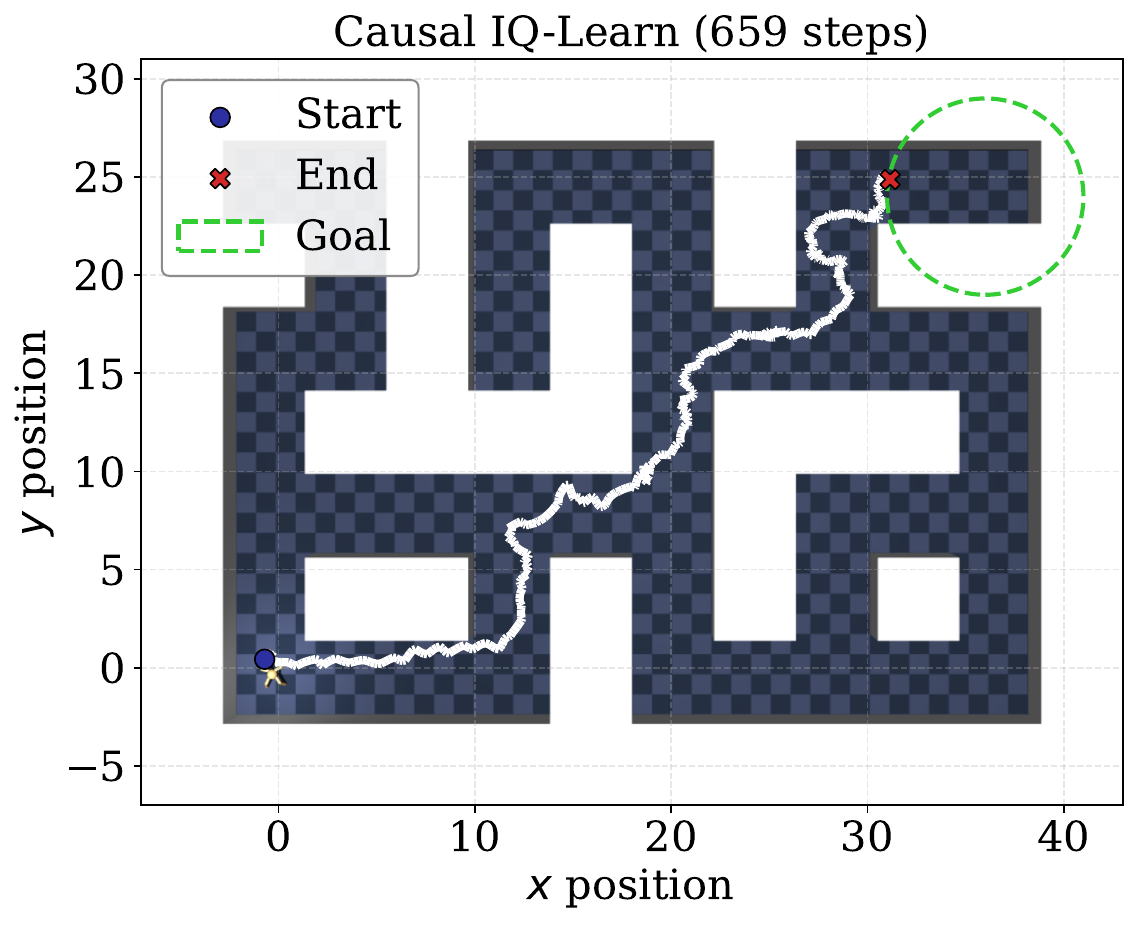}
        \label{fig:panel_a}
    \end{subfigure}
    \hspace{1.75cm}
    \begin{subfigure}[t]{0.38\textwidth}
        \centering
        \includegraphics[width=\linewidth]{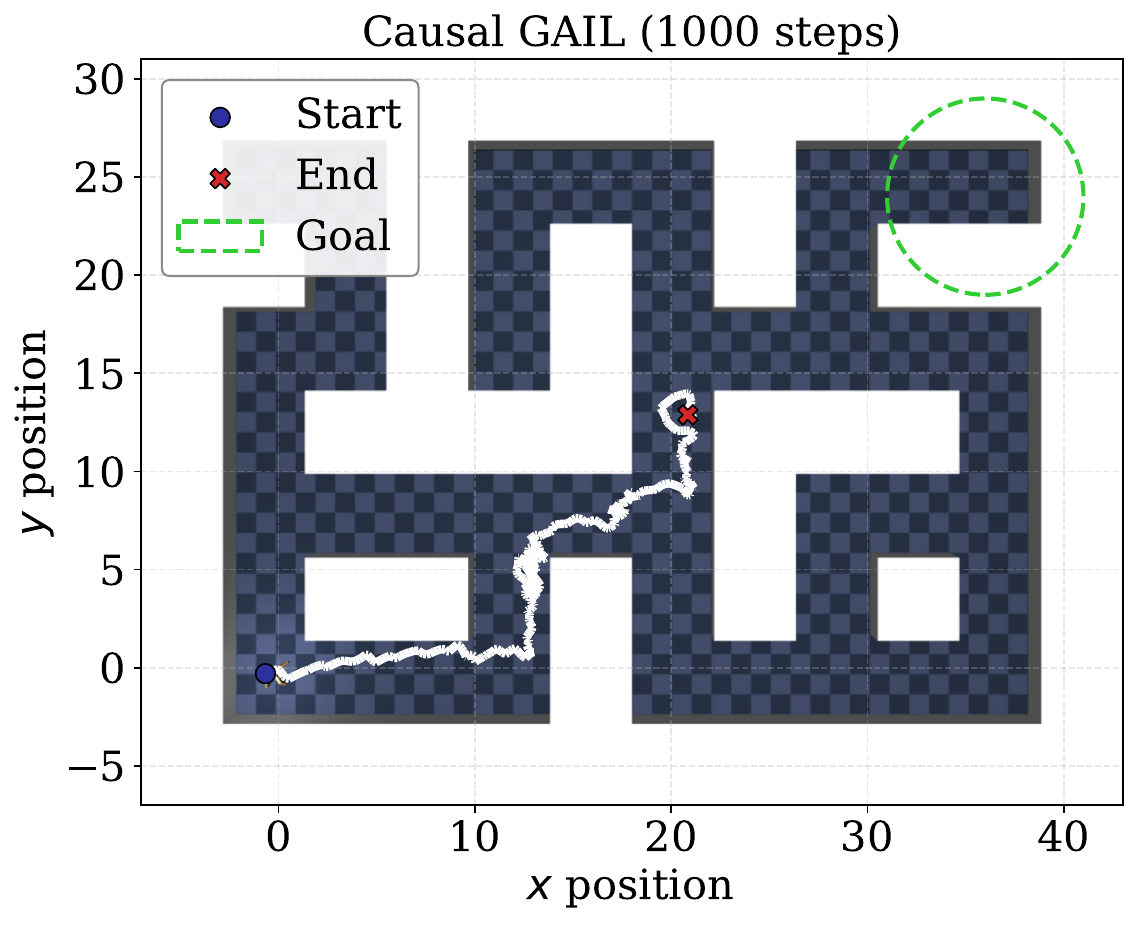}
        \label{fig:panel_b}
    \end{subfigure}
    \begin{subfigure}[c]{0.38\textwidth}
        \centering
        \includegraphics[width=\linewidth]{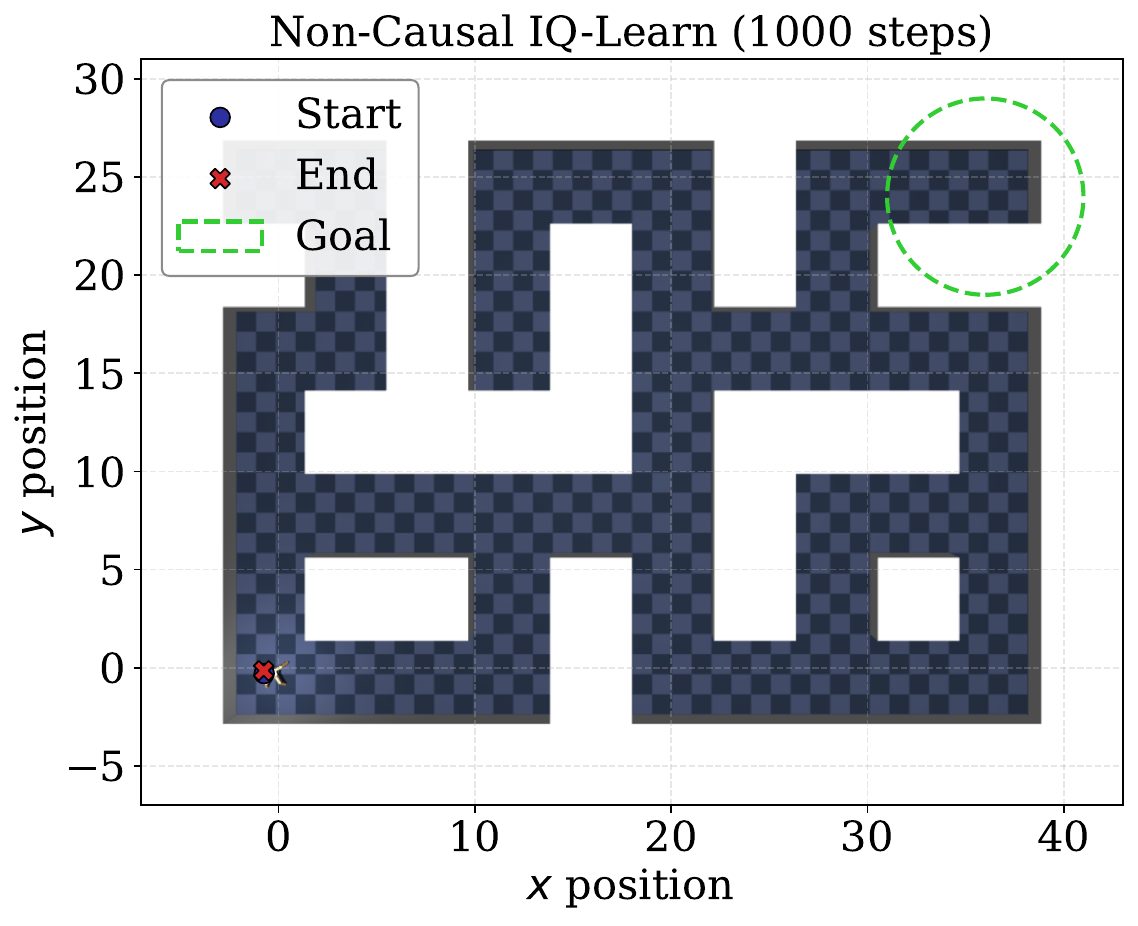}
        \label{fig:panel_c}
    \end{subfigure}
    \hspace{1cm}
    \begin{subfigure}[c]{0.46\textwidth}
        \centering
        \includegraphics[width=\linewidth]{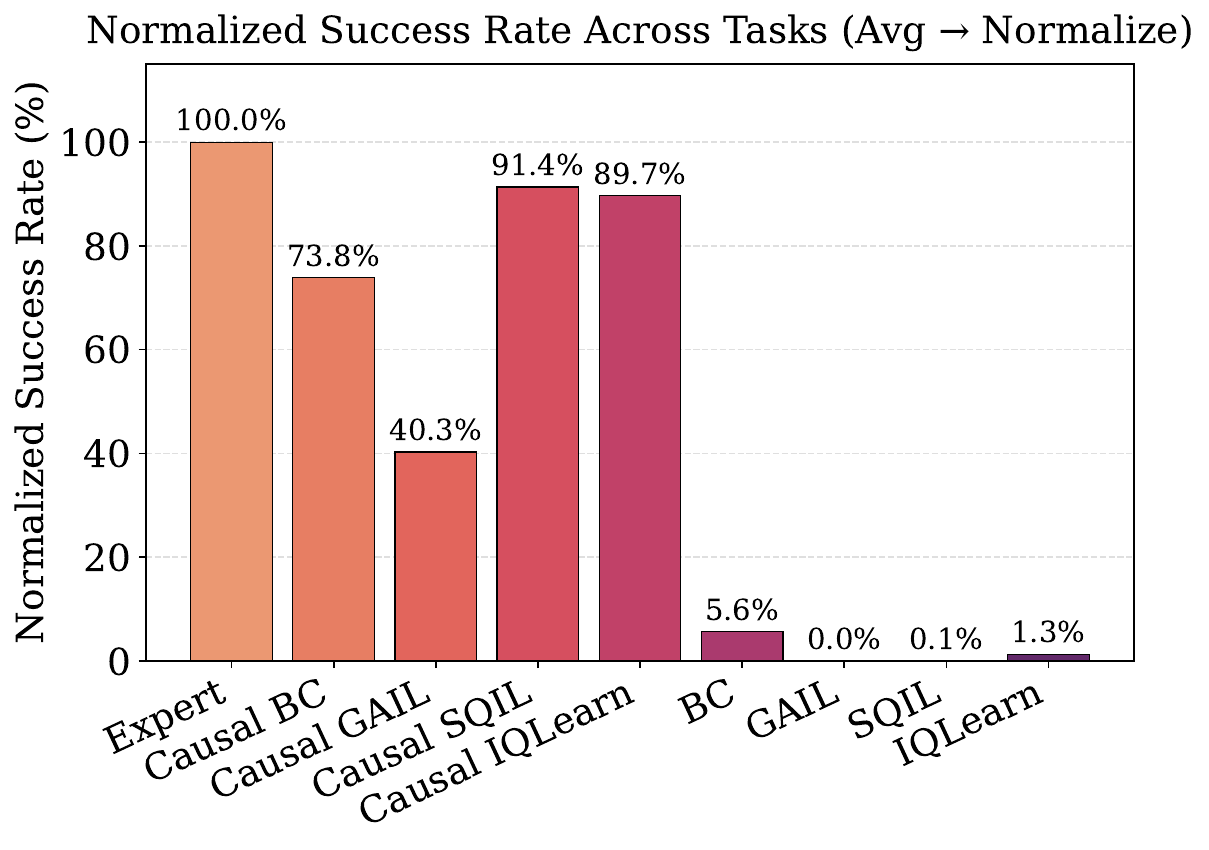}
        \label{fig:panel_d}
    \end{subfigure}
    \caption{Performance of various imitation learning algorithms on the Confounded AntMaze Large task. (a) One of our proposed algorithms, Causal IQ-Learn. (b) Existing CIL methods such as Causal GAIL struggle to scale to high-dimensional long-horizon tasks, even when equipped with our windowed adjustment (Section~\ref{algos-adjustment}) to bypass their infeasibility (see Appendix~\ref{app:wallclock}). (c) Non-causal methods fail by overfitting to spurious correlations corrupted by unobserved confounding. (d) Success rates for all algorithms tested aggregated across all tasks evaluated; Causal SQIL and Causal IQ-Learn perform best through leveraging causal knowledge and scalable policy learning.}
    \label{fig:four_panel}
\end{figure*}

By falsely assuming NUC or ignoring unobserved confounding, standard IL methods overfit to spurious correlations during training, thus failing to generalize when these correlations shift at runtime \citep{dehaan2019causal, lu2023imitation}. Figure~\ref{fig:four_panel} illustrates this on a confounded maze navigation task: the observation contains quantities that correlate with expert actions but, due to influence from unobserved confounders, suffer distribution shift at runtime. Causally unaware methods mistake these spurious correlations for causal signals during training and incorporate them into their decision-making, ultimately failing to exhibit coherent movement during runtime and achieve near-$0\%$ success (Figure~\ref{fig:four_panel}c,d). In safety-critical tasks in which IL is often applied due to its independence from reward engineering, such as autonomous driving \citep{chen2024endtoend, codevilla2019exploring} and robotic manipulation \citep{chi2023diffusion, zheng2024imitation}, such failures are only revealed during deployment and thus pose an unaffordable risk which severely limits the utility of IL.

Causal imitation learning (CIL) addresses this by leveraging structural causal knowledge to identify which observed variables are safe to condition on \citep{zhang2020causal, kumor2021sequential, ruan2023causal, ruan2024partial}. However, existing CIL methods have remained restricted to low-dimensional, short-horizon tasks: in continuous control benchmarks \citep{todorov2012mujoco, park2025ogbench} where observation spaces are high-dimensional and episodes span thousands of steps, they suffer from compounding error and training instability (Figure~\ref{fig:four_panel}b). On the same large maze task, these methods recover between $13-71\%$ of expert performance despite recovering $84-89\%$ on the medium maze (Table~\ref{tab:main-results}). A  review of the CIL and IL literature is provided in Appendix~\ref{app:related-work}.

Collectively, we identify the gap in the current state of imitation learning literature: scalable methods can imitate long-term expert policies, but are fragile when NUC is violated; causal methods are robust under confounding, but struggle to scale. Our contributions address this gap by introducing scalable CIL through soft Q-learning methods augmented by an efficient approximation to the sequential $\pi$-backdoor adjustment sets that leverages structural properties of the environment. For evaluation, we introduce a suite of confounded control tasks based on OGBench~\citep{park2025ogbench} where confounding biases are designed to mirror what is naturally found in real-world scenarios. We find that in these environments, existing CIL algorithms struggle to recover a consistent policy and non-causal IL algorithms fail altogether, whereas our proposed algorithms are able to achieve $~90\%$ of the expert's success rate on average (Figure~\ref{fig:four_panel}d) and even surpass it on some tasks.

\paragraph{Notations.}

We will consistently use capital letters ($X$) to represent variables either in the observation or in the causal diagram, and lowercase ($x$) for their values. We bold capital letters ($\mathbf{X}$) for sets of variables. We denote the parents of the variable $X$ in a causal diagram $\mathrm{pa}(X)$ and its children $\mathrm{ch}(X)$ (with $\mathrm{Pa}(X)$ and $\mathrm{Ch}(X)$ including $X$) and the C-component of $X$ in $\mathcal{G}$ as $\mathbf{C}(X)$. For a topological order $\prec$ and the subgraph $\mathcal{G}(X)$ induced by $X$ and its predecessors, the extended parent set $\mathrm{pa}^+(X) = \mathrm{pa}(\mathbf{C}(X)) \setminus \{X\}$; $\mathrm{pa}^+(X)$ contains every variable collider-connected to $X$ within $\mathcal{G}(X)$ together with the parents of those variables. We write $\mathrm{ch}^+(X)$ for the effective children of $X$, or every variable reachable from $X$ by a directed path whose internal nodes are all latent. $X_t$ and $X_H$ denote the instance of $\mathbf{X}$ at timestep $t$ and at the last step, respectively. We use $P(X)$ as the probability distribution over $X$, $\pi(X \mid \mathbf{Z})$ as the behavioral policy distribution conditioning on $\mathbf{Z}$, and $do(x)$ as the intervention fixing $X$ to take values $x$.

\section{The Challenge of Imitation under Unobserved Confounding}
\label{sec:cil}

We model the joint expert–environment system as a structural causal model (SCM).

\begin{definition}[Structural Causal Model \citep{pearl2009causality, bareinboim2022pearl}]
\label{def:scm}
A structural causal model (SCM) is a tuple $\mathcal{M} = \langle \mathbf{U}, \mathbf{V}, \mathscr{F}, P(\mathbf{u}) \rangle$, where:
\begin{enumerate}[label=\textbullet, leftmargin=20pt, topsep=0pt, parsep=0pt, itemsep=1pt]
    \item $\mathbf{U}$ is a set of exogenous variables determined by factors outside the model;
    \item $\mathbf{V} = \{V_1, \ldots, V_n\}$ is a set of endogenous variables determined by variables in $\mathbf{U} \cup \mathbf{V}$;
    \item $\mathscr{F} = \{f_{V} : V \in \mathbf{V}\}$ is a set of structural functions such that each $V \in \mathbf{V}$ is assigned via $V \leftarrow f_V(\mathrm{pa}(V), U_V)$, where $\mathrm{pa}(V) \subseteq \mathbf{V} \setminus \{V\}$ are endogenous parents of $V$ and $U_V \subseteq \mathbf{U}$;
    \item $P(\mathbf{u})$ is a joint probability distribution over the exogenous variables $\mathbf{U}$.
\end{enumerate}
Each SCM $\mathcal{M}$ induces a causal diagram $\mathcal{G}$ with one node for each $V \in \mathbf{V}$, directed edges $\mathrm{pa}(V)~\to~V$, and bidirected edges between variables that share an unobserved parent in $\mathbf{U}$.
\end{definition}

In our setting, $\mathbf{V}$ includes states, actions, and latent environment variables, $\mathbf{X} \subseteq \mathbf{V}$ is the action set, and $Y \in \mathbf{U}$ is the latent reward. To differentiate between variables that are unobserved, observed, and observed by the expert only, we partition the endogenous variables into
\[
\mathbf{V}^O \subseteq \mathbf{V} \quad\text{(observed to the imitator)}, \qquad
\mathbf{V}^L = \mathbf{V} \setminus \mathbf{V}^O \quad\text{(latent to the imitator)}.
\]

The expert demonstrations reflect the joint observational distribution $P(\mathbf{V}^O)$, whereas the imitator, operating under its own policy, induces the interventional distribution $P(\mathbf{V} \mid do(\pi))$. In the presence of latent variables $\mathbf{V}^L$, these two distributions may differ substantially: correlations between observed variables and expert actions may be driven by unobserved confounders rather than causal paths. The goal of CIL is therefore to determine, for each time step $t$, the sufficient subset of observed variables $\mathbf{Z}_t \subseteq \mathbf{V}^O$ for constructing an unbiased approximation of the expert's decision mechanism, $\pi_t(x_t \mid \mathbf{Z}_t) \approx P(x_t \mid \mathbf{Z}_t),$ in a way that is stable to the removal of latent confounding.

\begin{figure}[t]
    \centering
    \includegraphics[width=0.75\linewidth]{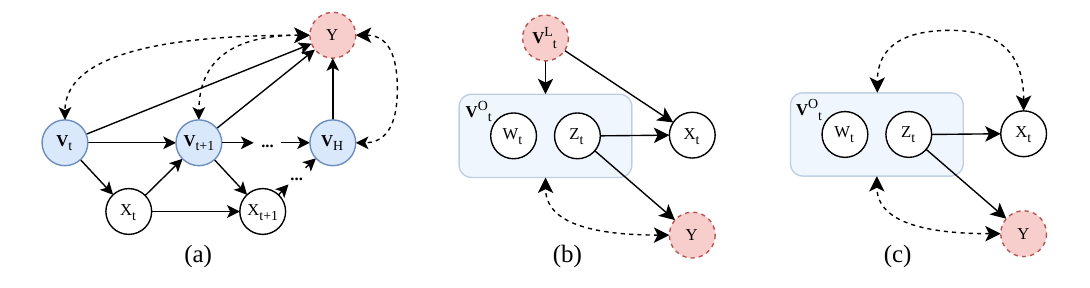}
    \caption{(a) Generic causal diagram for sequential imitation learning with unobserved confounders, as perceived by the expert. (b) Expanded single-timestep slice of a diagram with a spurious variable $W_t$, a useful state variable $Z_t$, and latent state $\mathbf{V}^L_t$. (c) Imitator perception of diagram (b), demonstrating how latent observation variables become additional unobserved confounders.}
    \label{fig:generic-causal-graph}
\end{figure}

\begin{example}[Confounded AntMaze]
\label{ex:antmaze-scm}
Consider an ant robot navigating a maze toward a goal region, receiving a terminal reward $Y$ upon success (see Figure~\ref{fig:four_panel} for visualization). The expert observes the full state $\mathbf{V}$, including its torso orientation $\mathbf{O}$, and selects joint torques $\mathbf{X}$ accordingly to move itself; Figure~\ref{fig:generic-causal-graph}a shows this sequential structure. The imitator, however, does not observe $\mathbf{O}$ (i.e.\ $\mathbf{O} \in \mathbf{V}^L$). A compass sensor $\mathbf{W}$ serves as a noisy surrogate for the ant's bearing. The environment is also subject to a latent wind $\mathbf{U}$ (shown implicitly through bidirectional edges) that applies an external force to the dynamics, affecting the compass reading and the difficulty of reaching the goal.

Figure~\ref{fig:generic-causal-graph}b shows a single-timestep slice of this structure. The compass $W_t$ receives incoming edges from both the hidden orientation $O_t \in \mathbf{V}^L_t$ and the latent wind $U_t$, but has no outgoing edges: it is a collider that plays no causal role in determining future states, actions, or rewards. The useful state variables (position, joint angles, velocities) are represented by $Z_t$, which does causally influence the expert's action due to its edge into $X_t$. From the imitator's perspective (Figure~\ref{fig:generic-causal-graph}c), the hidden $\mathbf{V}^L_t$ becomes an additional unobserved confounder, introducing a bidirected edge between $\mathbf{V}^O_t$ and $X_t$.

A causally unaware imitator that conditions on all observed variables, including $W_t$, inadvertently opens the spurious path $Y \gets U_t \to W_t \gets O_t \to X_t$. During training, the compass correlates with the expert's turning behavior because both are influenced by the wind. The imitator mistakes this for a causal signal: it learns, for instance, that when $W_t$ points east the expert turns right, not realizing both facts are driven by an eastward gust. Once deployed under a different $P(U)$, these associations become actively harmful: the imitator turns into walls whenever the wind changes direction. $\hfill \blacksquare$
\end{example}

\subsection{Sequential $\pi$-Backdoor Criterion}
\label{cil-pibd}

Formally,  the sequential $\pi$-backdoor criterion graphically determines what each $\mathbf{Z}_t$ must contain so that conditioning on $\mathbf{Z}_t$ blocks all noncausal paths from $X_t$ to the final outcome $Y$.

\begin{definition}[Sequential $\pi$-Backdoor Criterion \citep{kumor2021sequential}]
\label{def:seq-pi-backdoor}
Let $\mathcal{G}$ be the causal diagram induced by the SCM.
For each action $X_t$, define a manipulated graph $\mathcal{G}_t^{'}$ obtained by: (i) removing all incoming edges into future actions $\mathbf{X}_{t+1:H}$, and (ii) replacing each future action $X_j$ ($j>t$) by a node whose parents are restricted to $\mathbf{Z}_j$. A family of sets $\{\mathbf{Z}_t\}_{t=0}^{H}$ satisfies the sequential $\pi$-backdoor for $(\mathcal{G}, \mathbf{X}, Y)$ if, for every $t$, either $(X_t \perp\!\!\!\perp Y \mid \mathbf{Z}_t)_{(\mathcal{G}_t^{'})_{X_t}} \quad\text{or}\quad X_t \notin \mathrm{An}_{{\mathcal{G}_t^{'}}}(Y).$ Here $(\mathcal{G}_t^{'})_{X_t}$ denotes the graph obtained from $\mathcal{G}_t^{'}$ by deleting outgoing edges from $X_t$.
\end{definition}

When $\{\mathbf{Z}_t\}$ satisfies Definition~\ref{def:seq-pi-backdoor}, conditioning on $\mathbf{Z}_t$ removes all confounding and noncausal dependencies between $X_t$ and $Y$ that arise from shared latent parents in $\mathbf{V}^L$, proxy variables, or unobserved factors. Crucially, $\mathbf{Z}_t$ is restricted to the observable set $\mathbf{V}^O$. Returning to Example~\ref{ex:antmaze-scm} and Figure~\ref{fig:generic-causal-graph}, applying the sequential $\pi$-backdoor to the diagram yields an adjustment set $\mathbf{Z}_t$ that contains the useful state variables (position, joint angles, velocities) but excludes the compass $W_t$. By conditioning only on $\mathbf{Z}_t$, the imitator's policy $\pi(x_t \mid \mathbf{z}_t)$ is indifferent to the wind-driven distributional shift in $W_t$ and instead relies exclusively on variables that causally determine the expert's actions. Even when the orientation $O_t$ is unobserved and the full imitability condition is broken, learning a policy over $\mathbf{Z}_t \subseteq \mathbf{V}^O$ can still approximate the expert behavior because the remaining observed variables carry sufficient causal signal for navigation.

\section{Scalable Causal Imitation Learning}
\label{sec:algos}

The sequential $\pi$-backdoor criterion identifies the correct conditioning set, but correct adjustment sets alone are not sufficient for imitation at scale. Prior CIL work has instantiated the criterion within two algorithmic paradigms: Causal Behavioral Cloning (Causal BC), or supervised cloning of the expert's conditional policy on the adjustment sets $\hat{\pi}_t(x_t \mid \mathbf{z}_t) = P(X_t \mid \mathbf{Z}_t)$ \citep{kumor2021sequential}, and Causal Generative Adversarial Imitation Learning (Causal GAIL), which matches occupancy measures over the adjustment sets via adversarial training \citep{ruan2023causal}. Both exhibit fundamental scaling limitations, which we illustrate on the Confounded AntMaze task from Example~\ref{ex:antmaze-scm}.

\paragraph{Failure of Causal BC in long-horizon tasks.}
Even when supplied with correct adjustment sets, Causal BC remains a simple supervised learner. In long-horizon tasks, small prediction errors compound and the policy drifts into states not covered by the expert demonstrations, where its predictions are unreliable. While Causal BC achieves a respectable $88.9\%$ normalized to the expert in AntMaze-Medium ($H \approx 300$ effective steps for successful solves), it drops to $71.2\%$ in AntMaze-Large ($H \approx 700$ for successful episodes) and to $30.8\%$ in HumanoidMaze-Medium ($H > 1000$ for successful episodes). This degradation is shown in Figure~\ref{fig:failure-bc}, where Causal BC easily navigates a medium maze; on a large maze, however, the agent initially follows the expert path but gradually drifts off course due to compounding error, entering unseen states from which it cannot recover.

\begin{figure}[t]
    \centering
    \begin{subfigure}[m]{0.29\textwidth}
        \centering
        \includegraphics[width=\linewidth]{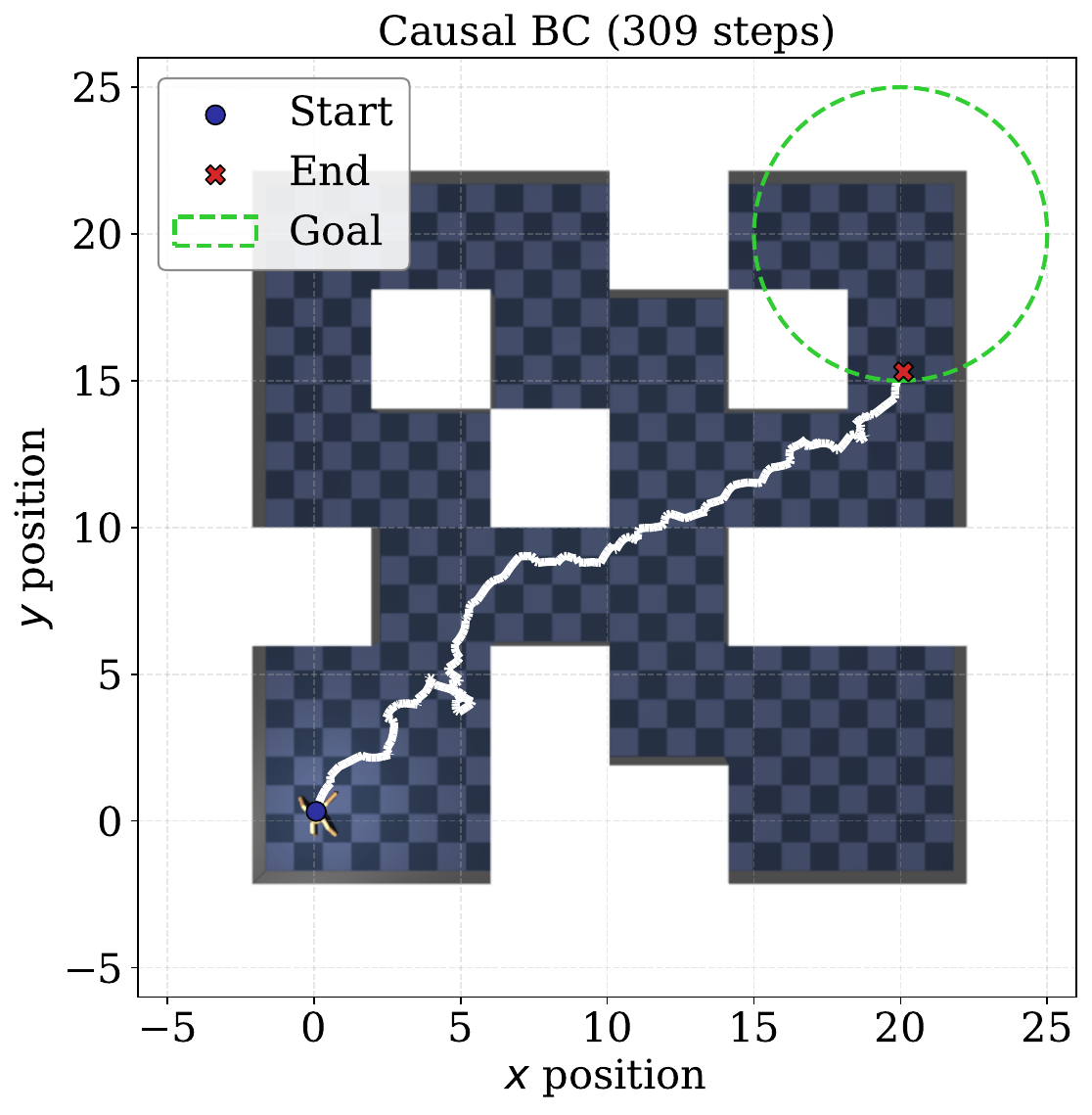}
        \label{fig:cbc-traj-med}
    \end{subfigure}
    \hspace{1.5cm}
    \begin{subfigure}[m]{0.38\textwidth}
        \centering
        \includegraphics[width=\linewidth]{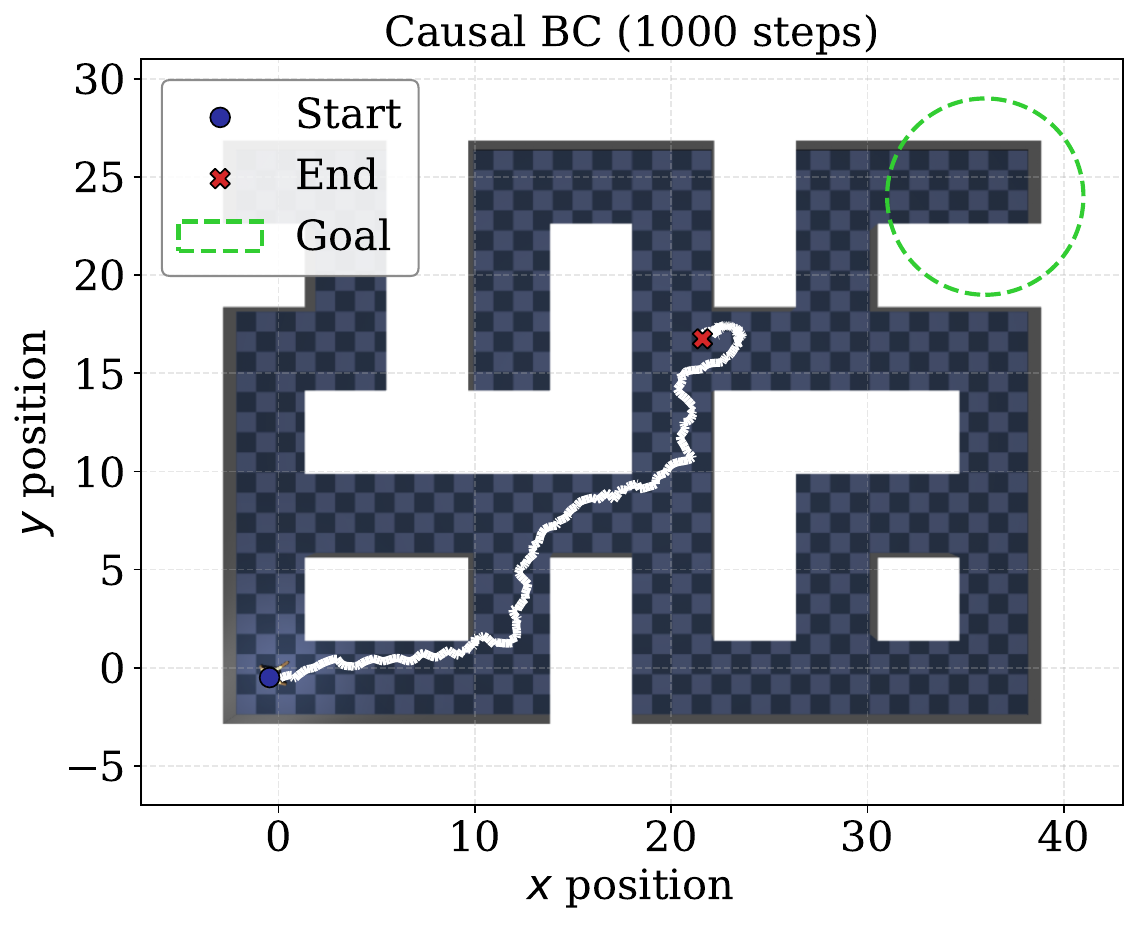}
        \label{fig:cbc-traj-large}
    \end{subfigure}

    \caption{Causal BC on Confounded AntMaze. On Medium (left), the agent reaches the goal. On Large (right), the agent follows the expert path initially but drifts off course partway through; once outside the demonstration support, it cannot recover and fails to reach the goal.}
    \label{fig:failure-bc}
\end{figure}

\paragraph{Failure of Causal GAIL in high-dimensional domains.}
While Causal GAIL addresses compounding error by matching occupancy measures via adversarial training, it introduces its own scaling difficulties through its reliance on on-policy rollouts (typically via PPO) and a discriminator to distinguish expert from imitator trajectories. In long-horizon tasks, the agent must discover the complete path to the goal through exploration before the discriminator can provide a useful learning signal for later portions of the trajectory. In AntMaze-Medium, this exploration is still somewhat feasible and Causal GAIL achieves $84.6\%$ success, but performance drops to $13.1\%$ in AntMaze-Large and collapses to $0\%$ by HumanoidMaze-Medium. Figure~\ref{fig:four_panel}b shows that even in AntMaze-Large, the agent learns only coherent navigation for roughly the first two-thirds of the maze. This indicates a credit-assignment failure, as on-policy training collects too few complete traversals to propagate reward signal to later stages of navigation and prevents scaling to long-horizon tasks.

These failure modes point to a clear algorithmic requirement: off-policy methods that treat expert demonstrations as a static buffer, learn from self-collected transitions, and propagate reward signal across long horizons via temporal-difference learning. SQIL~\citep{reddy2020sqil} and IQ-Learn~\citep{garg2021iqlearn}, two recent off-policy imitation learning methods built on soft Q-learning, satisfy all three properties but assume unconfounded environments. We now describe how to combine them with the causal adjustment framework.

\subsection{Causal SQIL and Causal IQ-Learn}
\label{algos-methods}

Given a $\pi$-backdoor admissible scope $\mathcal{S} = \{\langle X_t, \mathbf{Z}_t \rangle\}_{t=0}^{H-1}$, a straightforward causally-adjusted state representation would be $\mathbf{z}_t$ that concatenates the values of the variables in $\mathbf{Z}_t$ at each timestep. Both Causal SQIL and Causal IQ-Learn then operate on $(\mathbf{z}_t, x_t)$ pairs in place of the standard $(s_t, a_t)$ pairs, applying their respective objectives to the adjusted representation.

\paragraph{Causal SQIL.}
SQIL~\citep{reddy2020sqil} assigns a fixed reward of $r=1$ to expert transitions and $r=0$ to policy transitions, then trains an SAC agent on the combined replay buffer. We apply this to causally-adjusted inputs $(\mathbf{z}_t, x_t)$: the critic minimizes the soft Bellman residual
\begin{equation}
y = r + \gamma \Big(\min_{j=1,2} Q_{\bar{\theta}_j}(\mathbf{z}', a') - \alpha \log \pi_\phi(a' \mid \mathbf{z}')\Big), \qquad
\mathcal{L}_Q = \mathbb{E}\Big[\big(Q_\theta(\mathbf{z},x) - y\big)^2\Big],
\end{equation}
with $a' \sim \pi_\phi(\cdot \mid \mathbf{z}')$, and the actor maximizes the entropy-regularized objective
\begin{equation}
\mathcal{L}_\pi = \mathbb{E}_{\mathbf{z}}\Big[\alpha \log \pi_\phi(x \mid \mathbf{z}) - \min_{j=1,2} Q_{\theta_j}(\mathbf{z}, x)\Big], \quad x \sim \pi_\phi(\cdot \mid \mathbf{z}).
\end{equation}
Because the adjustment set $\mathbf{Z}_t$ satisfies the $\pi$-backdoor criterion, the Q-function learns value estimates that are not confounded by latent variables, while temporal-difference learning propagates the expert signal across the full horizon. Full pseudocode is given in Algorithm~\ref{alg:causal-sqil} (Appendix~\ref{app:pseudocode}).

\paragraph{Causal IQ-Learn.}
IQ-Learn~\citep{garg2021iqlearn} learns a Q-function whose implicit reward is consistent with expert behavior. The critic enforces the soft Bellman equation on expert data,
\begin{equation}
\mathcal{L}_{\mathrm{expert}} = \mathbb{E}_{(\mathbf{z},x,\mathbf{z}') \sim \mathcal{D}_{\mathrm{exp}}}
\Big[\big(Q_\theta(\mathbf{z}, x) - \gamma \, V_{\bar{\theta}}(\mathbf{z}')\big)^2\Big],
\end{equation}
where $V(\mathbf{z}) = \log \mathbb{E}_{x \sim \pi}[\exp Q(\mathbf{z},x)]$, and applies a policy-consistency regularizer on policy data,
\begin{equation}
\mathcal{L}_{\mathrm{reg}} = \mathbb{E}_{(\mathbf{z},x) \sim \pi}
\Big[\big(\log \pi_\phi(x \mid \mathbf{z}) - Q_\theta(\mathbf{z}, x) + V_\theta(\mathbf{z})\big)^2\Big].
\end{equation}
The combined loss $\mathcal{L}_Q = \mathcal{L}_{\mathrm{expert}} + \lambda \, \mathcal{L}_{\mathrm{reg}}$ trains the critic, grounding the implicit reward $r(\mathbf{z},x) = Q(\mathbf{z},x) - \gamma V(\mathbf{z}')$ in deconfounded state-action associations rather than spurious correlations. The actor uses the same entropy-regularized objective as Causal SQIL. Full pseudocode is given in Algorithm~\ref{alg:causal-iqlearn} (Appendix~\ref{app:pseudocode}).

The causal adjustment layer is algorithm agnostic. The downstream RL algorithm is unmodified; rather, it is the input representation that is changed. This property means that any future IL algorithm built on the soft Q learning approach can be made causal by the same procedure.

\subsection{Scalable Causal Adjustment}
\label{algos-adjustment}

Although the sequential $\pi$-backdoor yields a principled solution to the causal imitation problem, applying it swiftly becomes intractable as horizon and dimensionality increase. In a sequential decision-making problem with a terminal reward $Y$ and a horizon $H \geq 1000$, the \textsc{FindOX} algorithm of \citet{kumor2021sequential} (which returns the maximal admissible set $\mathbf{V}^O_X$ for sequential $\pi$-backdoor adjustment; reproduced in Appendix~\ref{app:pseudocode}) operates over thousands of nodes and results in $\left| \mathbf{Z}_t \right|$ that linearly increases with $t$. In an SCM where the full histories of $\mathbf{V}$ and $\mathbf{X}$ have at least a potential causal effect on $X_t$, adjustment is inefficient and bloats observational dimensionality.

To make adjustment feasible, we propose an approximation of the true adjustment sets by generalizing structural properties of confounded control environments. Our assumption is stated in terms of the extended parent set $\mathrm{pa}^+$, which rests on the notion of a confounded component.

\begin{definition}[Confounded Component \citep{tian2002ccomp}]
\label{def:ccomp}
Given a causal diagram $\mathcal{G}$ over variables $\mathbf{V}$, a confounded component (C-component) is a maximal set $\mathbf{C} \subseteq \mathbf{V}$ such that every pair of variables in $\mathbf{C}$ is joined by a path consisting entirely of bidirected edges. For $V \in \mathbf{V}$, $\mathbf{C}(V)$ denotes the C-component containing $V$.
\end{definition}

Fixing a topological order and letting $\mathcal{G}(V)$ be the subgraph induced by $V$ and its predecessors, recall that the extended parent set is $\mathrm{pa}^+(V) = \mathrm{pa}(\mathbf{C}(V)) \setminus \{V\}$, essentially comprising of all variables collider-connected to $V$ and their parents. The C-component generalizes the parent set to the semi-Markovian setting: under unobserved confounding a variable can remain dependent on a non-descendant joined to it by a bidirected edge even after conditioning on its directed parents, which the extended parents set accounts for. Throughout, $\mathrm{pa}^+$ is computed on the diagram's latent projection $\widetilde{\mathcal{G}}$ onto $\mathbf{V}^O \cup \{Y\}$ \citep{verma1990equivalence, tian2002ccomp}, in which latent-mediated paths appear as single edges (formal definition in Appendix~\ref{app:proof-windowed}). We use this to bound the temporal reach of confounding:

\begin{assumption}[$k$-Bounded Time-Homogeneous Confounding]
\label{assm:bounded-confounding}
Let $\mathcal{G}_H$ be the causal diagram induced by the SCM unrolled over horizon $H$, with endogenous variables $\mathbf{V}_t$ at each timestep $t$, and let $\widetilde{\mathcal{G}}_H$ be its latent projection onto $\mathbf{V}^O \cup \{Y\}$, with $Y$ treated as an element of $\mathbf{V}_H$:
\begin{enumerate}[label=(\roman*), leftmargin=20pt, topsep=2pt, parsep=0pt, itemsep=2pt]
    \item \textbf{$k$-bounded influence.} There exists a constant $k \geq 1$ such that for every $t$, every node in $\mathrm{pa}^+(\mathbf{V}_t)$ occurs at a timestep no earlier than $t - k$.
    \item \textbf{Endogenous time-homogeneity.} The structural functions $\mathscr{F}$ are identical at every timestep: for all $t, t'$ and for any fixed $j \leq k$, the local causal structure among $(\mathbf{V}_{t-j}, \ldots, \mathbf{V}_t, X_t)$ is isomorphic to that among $(\mathbf{V}_{t'-j}, \ldots, \mathbf{V}_{t'}, X_{t'})$.
\end{enumerate}
\end{assumption}

Condition~(i) ensures that all causal influence on $X_t$, whether from observed state variables or unobserved confounders, is fully captured within a window of length $k$ timesteps. Condition~(ii) ensures that this local causal structure applies to any timestep in the unrolled graph, such that computing adjustment sets for one $k$-length window yields sufficient information of the relevant causal relationships to apply for all timesteps. Both conditions of Assumption~\ref{assm:bounded-confounding} are naturally satisfied in physics-based continuous control environments such as MuJoCo, where dynamics depend solely on the immediate state and external forces from confounders have temporally localized effects.

\begin{algorithm}[t]
\caption{Windowed Sequential $\pi$-Backdoor Adjustment}
\label{alg:feasible-pi-backdoor}
\begin{algorithmic}[1]

\REQUIRE Window size $k$, full horizon $H$.

\STATE Construct proxy environment $\mathcal{E}_k$ with horizon $h=k+1$ and extract its causal graph $\mathcal{G}_k$.
\STATE Run \textsc{FindOX}$(\mathcal{G}_k, \mathbf{X}, Y)$ $\to$ $\mathbf{O}^X$. \textbf{If} $\mathbf{X} \not\subseteq \mathbf{O}^X$, \textbf{return} not imitable.
\STATE Compute ancestral graph $\mathcal{G}_k^Y$. Compute Markov boundary $\mathrm{MB} \leftarrow \mathrm{Pa}^{+}(\mathbf{C}(\mathrm{Ch}^{+}(\mathbf{O}^X))) \setminus \mathbf{O}^X$ and boundary actions $\mathrm{BA} \leftarrow \{X_i \in \mathbf{X} \cap \mathbf{O}^X \mid \mathrm{ch}^+(X_i) \not\subseteq \mathbf{O}^X\}$ in $\mathcal{G}_k^Y$.
\FOR{each action $X_t$ in $\mathcal{G}_k$}
    \IF{$\mathbf{X}_t \in \mathrm{BA}$}
        \STATE $\mathbf{Z}_t^k \leftarrow (\mathrm{MB} \cup \mathrm{BA}) \cap \mathrm{before}(X_t)$
    \ELSE
        \STATE $\mathbf{Z}_t^k \leftarrow \varnothing$ \hfill \textit{// $X_t \notin \mathrm{An}(Y)$ in $\mathcal{G}_t'$; satisfies condition (2) of Def.~\ref{def:seq-pi-backdoor}}
    \ENDIF
\ENDFOR
\FOR{$t = 0, \dots, H-1$}
    \STATE $\mathbf{Z}_t^H \leftarrow \{(v, \tau) \in \mathbf{Z}_t^k \mid \tau \ge t - k\}$ \hfill \textit{// Clip to window of width $k$}
\ENDFOR
\STATE Build sliding window specification $S$: for each observed variable type $V$ appearing in $\bigcup_t \mathbf{Z}_t^k$, enumerate lags $\{-1, \dots, -k\}$ with dimension $d_V$.

\STATE \textbf{return} $\{\mathbf{Z}_t^H\}_{t=0}^{H-1}$, $S$.

\end{algorithmic}
\end{algorithm}

Algorithm~\ref{alg:feasible-pi-backdoor} implements this approximation in two stages. The first stage (Lines~1--10) solves the exact sequential $\pi$-backdoor on a short-horizon proxy graph $\mathcal{G}_k$ with horizon $k{+}1$. \textsc{FindOX}~\citep{kumor2021sequential} identifies the maximal admissible set $\mathbf{O}^X$; if $\mathbf{X} \not\subseteq \mathbf{O}^X$, the problem is not imitable. The algorithm then constructs per-action adjustment sets from the Markov boundary $\mathrm{MB}$ of $\mathbf{O}^X$ in the ancestral graph $\mathcal{G}_k^Y$ (the minimal observed set blocking $\mathbf{O}^X$ from all other variables) and the boundary actions $\mathrm{BA}$ (actions whose causal effect on $Y$ persists regardless of future actions) \citep[Definition~3.1, Lemma~3.2]{kumor2021sequential}: boundary actions condition on $(\mathrm{MB} \cup \mathrm{BA}) \cap \mathrm{before}(X_t)$, while non-boundary actions require no conditioning. Intuitively, in Example~\ref{ex:antmaze-scm} and Figure~\ref{fig:generic-causal-graph}, the Markov boundary selects recent instances of $\mathbf{Z}$ and $\mathbf{X}$ within the $k$-step window while instances of $\mathbf{W}$, a collider with no outgoing edges, is excluded from $\mathbf{O}^X$ and thus from any adjustment set.

The second stage (Lines~11-15) transfers the proxy-graph adjustment sets to the full horizon. By time-homogeneity, the sets $\{\mathbf{Z}_t^k\}$ exhibit a repeating pattern of relative lags; clipping each to a window of $k$ steps reduces dimensionality from $O(H)$ to $O(k)$. The output is a fixed-dimensional sliding window $S$ that specifies adjustment set lags and dimensions for each timestep, used to encode inputs for the imitator. Its correctness is guaranteed by the following theorem (Proof in Appendix~\ref{app:proof-windowed}).

\begin{theorem}[Correctness of Windowed Adjustment]
\label{thm:windowed-adjustment}
Let $\mathcal{G}_H$ be the causal diagram of an SCM unrolled over horizon $H$, and let Assumption~\ref{assm:bounded-confounding} hold with window size $k$. Let $\{\mathbf{Z}_t^k\}_{t=0}^{k}$ be adjustment sets constructed by Algorithm~\ref{alg:feasible-pi-backdoor} that satisfy the sequential $\pi$-backdoor criterion (Definition~\ref{def:seq-pi-backdoor}) for the proxy graph $\mathcal{G}_k$ with horizon $k+1$. Then the transferred sets $\{\mathbf{Z}_t^H\}_{t=0}^{H-1}$ returned by Algorithm~\ref{alg:feasible-pi-backdoor} satisfy the sequential $\pi$-backdoor criterion for $(\mathcal{G}_H, \mathbf{X}, Y)$. 
\end{theorem}

With windowed adjustment sets $\{\mathbf{Z}_t^H\}_{t=0}^{H-1}$ and specification $S$ in hand, the causal encoding $\mathbf{z}_t$ is constructed at each timestep: for each observed variable type $V$ appearing in the adjustment sets, we concatenate its values at lags $\{-1, \dots, -k\}$ relative to $t$. This representation comprises the state observation in both Causal SQIL and Causal IQ-Learn, making the full procedure from causal graph analysis to policy optimization accessible and feasible for arbitrarily long horizons.

\section{Experiments}
\label{sec:experiments}

We evaluate the proposed algorithms on confounded continuous-control locomotion environments derived from OGBench \citep{park2025ogbench}. Each environment is defined as an SCM to support the modeling of unobserved confounders. In Confounded AntMaze ($H{=}1000$), an 8-DoF ant navigates a maze under latent wind; the imitator observes a wind-affected compass $\mathbf{W}$ in place of orientation $\mathbf{O}$ (see Example~\ref{ex:antmaze-scm}). In Confounded HumanoidMaze ($H{=}2000$), a 21-DoF humanoid navigates under latent seismic tremors; the imitator observes a tremor-affected vibration sensor $\mathbf{W}$ in place of the hidden center-of-mass velocity $\mathbf{C}$. In both environments, causal methods exclude $\mathbf{W}$ from the adjustment set, while causally unaware methods condition on it and thereby fail under distributional shift. Full environment details, causal graphs, and visualizations are provided in Appendix~\ref{app:envs}.

We compare eight algorithms total, four causal (Causal BC~\citep{kumor2021sequential}, Causal GAIL~\citep{ruan2023causal}, Causal SQIL (ours), Causal IQ-Learn (ours)) and four causally unaware (BC~\citep{ross2011reduction}, GAIL~\citep{ho2016generative}, SQIL~\citep{reddy2020sqil}, IQ-Learn~\citep{garg2021iqlearn}), to isolate the contributions of causal adjustment and algorithmic choice. Expert policies are constructed via offline-to-online RL (BC and TD3 fine-tuning); full implementation details are in Appendix~\ref{app:hparams}.

\subsection{Results}
\label{sec:experiments-results}

Causal variants use Algorithm~\ref{alg:feasible-pi-backdoor} to compute their per-timestep observation using windowed causal adjustment. Non-causal variants condition on the full $\mathbf{V}^O$ at each timestep. All hyperparameters remain the same between causal and non-causal variants of the same algorithm.

\begin{table}[t]
\vspace{-0.15in}
\centering
\caption{Evaluation results on confounded tasks. Normalized $\mathbb{E}[Y]$ linearly shifts the worst-performing algorithm to $0$ due to the purely negative reward. Best non-expert result per task in \textbf{bold}. Standard errors are provided; unnormalized data can be found in Appendix~\ref{app:raw-data}.}
\label{tab:main-results}
\resizebox{\textwidth}{!}{
\begin{tabular}{llccccccccc}
\hline
& & \textbf{Expert} & \textbf{C-BC} & \textbf{C-GAIL} & \textbf{C-SQIL} & \textbf{C-IQ-Learn} & \textbf{BC} & \textbf{GAIL} & \textbf{SQIL} & \textbf{IQ-Learn} \\
\hline
\multirow{2}{*}{AntMaze-Medium}
& Norm.\ $\mathbb{E}[Y]$  & $271.4 {\scriptstyle \pm 87.5}$ & $250.1 {\scriptstyle \pm 100.9}$ & $239.5 {\scriptstyle \pm 113.6}$ & $\mathbf{275.8} {\scriptstyle \pm 85.7}$ & $257.1 {\scriptstyle \pm 101.3}$ & $0.0 {\scriptstyle \pm 119.2}$ & $10.4 {\scriptstyle \pm 112.2}$ & $4.7 {\scriptstyle \pm 119.1}$ & $2.8 {\scriptstyle \pm 120.8}$ \\
& Success rate\ (\%)       & $87.6 {\scriptstyle \pm 1.0}$ & $77.9 {\scriptstyle \pm 1.3}$ & $74.1 {\scriptstyle \pm 1.4}$ & $\mathbf{90.7} {\scriptstyle \pm 0.9}$ & $84.3 {\scriptstyle \pm 1.2}$ & $0.0 {\scriptstyle \pm 0.0}$ & $0.0 {\scriptstyle \pm 0.0}$ & $0.0 {\scriptstyle \pm 0.0}$ & $0.0 {\scriptstyle \pm 0.0}$ \\
\hline
\multirow{2}{*}{AntMaze-Large}
& Norm.\ $\mathbb{E}[Y]$  & $229.0 {\scriptstyle \pm 122.8}$ & $199.3 {\scriptstyle \pm 121.6}$ & $155.6 {\scriptstyle \pm 122.6}$ & $204.3 {\scriptstyle \pm 125.9}$ & $\mathbf{225.6} {\scriptstyle \pm 127.9}$ & $14.0 {\scriptstyle \pm 150.2}$ & $85.8 {\scriptstyle \pm 132.9}$ & $0.0 {\scriptstyle \pm 151.5}$ & $16.1 {\scriptstyle \pm 153.5}$ \\
& Success rate\ (\%)       & $55.9 {\scriptstyle \pm 1.6}$ & $39.8 {\scriptstyle \pm 1.5}$ & $7.3 {\scriptstyle \pm 0.8}$ & $45.0 {\scriptstyle \pm 1.6}$ & $\mathbf{58.9} {\scriptstyle \pm 1.6}$ & $0.0 {\scriptstyle \pm 0.0}$ & $0.0 {\scriptstyle \pm 0.0}$ & $0.0 {\scriptstyle \pm 0.0}$ & $0.0 {\scriptstyle \pm 0.0}$ \\
\hline
\multirow{2}{*}{HumanoidMaze-Medium}
& Norm.\ $\mathbb{E}[Y]$  & $224.8 {\scriptstyle \pm 218.4}$ & $92.0 {\scriptstyle \pm 172.7}$ & $0.5 {\scriptstyle \pm 161.9}$ & $\mathbf{192.2} {\scriptstyle \pm 217.2}$ & $158.9 {\scriptstyle \pm 203.8}$ & $69.4 {\scriptstyle \pm 169.1}$ & $0.0 {\scriptstyle \pm 162.5}$ & $41.5 {\scriptstyle \pm 157.0}$ & $50.1 {\scriptstyle \pm 169.7}$ \\
& Success rate\ (\%)       & $33.8 {\scriptstyle \pm 1.5}$ & $10.4 {\scriptstyle \pm 1.0}$ & $0.0 {\scriptstyle \pm 0.0}$ & $\mathbf{24.7} {\scriptstyle \pm 1.4}$ & $19.1 {\scriptstyle \pm 1.2}$ & $5.4 {\scriptstyle \pm 0.7}$ & $0.0 {\scriptstyle \pm 0.0}$ & $0.1 {\scriptstyle \pm 0.1}$ & $2.4 {\scriptstyle \pm 0.5}$ \\
\hline
\multirow{2}{*}{HumanoidMaze-Large}
& Norm.\ $\mathbb{E}[Y]$  & $136.4 {\scriptstyle \pm 212.9}$ & $125.1 {\scriptstyle \pm 227.9}$ & $0.8 {\scriptstyle \pm 210.2}$ & $\mathbf{139.3} {\scriptstyle \pm 226.0}$ & $90.0 {\scriptstyle \pm 220.5}$ & $93.0 {\scriptstyle \pm 204.6}$ & $0.0 {\scriptstyle \pm 208.9}$ & $80.7 {\scriptstyle \pm 194.9}$ & $70.8 {\scriptstyle \pm 193.3}$ \\
& Success rate\ (\%)       & $7.0 {\scriptstyle \pm 0.8}$ & $\mathbf{8.0} {\scriptstyle \pm 2.7}$ & $0.0 {\scriptstyle \pm 0.0}$ & $\mathbf{8.0} {\scriptstyle \pm 2.7}$ & $3.0 {\scriptstyle \pm 1.7}$ & $5.0 {\scriptstyle \pm 2.2}$ & $0.0 {\scriptstyle \pm 0.0}$ & $0.0 {\scriptstyle \pm 0.0}$ & $0.0 {\scriptstyle \pm 0.0}$ \\
\hline
\end{tabular}
}
\end{table}

\paragraph{All non-causal methods fail catastrophically.}
As seen in Table~\ref{tab:main-results}, non-causal algorithms consistently fail across all environments: standard BC, GAIL, SQIL, and IQ-Learn achieve $0\%$ or near-$0\%$ success rates with low $\mathbb{E}[Y]$ across all tasks. This confirms that the fundamental failure is not a consequence of any particular learning paradigm but of the decision to condition on all observed variables, and no amount of temporal-difference learning or adversarial training can overcome a fundamentally misspecified conditioning set. The improvement of non-causal algorithms in HumanoidMaze tasks from AntMaze tasks, despite the former being more difficult, can be attributed to the less disruptive confounding effect of seismic tremors than wind fields on the state dynamics.

\paragraph{Causal adjustment is necessary but not sufficient for scaling.}
Although causal adjustment alone leads to significant improvements in most algorithms (Table~\ref{tab:main-results}), Causal BC and especially Causal GAIL see substantial performance decreases as the task horizon and dimensionality increases, with Causal GAIL collapsing to about the level of non-causal GAIL by the HumanoidMaze-Medium task. Causal SQIL and Causal IQ-Learn scale more gracefully despite increasing dimensionality and horizon, and at times surpassing the expert's performance.

\begin{figure}[t]
    \centering
    \begin{subfigure}[t]{0.24\textwidth}
        \centering
        \includegraphics[width=\linewidth]{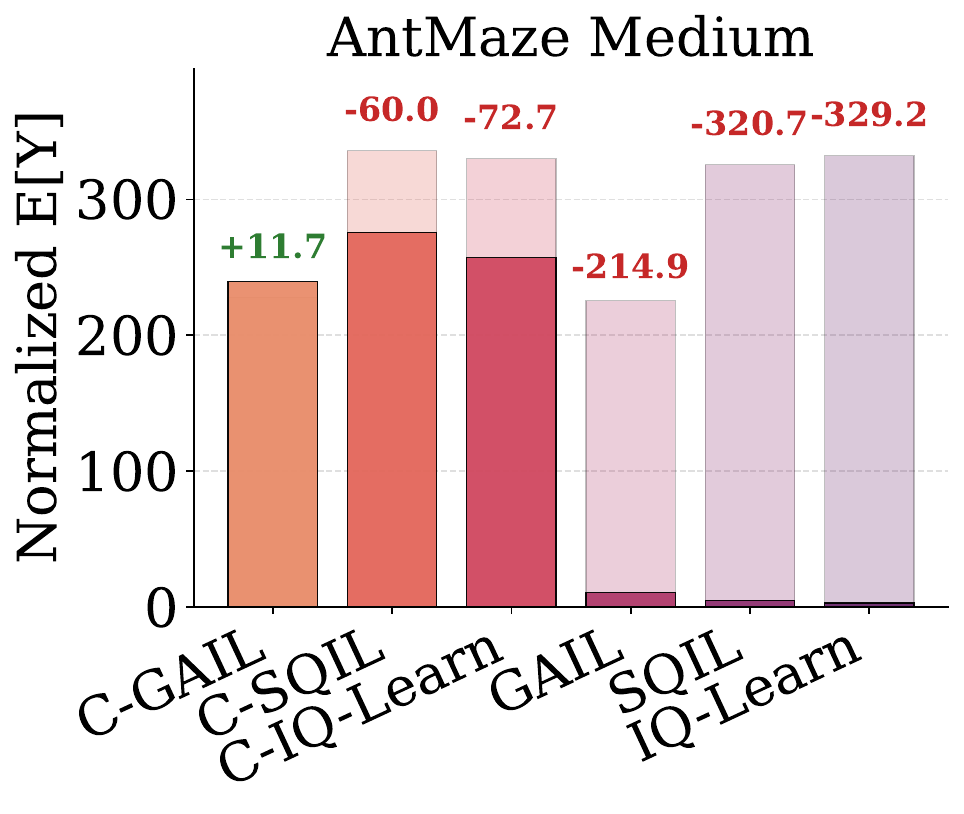}
    \end{subfigure}
    \hfill
    \begin{subfigure}[t]{0.24\textwidth}
        \centering
        \includegraphics[width=\linewidth]{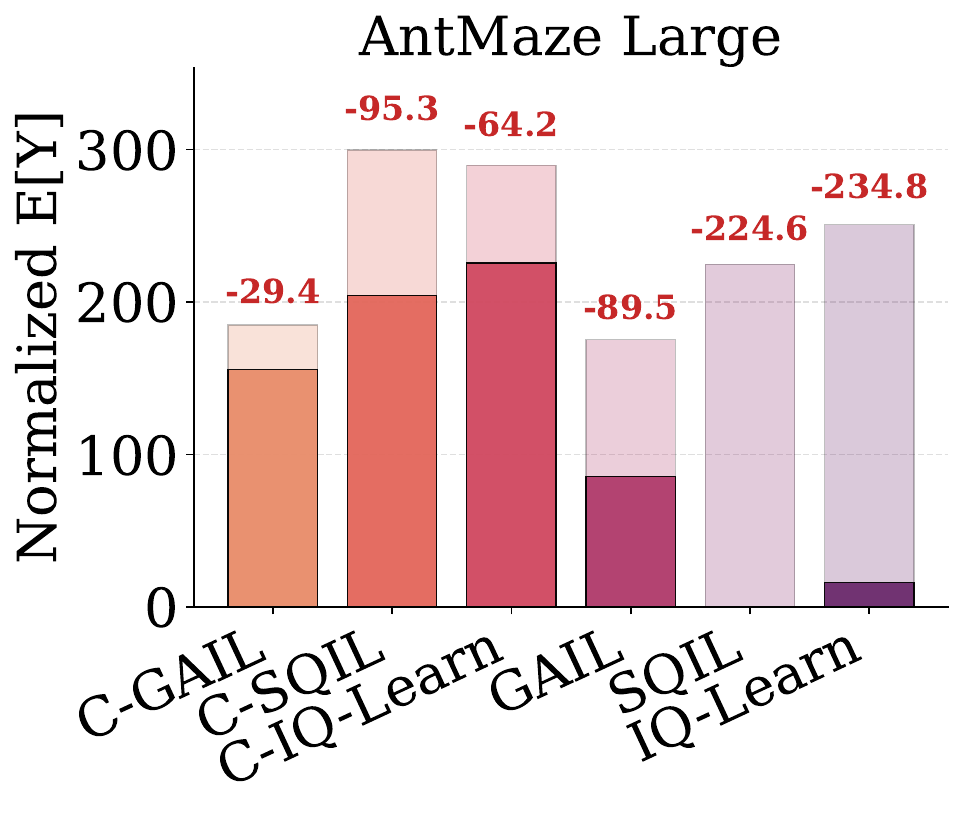}
    \end{subfigure}
    \hfill
    \begin{subfigure}[t]{0.24\textwidth}
        \centering
        \includegraphics[width=\linewidth]{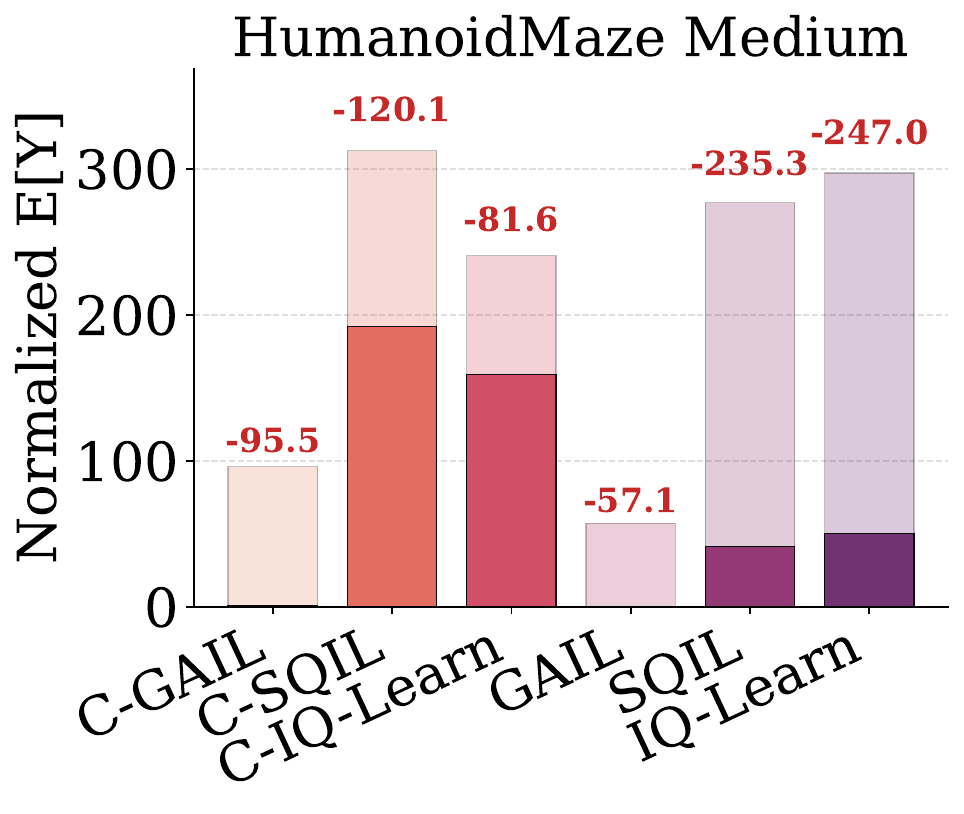}
    \end{subfigure}
    \hfill
    \begin{subfigure}[t]{0.24\textwidth}
        \centering
        \includegraphics[width=\linewidth]{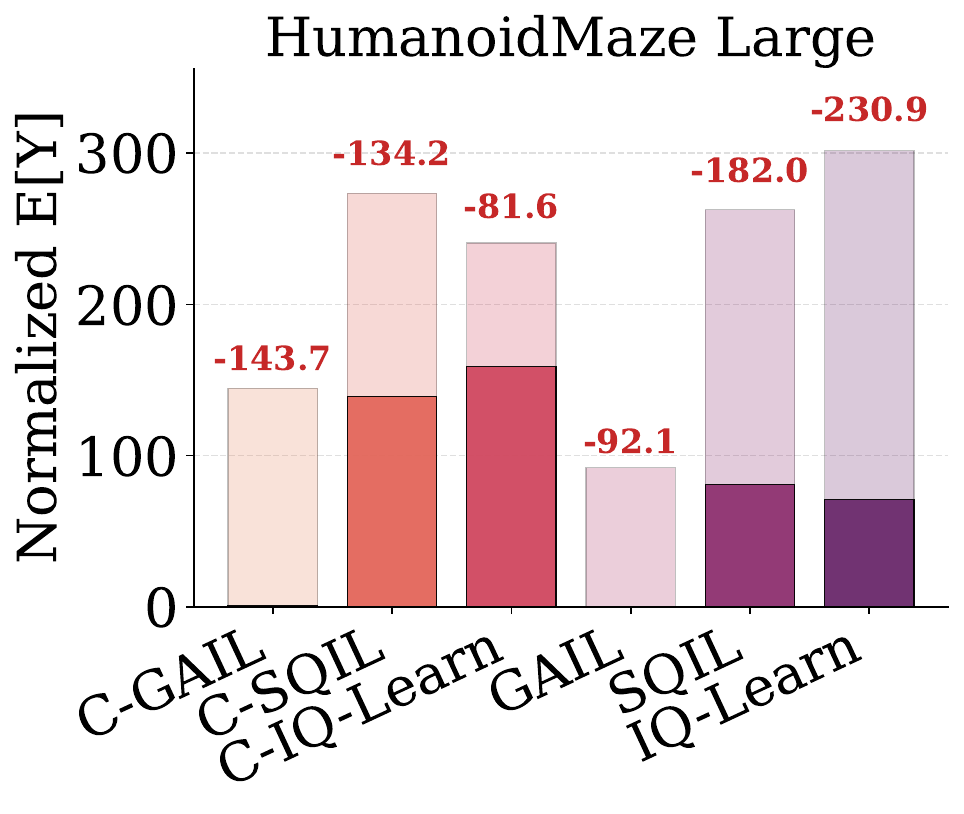}
    \end{subfigure}
    \caption{Evaluation return during training and runtime for GAIL, SQIL, and IQ-Learn (causal and non-causal variants). During training, both variants of each algorithm achieve comparable returns. The gap emerges only at runtime (Table~\ref{tab:main-results}) due to unobserved confounders causing distributional changes from where the expert demonstrations were collected, which affects only non-causal methods. Hyperparameters and other training metrics are provided in detail in Appendix~\ref{app:hparams}.}
    \label{fig:training-curves-main}
\end{figure}

\paragraph{Confounding cannot be revealed by in-distribution evaluation.}
Figure~\ref{fig:training-curves-main} reveals that during training, when the imitator operates under the same confounder distribution $P(\mathbf{U})$ as the expert, non-causal variants of each algorithm achieve nearly identical evaluation returns to causal variants, and in some cases appear better since conditioning on the spurious proxy $\mathbf{W}$ provides additional predictive signal that is useful when under the training $P(\mathbf{U)}$. At runtime, when $P(\mathbf{U})$ shifts, every non-causal method collapses to near-$0\%$ success while the causal methods retain significantly more of the performance. Thus, confounding cannot be diagnosed during training.

\begin{wrapfigure}[10]{r}{0.35\textwidth}
    \centering
    \vspace{-25pt}
    \includegraphics[width=\linewidth]{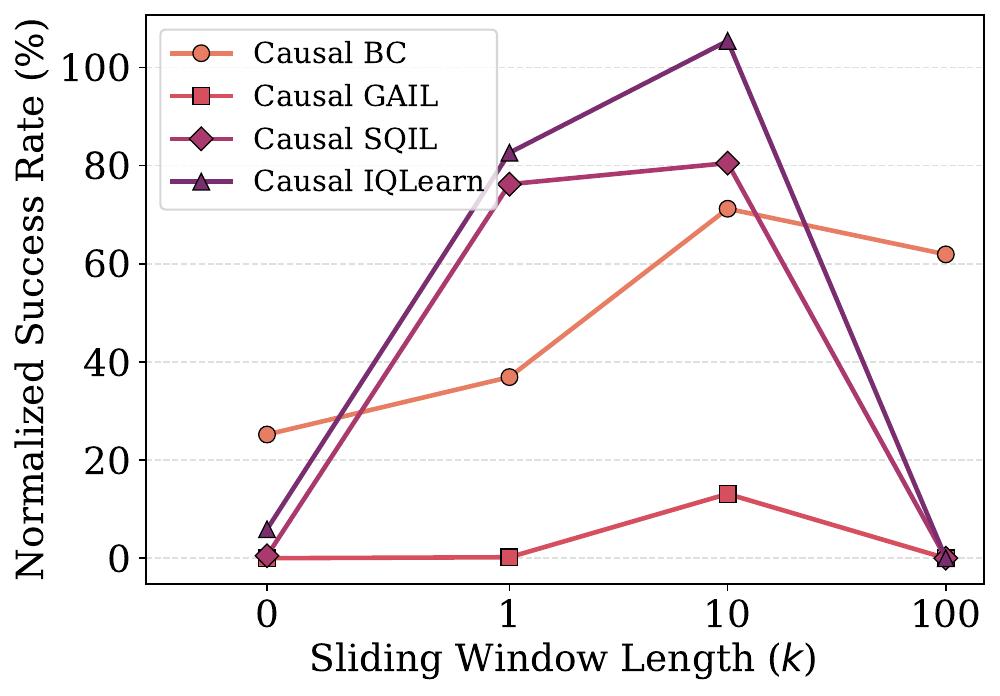}
    \caption{Sensitivity of causal algorithms to $k$ in Algorithm~\ref{alg:feasible-pi-backdoor} on Confounded AntMaze-Large.}
    \label{fig:window-ablation}
\end{wrapfigure}

\paragraph{Windowed approximation is necessary.}
Figure~\ref{fig:window-ablation} demonstrates that a moderate $k \in [1,10]$ is ideal for imitation, whereas $k=0$ (pure Markov) and $k=100$ see significant drops in most methods. This implies that while the environments require sequential decision-making capabilities, the relevant causal effects on any $X_t$ can be captured within a few timesteps; meanwhile, large $k$ bloats the representation. Interestingly, at $k=100$, Q-learning methods collapse due to bootstrapping instability in high-dimensional state spaces while BC remains robust due to its supervised learning approach.

\paragraph{The causality gap dominates the algorithmic gap.}
In every task, the causal variant of each algorithm outperforms its non-causal counterpart, with the exception of Causal GAIL and non-causal GAIL both achieving $0\%$ success rate in HumanoidMaze tasks (Table~\ref{tab:main-results}). Thus, causal reasoning remains the primary determinant of success in confounded environments. The algorithm choice becomes the secondary yet still substantial factor that determines how well a causal method performs.

\section{Conclusion}
\label{sec:conclusions}

We introduce Causal SQIL and Causal IQ-Learn, off-policy CIL algorithms that combine sequential $\pi$-backdoor adjustment with soft Q-learning objectives, and windowed causal adjustment that is tractable for long-horizon control. Our experiments demonstrate that causal adjustment is necessary for robust policy learning and that off-policy Q-learning is necessary for scaling. While our approach relies on knowledge of the causal diagram (see Appendix~\ref{sec:limitations} for a full discussion on limitations), it is algorithm-agnostic. Thus, looking forward, it can be composed with expressive policy classes (e.g. diffusion policies~\citep{chi2023diffusion} or action-chunking transformers~\citep{zhao2023aloha}), high-dimensional sensory inputs where confounding manifests through pixel-level corruptions~\citep{li2025confounding, li2026causal, juliani2026confounding}, and any setting in which the expert and imitator operate under different sensor configurations, such as tele-operation, sim-to-real transfer, and multi-agent imitation. As unobserved confounding is the norm in real-world deployment, integrating causal reasoning with scalable policy learning is essential for trustworthy imitation in safety-critical domains.

\section*{Acknowledgements}
\label{sec:acknowlegements}

This research is supported in part by the NSF, ONR, AFOSR, DoE, Amazon, JP Morgan, and The Alfred P. Sloan Foundation.

\bibliography{main}
\bibliographystyle{rlj}

\beginSupplementaryMaterials


\appendix
\appendix
\section{Related Work}
\label{app:related-work}

\paragraph{Imitation learning.}
Imitation learning (IL) trains policies from expert demonstrations without access to a reward signal. Behavioral cloning (BC) reduces IL to supervised learning~\citep{pomerleau1989alvinn}, but suffers from compounding errors due to covariate shift~\citep{ross2011reduction}. Interactive methods such as DAgger~\citep{ross2011reduction} mitigate this by querying the expert on-policy, though at the cost of requiring an interactive demonstrator. Inverse reinforcement learning (IRL) recovers a reward function rationalizing expert behavior~\citep{abbeel2004apprenticeship, ziebart2008maxent}; adversarial formulations such as GAIL~\citep{ho2016generative} and AIRL~\citep{fu2018learning} cast IRL as occupancy-measure matching, avoiding explicit reward modeling but relying on on-policy rollouts and adversarial training that scale poorly to long horizons. More recently, off-policy value-based methods have achieved strong results on continuous-control benchmarks: SQIL~\citep{reddy2020sqil} reformulates IL as soft Q-learning with binary rewards, and IQ-Learn~\citep{garg2021iqlearn} learns a Q-function whose implicit reward is consistent with expert data. Both build on the SAC framework~\citep{haarnoja2018sac} and propagate the expert signal across trajectories via temporal-difference learning, making them substantially more effective than BC or GAIL on long-horizon tasks. Modern IL has also scaled to real-world robotics through expressive policy classes such as diffusion policies~\citep{chi2023diffusion} and action-chunking transformers~\citep{zhao2023aloha}. All of these methods, however, assume that the expert and imitator share the same observation space, i.e., No Unobserved Confounders (NUC). The present work retains the scalability advantages of off-policy soft Q-learning while relaxing NUC via causal adjustment.

\paragraph{Formal objectives of prior CIL and IL methods.}
For reference, we provide the formal objectives of the methods discussed in the main text. Causal BC~\citep{kumor2021sequential} directly clones the expert's conditional policy over the admissible adjustment variables, $\hat{\pi}_t(x_t \mid \mathbf{z}_t) = P(X_t \mid \mathbf{Z}_t)$, via supervised learning on expert demonstrations. Causal GAIL~\citep{ruan2023causal} extends the framework to inverse reinforcement learning by matching expert and imitator occupancy measures over the adjusted variables:
\begin{equation}
\min_{\pi \sim \mathcal{S}} \max_D \;
\mathbb{E}\!\left[\log D(\mathbf{z}_t, x_t)\right]
+
\mathbb{E}_\pi\!\left[\log(1 - D(\mathbf{z}_t, x_t))\right].
\end{equation}
In non-causal settings, SQIL~\citep{reddy2020sqil} assigns a fixed reward of $1$ to expert transitions and $0$ to policy transitions, and then applies SAC on the combined replay buffer. The SAC critic minimizes the soft Bellman residual
\begin{equation}
\label{eq:sqil-bellman}
y = r + \gamma \Big(\min_{j=1,2} Q_{\bar{\theta}_j}(s', a') - \alpha \log \pi_\phi(a' \mid s')\Big), \qquad
\mathcal{L}_Q = \mathbb{E}\Big[\big(Q_\theta(s,a) - y\big)^2\Big],
\end{equation}
with $a' \sim \pi_\phi(\cdot \mid s')$. The actor maximizes the entropy-regularized objective
\begin{equation}
\label{eq:sqil-obj}
\mathcal{L}_\pi = \mathbb{E}_{s}\Big[\alpha \log \pi_\phi(a \mid s) - \min_{j=1,2} Q_{\theta_j}(s, a)\Big], \quad a \sim \pi_\phi(\cdot \mid s).
\end{equation}
IQ-Learn~\citep{garg2021iqlearn} learns a Q-function whose induced reward is consistent with expert behavior. The critic enforces the soft Bellman equation on expert data,
\begin{equation}
\label{eq:iql-bellman}
\mathcal{L}_{\mathrm{expert}} = \mathbb{E}_{(s,a,s') \sim \mathcal{D}_{\mathrm{exp}}}
\Big[\big(Q_\theta(s, a) - \gamma \, V_{\bar{\theta}}(s')\big)^2\Big],
\end{equation}
where $V(s) = \log \mathbb{E}_{a \sim \pi}[\exp Q(s,a)]$, and applies a policy-consistency regularizer on policy data,
\begin{equation}
\label{eq:iql-reg}
\mathcal{L}_{\mathrm{reg}} = \mathbb{E}_{(s,a) \sim \pi}
\Big[\big(\log \pi_\phi(a \mid s) - Q_\theta(s, a) + V_\theta(s)\big)^2\Big].
\end{equation}
The combined loss $\mathcal{L}_Q = \mathcal{L}_{\mathrm{expert}} + \lambda \, \mathcal{L}_{\mathrm{reg}}$ trains the critic, while the actor uses the same objective as above. Despite their scalability advantages, these methods assume fully observed (i.e., unconfounded) environments; using them in the presence of latent confounders leads to biased policies.

\paragraph{Causal imitation learning.}
Causal imitation learning (CIL) leverages structural causal knowledge to approximate the expert policy in the presence of unobserved confounding.
\citet{zhang2020causal} introduces the $\pi$-backdoor criterion, a complete graphical condition for determining policy imitability from observational data when the expert and imitator have different sensory inputs.
\citet{kumor2021sequential} extend this to sequential decision-making with the sequential $\pi$-backdoor criterion and the \textsc{FindOX} algorithm for constructing per-timestep adjustment sets.
\citet{ruan2023causal} develop CIL via inverse reinforcement learning, and \citet{ruan2024partial} introduce a partial-identification approach that can enable the imitator to surpass expert performance.
In practice, these methods have been instantiated as Causal BC and Causal GAIL, both of which assume access to a $\pi$-backdoor admissible scope.
Despite their theoretical appeal, existing CIL algorithms have remained restricted to low-dimensional, short-horizon domains:
Causal BC inherits the covariate-shift limitations of behavioral cloning, Causal GAIL's reliance on on-policy rollouts and adversarial training leads to sample and stability challenges at scale, and the adjustment sets produced by \textsc{FindOX} grow linearly with the horizon, making exact computation infeasible in long-horizon settings.
Our work addresses all three limitations by combining the causal adjustment framework with off-policy soft Q-learning objectives and introducing a windowed approximation that reduces the adjustment set computation to a fixed-size sliding window.

\paragraph{Confounding, causal confusion, and spurious correlations in IL.}
A related line of work studies the problem of spurious correlations in IL without assuming access to a causal graph.
\citet{dehaan2019causal} identify ``causal confusion,'' showing that conditioning on non-causal features can degrade imitation performance, and propose targeted interventions to select the correct causal model.
\citet{ortega2021delusions} formalize ``delusions'' in sequence models, demonstrating that treating past actions as observations rather than interventions leads to incorrect inference.
Variations of this problem have been studied under different names: as feedback-driven covariate shift~\citep{spencer2021feedback}, as the copycat problem in BC from observation histories~\citep{wen2020copycat}, as confounding in driving settings~\citep{codevilla2019exploring}, and as deconfounding via initial-state interventions~\citep{pfrommer2023initial}.
These works diagnose the problem and propose heuristic or environment-specific solutions. In contrast, the CIL framework this paper builds on provides a complete, non-parametric graphical criterion for determining imitability and identifying the correct adjustment set.

Several works address confounded IL through mechanisms other than graphical backdoor adjustment.
\citet{swamy2022causal} apply instrumental variable (IV) regression to handle temporally correlated noise in expert actions, proposing DoubIL and ResiduIL.
\citet{swamy2022sequence} prove that on-policy sequence models can asymptotically recover expert behavior when the expert observes privileged information, while showing that off-policy methods ``latch'' onto confounded features---a failure mode that directly motivates the causal adjustment layer in our off-policy algorithms.
\citet{vuorio2024deconfounded} train a variational inference model to infer the expert's latent information and learn a latent-conditional policy.
\citet{shao2025dmlil} propose DML-IL, a unifying framework that reformulates causal IL as a conditional moment restriction problem solved via IV regression, handling both expert-observable and expert-unobservable confounders.
\citet{zeng2025confounded} similarly leverage IVs for confounded sequential IL.
These IV and latent-inference approaches typically assume additive or temporally bounded confounding for instrument validity, or require ergodic dynamics for latent inference; they do not require a known causal graph but trade this for stronger distributional assumptions. Our approach instead assumes a known causal graph and applies non-parametric graphical adjustment, which is complementary: it provides exact deconfounding guarantees when the graph is available, and our windowed approximation makes it tractable at scale.

\paragraph{Robust imitation learning without causal structure.}
A complementary line of work improves the robustness of imitation policies to distribution shift without positing a causal model, and it is natural to ask whether such methods could substitute for causal adjustment. These methods are best organized by the shift they target. A first family targets the covariate shift induced by the imitator's compounding errors: DAgger~\citep{ross2011reduction} queries the expert on learner-visited states, DART~\citep{laskey2017dart} injects noise into the expert's controls during collection so that demonstrations cover the learner's error distribution, and \citet{chang2021mitigating} handle the offline case by exploiting a supplementary dataset with partial coverage through model-based pessimism. A second family targets covariate shift already present in the demonstrations: distribution-matching methods~\citep{kim2022demodice} align the imitator's stationary distribution with that of the data, while distributionally robust formulations optimize the worst-case imitation loss over an uncertainty set of state distributions around the data: a total-variation ball in DR-BC~\citep{panaganti2023distributionally}, and, in DrilDICE~\citep{seo2024drildice}, an $f$-divergence ball intersected with the set of stationary distributions satisfying the Bellman flow constraint, which handles demonstrations collected from arbitrary non-stationary state distributions. Related robust formulations address noisy expert actions~\citep{bashiri2021distributionally}, suboptimal or corrupted demonstrations~\citep{sasaki2021behavioral, wu2019imitation, xu2022discriminator}, and variations in the environment dynamics~\citep{chae2022robust}; \citet{tennenholtz2021covariate} intertwine this literature with ours by studying covariate shift of latent confounders themselves. The unifying structure of the covariate-shift family is that each method changes where the supervised imitation loss is evaluated, whether by reweighting states, enlarging coverage, or optimizing a worst case over state distributions, while fixing the observational conditional of the expert's action given the imitator's observations. This is explicit in \citet{seo2024drildice}, whose problem statement assumes that the training and target distributions differ only in their state marginals while sharing the same state-conditional expert action everywhere.

\paragraph{Orthogonality of confounding and covariate-shift robustness.}
Distribution shift in sequential decision-making arises from distinct causes (see \citet{bareinboim2016fusion} for a broader treatment of the biases underlying generalization): (1) non-overlap, where the imitator visits state-action regions unsupported by the data; (2) confounding bias, where unobserved variables jointly influence observations, actions, and outcomes, so that the observational conditional differs from the interventionally correct policy; and (3) structural change between the training and deployment environments. Covariate-shift-robust IL addresses cause (1) (and dynamics-robust variants a form of (3)) under the assumption that the per-observation conditional is invariant. However, cause (2) violates that invariance. When unobserved confounders shift at deployment, the conditional $P(x \mid \mathbf{v}^O)$ that was correct on the demonstrations becomes the wrong policy, and no reweighting of where that conditional is fit can repair it, because the bias lives inside the conditional rather than in the marginal it is averaged over. This mirrors the situation in offline RL, where conservatism toward out-of-distribution actions cannot correct value estimates biased by confounding~\citep{li2025confounding}. This is seen in the following instance, in which the covariate-shift machinery is provably inert.

\begin{example}[Confounded Route Choice]
\label{ex:orthogonality}
Consider a single-stage SCM with binary variables. A goal lies on side $Z \sim \mathrm{Bern}(1/2)$ of a junction, and the binary actuation $X$ selects a route $A = X \oplus Z$ relative to the goal side: route $a{=}0$ is exposed and succeeds only in calm conditions, $P(Y{=}1 \mid a{=}0, u) = 1 - u$, while route $a{=}1$ is sheltered and succeeds with probability $3/4$ regardless, $P(Y{=}1 \mid a{=}1, u) = 3/4$. The wind $U \sim \mathrm{Bern}(p)$ is exogenous and unobserved by all agents, with $p_{\mathrm{tr}} = 1/4$ during demonstration collection. The expert observes its orientation $O \sim \mathrm{Bern}(2/3)$ in addition to $Z$; the imitator does not observe $O$ but instead carries a compass $W = O \oplus U$ whose needle is deflected by the wind, so that $O \to W \gets U$ is a collider with no outgoing edges, exactly as in Figure~\ref{fig:generic-causal-graph}. Under $p_{\mathrm{tr}} = 1/4$ both routes have expected reward $3/4$ for every $(z, o)$, so the expert is optimal over its information set and, being indifferent, takes the route it is facing: $X = Z \oplus O$. The induced diagram has edges $Z \to X$, $Z \to Y$, $X \to Y$, $O \to X$, $O \to W$, $U \to W$, $U \to Y$; its latent projection onto $\{Z, W, X, Y\}$ has directed edges $Z \to X$, $Z \to Y$, $X \to Y$ and bidirected edges $X \leftrightarrow W$, $W \leftrightarrow Y$. Applying Definition~\ref{def:seq-pi-backdoor}, the set $\{Z\}$ is $\pi$-backdoor admissible; $\{Z, W\}$ is not, since conditioning on the collider $W$ opens $X \leftrightarrow W \leftrightarrow Y$; and $\varnothing$ is not, since $X \gets Z \to Y$ remains open.

The demonstrations $(z, w, x) \sim P_{\mathrm{tr}}$ have full support (the smallest cell has probability $1/24$) and are drawn i.i.d.\ from the expert's own occupancy. There is thus no non-overlap, no compounding error as the horizon is one, and no mismatch between the dataset and the expert's state distribution as the shift targeted by \citet{seo2024drildice} is intentionally absent. Nevertheless, the compass is genuinely informative in-distribution,
\[
\begin{aligned}
    P_{\mathrm{tr}}(X = z \mid z, W{=}0) = \tfrac{(1/3)(3/4)}{(1/3)(3/4) + (2/3)(1/4)} = \tfrac{3}{5}, \\
    P_{\mathrm{tr}}(X = 1{-}z \mid z, W{=}1) = \tfrac{(2/3)(3/4)}{(2/3)(3/4) + (1/3)(1/4)} = \tfrac{6}{7},
\end{aligned}
\]

so the loss-minimizing deterministic imitator over $(Z, W)$ follows the compass, $\hat{\pi}_n(z, w) = z \oplus w$ (i.e., $a = w$), while the adjusted conditional over $\{Z\}$ is $\pi_c(x \mid z) = P_{\mathrm{tr}}(x \mid z)$, which takes the sheltered route with probability $P(O{=}1) = 2/3$, and its deterministic counterpart $\hat{\pi}_c$ always takes it. At deployment the wind picks up, $p_{\mathrm{te}} = 3/4$. The compass-following imitator now takes the exposed route precisely when the wind is blowing ($P(a{=}0 \mid U{=}1) = P(W{=}0 \mid U{=}1) = 2/3$) and collects zero reward there, whereas the adjusted policies are indifferent to the compass and unaffected by its miscalibration. Table~\ref{tab:orthogonality-values} reports the exact interventional values. Two features mirror our main experiments: in-distribution, the confounded policies trail the expert by margins comparable to evaluation noise ($0.721$ and $0.688$ against $0.750$), so the bias is nearly invisible during training, as in Figure~\ref{fig:training-curves-main}; and the deterministic adjusted policy surpasses the fixed expert at deployment, as causal methods occasionally do in Table~\ref{tab:main-results} \citep[cf.][]{ruan2024partial}. By contrast, the stochastic adjusted policy $\pi_c$ attains the expert's interventional value under every deployment wind distribution: the expert's action satisfies $X = Z \oplus O$ with $O \perp (Z, U)$, so the expert and $do(\pi_c)$ induce the same joint over $(Z, X)$ with $X \perp U \mid Z$ in both, and since the parents of $Y$ are $\{X, Z, U\}$, the two interventional distributions of $Y$ coincide for every $P(U)$.
$\hfill \blacksquare$
\end{example}

\begin{table}[H]
\centering
\caption{Interventional values $\mathbb{E}[Y \mid do(\pi)]$ in Example~\ref{ex:orthogonality} under the training wind ($p_{\mathrm{tr}} = 1/4$) and the deployment wind ($p_{\mathrm{te}} = 3/4$). ``Robust-BC output'' is the deterministic policy returned by BC and by every covariate-shift-robust method in the class of Proposition~\ref{prop:reweighting}. All values are exact; the general closed forms in $s = P(U{=}1)$ are shown for reference. Every entry is reproduced by exhaustive enumeration.}
\label{tab:orthogonality-values}
\resizebox{\textwidth}{!}{
\begin{tabular}{llccc}
\hline
\textbf{Policy} & \textbf{Conditions on} & $\mathbb{E}[Y \mid do(\pi)]$, any $s$ & \textbf{Train} ($s{=}\tfrac14$) & \textbf{Deploy} ($s{=}\tfrac34$) \\
\hline
Expert ($X = Z \oplus O$) & $\{Z, O\}$ & $\tfrac12 + \tfrac{1-s}{3}$ & $3/4 = 0.750$ & $7/12 \approx 0.583$ \\
$\pi_c(x \mid z) = P_{\mathrm{tr}}(x \mid z)$ & $\{Z\}$ & $\tfrac12 + \tfrac{1-s}{3}$ & $3/4 = 0.750$ & $7/12 \approx 0.583$ \\
$\hat{\pi}_c$ (deterministic, sheltered) & $\{Z\}$ & $\tfrac34$ & $3/4 = 0.750$ & $3/4 = 0.750$ \\
$\pi_n(x \mid z, w) = P_{\mathrm{tr}}(x \mid z, w)$ & $\{Z, W\}$ & $(1{-}s)\tfrac{173}{210} + s\tfrac{87}{210}$ & $101/140 \approx 0.721$ & $31/60 \approx 0.517$ \\
$\hat{\pi}_n$ (Robust-BC output: $a = w$) & $\{Z, W\}$ & $(1{-}s)\tfrac{5}{6} + s\tfrac{1}{4}$ & $11/16 \approx 0.688$ & $\mathbf{19/48 \approx 0.396}$ \\
\hline
\end{tabular}
}
\end{table}

\begin{proposition}[State reweighting cannot remove confounding bias]
\label{prop:reweighting}
Consider the demonstration distribution $P_{\mathrm{tr}}$ of Example~\ref{ex:orthogonality}; any single-stage instance with full support over the imitator's observations suffices. For a deterministic policy $\pi$ over observations $(z, w)$, let $C_\pi(z, w) := \mathbb{E}_{x \sim P_{\mathrm{tr}}(\cdot \mid z, w)}[\ell(\pi(z, w), x)]$ denote the expected loss against the sampled expert actions, with $\ell$ the 0--1 loss. Then $\hat{\pi}_n(z, w) = \arg\max_x P_{\mathrm{tr}}(x \mid z, w)$ minimizes $\mathbb{E}_{(z,w) \sim d}[C_\pi(z, w)]$ simultaneously for every distribution $d$ over observations. Consequently, $\hat{\pi}_n$ attains the minimum of every objective formed by nonnegatively aggregating the per-observation losses $\{C_\pi(z, w)\}$: including any supremum or infimum over a family of state weightings, with or without an $f$-divergence regularizer,
\[
\min_\pi \; \sup_{d \in \mathcal{Q}} \; \mathbb{E}_{d}\!\left[C_\pi\right] - \alpha D_f(d \,\Vert\, d_D),
\]
for every uncertainty set $\mathcal{Q}$, every convex generator $f$, and every $\alpha \geq 0$. This class contains behavioral cloning, adversarially weighted BC, DR-BC~\citep{panaganti2023distributionally}, OptiDICE-BC~\citep{lee2021optidice}, and DrilDICE~\citep{seo2024drildice}, which therefore all return the compass-following policy of Example~\ref{ex:orthogonality} and its deployment value of $19/48$. This also holds for the stochastic conditional $\pi_n(x \mid z, w) = P_{\mathrm{tr}}(x \mid z, w)$ under the log loss.
\end{proposition}

\begin{proof}
Each objective decomposes over observations, $\mathbb{E}_d[C_\pi] = \sum_{z,w} d(z, w)\, C_\pi(z, w)$, and by full support $C_\pi(z, w)$ is well defined at every $(z, w)$ and is minimized pointwise by $\hat{\pi}_n$ (respectively, by $\pi_n$ under the log loss). A simultaneous pointwise minimizer minimizes every nonnegatively weighted aggregate of the per-observation losses, hence every supremum or infimum of such aggregates over families of weightings, and subtracting a term $\alpha D_f(d \Vert d_D)$ that does not depend on $\pi$ changes neither the inner optimizer's feasible directions in $\pi$ nor the outer $\arg\min$. For DrilDICE specifically, each policy update minimizes $\mathbb{E}_{(s,a) \sim d_D}[w^*(s, a)\, C_\pi(s)]$ with fixed nonnegative weights; since the cost is constant in $a$ given $s$, this equals a state-weighted BC loss with weights $\omega(s) = \mathbb{E}_{d_D}[w^*(s, a) \mid s] \geq 0$, and in the single-stage problem the weights are moreover directly state-measurable, as $e_{\pi,\nu}(s, a, s') = C_\pi(s) - \nu(s)$ with the terminal value fixed to zero. Hence $\hat{\pi}_n$ is returned by the two-step instantiation of DrilDICE and is a fixed point of its alternating scheme. When the optimal weighting places positive mass on every observation, $\hat{\pi}_n$ is the unique minimizer, since the argmax in Example~\ref{ex:orthogonality} is strict at every $(z, w)$.
\end{proof}

The orthogonality runs in both directions, and neither correction subsumes the other. Causal adjustment does nothing for cause (1): Causal BC in Figure~\ref{fig:failure-bc} drifts outside the demonstration support on AntMaze-Large despite conditioning on admissible adjustment sets, purely a non-overlap failure. Symmetrically, the covariate-shift constructions of \citet{seo2024drildice} contain no unobserved confounding, so a causal method handed the unconfounded diagram would reduce to plain BC and inherit its failure there. Note that robustifying against the deployment-time shift in the observable marginals would not rescue the confounded imitator in Example~\ref{ex:orthogonality}: the compass marginal does shift ($P(W{=}1)$ moves from $7/12$ to $5/12$), but the failure is driven by the shifted conditional, which every method in the class intentionally leaves unmodified. Because the two corrections act on different components of the learning problem, either the conditioning set or the distribution over which the loss is evaluated, they compose rather than compete: Causal SQIL and Causal IQ-Learn are themselves such a composition, pairing adjustment for cause (2) with off-policy temporal-difference learning against the compounding-error consequences of cause (1), and composing windowed causal adjustment with DICE-style stationary distribution corrections is a natural direction for future work.

\paragraph{Causal reinforcement learning.}
Causal reasoning has been increasingly integrated into RL to handle confounded offline data and distribution shift.
\citet{li2025confounding} develop confounding-robust deep RL via causal adjustment in the presence of unobserved confounders;
\citet{li2026causal} extend this to offline settings via causal flow Q-learning;
and \citet{juliani2026confounding} address confounders in continuous control through automatic reward shaping.
These works operate in the RL setting where a reward signal is available, whereas the present paper addresses the harder IL setting where the reward is entirely latent.
Nevertheless, the structural causal formalization and the challenge of off-policy learning under confounders are shared, and our causal adjustment layer is compatible with advances in causal RL.

\paragraph{Learned history compression.}
Since Algorithm~\ref{alg:feasible-pi-backdoor} ultimately encodes a fixed window of recent observations, a natural question is whether a learned sequence model, such as an RNN or Transformer over $\mathbf{s}_{t-k:t}$, could replace graphical adjustment altogether. The two mechanisms are not interchangeable: Algorithm~\ref{alg:feasible-pi-backdoor} determines both which variables may be conditioned on and which lags are required, whereas a learned compressor addresses only the latter. A sequence model over the raw history still conditions on every observed variable, including the collider $\mathbf{W}$, and its training objective provides no signal to exclude it: the spurious association is predictive in-distribution (Figure~\ref{fig:training-curves-main}), and history-conditioned imitators are prone to latching onto past actions and confounded features~\citep{wen2020copycat, swamy2022sequence}, so such a model would inherit the biased conditional that defeats the non-causal baselines in Table~\ref{tab:main-results}. Furthermore, Theorem~\ref{thm:windowed-adjustment} and the window ablation (Figure~\ref{fig:window-ablation}) indicate that the sufficient window is small relative to the horizon in our environments, limiting the benefit of expressive history models over the direct encoding of Algorithm~\ref{alg:feasible-pi-backdoor}. Learned compression instead becomes appropriate when Assumption~\ref{assm:bounded-confounding} fails and long-range latent dependencies require summarizing history beyond any fixed window (Appendix~\ref{sec:limitations}); composing the variable selection of Algorithm~\ref{alg:feasible-pi-backdoor} with learned temporal compression of the selected variables is a noteworthy direction for future work.

\section{Proof}
\label{app:proof-windowed}

\begin{theorem}[Restatement of Theorem~\ref{thm:windowed-adjustment}]
Let $\mathcal{G}_H$ be the causal diagram of an SCM unrolled over horizon $H$, and let Assumption~\ref{assm:bounded-confounding} hold with window size $k$. Let $\{\mathbf{Z}_t^k\}_{t=0}^{k}$ be adjustment sets constructed by Algorithm~\ref{alg:feasible-pi-backdoor} that satisfy the sequential $\pi$-backdoor criterion (Definition~\ref{def:seq-pi-backdoor}) for the proxy graph $\mathcal{G}_k$ with horizon $k+1$. Then the transferred sets $\{\mathbf{Z}_t^H\}_{t=0}^{H-1}$ returned by Algorithm~\ref{alg:feasible-pi-backdoor} satisfy the sequential $\pi$-backdoor criterion for $(\mathcal{G}_H, \mathbf{X}, Y)$.
\end{theorem}

\begin{proof}
We show that for every $t \in \{0, \ldots, H{-}1\}$, one of the two conditions of Definition~\ref{def:seq-pi-backdoor} holds for $\mathbf{Z}_t^H$ in the full-horizon manipulated graph $(\mathcal{G}_H)_t'$. We handle $t \geq k$ first, then $t < k$.

\paragraph{Setup and notation.}
Throughout the proof we work in the latent projections $\widetilde{\mathcal{G}}_H$ and $\widetilde{\mathcal{G}}_k$ onto $\mathbf{V}^O \cup \{Y\}$ (respectively $Y_k$), on which Assumption~\ref{assm:bounded-confounding}(i) is stated. Since latent projection preserves d-separation among the retained variables and directed (ancestral) relationships between them \citep{verma1990equivalence, tian2002ccomp}, both conditions of Definition~\ref{def:seq-pi-backdoor} are invariant under projection, and it suffices to establish them in the projected graphs; the manipulated-graph operations (deleting incoming edges of future actions and re-parenting them on the observed sets $\mathbf{Z}_j^H$) involve only retained variables and can equivalently be performed on the projections. To avoid clutter, we write $\mathcal{G}_H$ and $\mathcal{G}_k$ for the projections in the remainder of the proof. In particular, every node of these graphs is observed (or is the reward), $\mathrm{pa}^+$ always denotes the extended parent set in the projection, and a single edge may represent a latent-mediated path of the original SCM, so edges span up to $k$ timesteps rather than one.

Let $(\mathcal{G}_H)_t'$ denote the manipulated graph at timestep $t$: all incoming edges into future actions $X_j$ ($j > t$) are removed and each $X_j$ is given parents $\mathbf{Z}_j^H$. Let $((\mathcal{G}_H)_t')_{X_t}$ further delete all outgoing edges from $X_t$. Define $\mathrm{Win}_t := \{t{-}k, \ldots, t\}$. We use $V_r$ (non-bold) for an individual variable at timestep $r$.

\paragraph{Step 1: Structure of the manipulated graph.}

By $k$-bounded influence (Assumption~\ref{assm:bounded-confounding}(i)), for every $t$ the extended parent set $\mathrm{pa}^+(\mathbf{V}_t)$ lies within the window $\mathrm{Win}_t = \{t{-}k, \ldots, t\}$: all directed parents of $\mathbf{V}_t$, all variables collider-connected to $\mathbf{V}_t$, and the parents of those variables, occur at timesteps $\geq t - k$. In $(\mathcal{G}_H)_t'$, the original incoming edges to future actions are replaced by edges from $\mathbf{Z}_j^H$, which by Algorithm~\ref{alg:feasible-pi-backdoor} contain only variables within $k$ steps of $j$; this replacement only deletes edges and redirects action parents to in-window variables; since deleting a bidirected edge can only shrink a C-component (Definition~\ref{def:ccomp}) and redirecting adds only in-window parents, no extended parent set grows beyond its width-$k$ bound. Therefore the bound on $\mathrm{pa}^+$ is preserved in $(\mathcal{G}_H)_t'$. Note also that the bound implies every edge of the (projected, manipulated) graph spans at most $k$ timesteps: for a directed edge $B \to A$, $B \in \mathrm{pa}(A) \subseteq \mathrm{pa}^+(\mathbf{V}_{\mathrm{slice}(A)})$, and for a bidirected edge $B \leftrightarrow A$ with $B$ earlier, $B$ is collider-connected to $A$ and hence $B \in \mathrm{pa}^+(\mathbf{V}_{\mathrm{slice}(A)})$.

Crucially, the window is sealed against \emph{open} escapes into the past. Suppose a path segment lying behind $X_t$ (i.e., before any crossing into the tail defined below) leaves $\mathrm{Win}_t$ downward, and let $A \in \mathrm{Win}_t$ be a node at which it exits or re-enters, with neighbor $B$ on the segment at a timestep $< t-k$. Since directed edges respect the temporal order and every edge spans at most $k$ lags, the boundary edge must be $B \to A$ or $B \leftrightarrow A$; either way it places an arrowhead at $A$. If the other path-edge at $A$ also carries an arrowhead, $A$ is a collider, and the segment is blocked at $A$ unless $A \in \mathrm{An}(\mathbf{Z}_t^H)$; if it carries a tail, $A$ is a non-collider, and the segment is blocked at $A$ unless $A \notin \mathbf{Z}_t^H$. It is here that the proof uses the specific construction of Lines~1--10 of Algorithm~\ref{alg:feasible-pi-backdoor} rather than mere proxy admissibility. By the Markov-boundary property of the construction \citep[Definition~3.1 and Lemma~3.2]{kumor2021sequential}, transferred to $\mathrm{Win}_t$ by the isomorphism of Step~2, every observed window variable that receives an incoming arrowhead from outside the conditioned corridor and transmits along a tail edge either belongs to $\mathbf{Z}_t^H$ (and hence blocks the segment as a conditioned non-collider) or is d-separated from $X_t$ given $\mathbf{Z}_t^H$ within the window (so that the segment is blocked before reaching $X_t$). In the collider case, $A \in \mathbf{Z}_t^H$ opens $A$, but the segment then continues below the window, where all nodes are unconditioned, and must eventually re-enter $\mathrm{Win}_t$ to reach $Y$ (every edge spans at most $k$ lags, and $Y$ lies above the window); at the re-entry node the same dichotomy applies, and since a finite path can sustain only finitely many below-window excursions, it must ultimately traverse a transmitting (tail-edge) window node, where the Markov-boundary property blocks it. Hence no open path behind $X_t$ leaves $\mathrm{Win}_t$, and no variable at a timestep $\leq t-k-1$ participates in an open path into $\mathrm{Win}_t$.

We decompose $(\mathcal{G}_H)_t'$ into two parts:
\begin{itemize}[topsep=2pt, parsep=0pt, itemsep=2pt]
    \item The local subgraph $\mathcal{L}_t$: the restriction of $(\mathcal{G}_H)_t'$ to variables at timesteps in $\mathrm{Win}_t$.
    \item The tail subgraph $\mathcal{T}_t$: the restriction of $(\mathcal{G}_H)_t'$ to variables at timesteps in $\{t{+}1, \ldots, H\}$, including $Y$.
\end{itemize}
By the $\mathrm{pa}^+$ bound, every directed or bidirected connection between $\mathcal{L}_t$ and $\mathcal{T}_t$ runs from a variable in $\{t{-}k{+}1, \ldots, t\} \subset \mathrm{Win}_t$ to a variable in $\{t{+}1, \ldots, t{+}k\} \subset \mathcal{T}_t$. We call the variables in $\mathcal{T}_t$ at timesteps $\{t{+}1, \ldots, t{+}k\}$ that are collider-connected to, or have a parent in, $\mathcal{L}_t$ the interface variables $\mathbf{I}_t$. By the sealing property, every path leaving $\mathcal{L}_t$ for $\mathcal{T}_t$ does so through a node in $\mathbf{I}_t$ whether via a directed or bidirected edge; these interface nodes together with their incident edges are themselves contained in $\mathrm{pa}^+$-bounded neighborhoods of the window. In particular, both directed and bidirected boundary crossings are confined to $\mathbf{I}_t$; the manipulated-graph construction additionally removes incoming edges to future actions $X_j \in \mathcal{T}_t$, so no boundary crossing terminates at a future action.

\paragraph{Step 2: Isomorphism between the local subgraph and the proxy graph.}

Consider the proxy graph $\mathcal{G}_k$ on timesteps $\{0, \ldots, k\}$ with terminal reward $Y_k$. Define the time-shift $\phi_t(\tau) = \tau + (t - k)$. By time-homogeneity (Assumption~\ref{assm:bounded-confounding}(ii)), $\phi_t$ induces a graph isomorphism between $\mathcal{G}_k$ restricted to $\{0, \ldots, k\}$ and $\mathcal{L}_t$, preserving all directed and bidirected edges of the projections as well as the identity of state variables, actions, and reward. (Time-homogeneity of $\mathscr{F}$ carries over to the projections, since the latent projection is determined by the diagram.)

This extends to manipulated graphs. In the proxy manipulated graph $(\mathcal{G}_k)_k'$, timestep $k$ is the final step so there are no future actions to modify. In $\mathcal{L}_t$, the action $X_t$ is likewise the last action (all actions at $j > t$ belong to $\mathcal{T}_t$). The adjustment sets satisfy $\mathbf{Z}_{\phi_t(\tau)}^H = \phi_t(\mathbf{Z}_\tau^k)$ by the identical relative-lag construction of Algorithm~\ref{alg:feasible-pi-backdoor}. Therefore $\mathcal{L}_t$ with the manipulated-graph modifications is isomorphic to $(\mathcal{G}_k)_k'$ under $\phi_t$.

\paragraph{Step 3: Relating $Y_k$ and $Y$ via the interface.}

In $\mathcal{G}_k$, the terminal reward $Y_k$ receives edges from variables at timesteps $\{k{-}k, \ldots, k\} = \{0, \ldots, k\}$ (by $k$-bounded influence). Since $Y$ is a terminal sink node with only incoming edges in both $\mathcal{G}_k$ and $\mathcal{G}_H$, and its parent structure follows the same time-homogeneous template, the structural role of $Y_k$ relative to $\{0, \ldots, k\}$ in $\mathcal{G}_k$ is identical to that of $Y$ relative to $\{H{-}k, \ldots, H\}$ in $\mathcal{G}_H$. Under the isomorphism, these correspond to variables in $\mathrm{Win}_t$ in $\mathcal{G}_H$.

In $(\mathcal{G}_H)_t'$, the terminal reward $Y$ at timestep $H$ is not solely adjacent to $\mathrm{Win}_t$; it is connected to $\mathrm{Win}_t$ only through paths that traverse $\mathcal{T}_t$. However, by Step~1, \emph{all} connections from $\mathcal{L}_t$ into $\mathcal{T}_t$ pass through the interface $\mathbf{I}_t$ at timesteps $\{t{+}1, \ldots, t{+}k\}$. This means that $\mathbf{I}_t$ is a cut set between $\mathrm{Win}_t$ and $Y$ in $\mathcal{T}_t$: every path from any variable in $\mathcal{L}_t$ to $Y$ in $(\mathcal{G}_H)_t'$ must pass through some node of $\mathbf{I}_t$.

The extended parents of $Y_k$ in $\mathcal{G}_k$ (variables at timesteps $\{0, \ldots, k\}$) are mapped by $\phi_t$ to variables in $\mathrm{Win}_t$ that are exactly the extended parents of the interface variables $\mathbf{I}_t$ within $\mathcal{L}_t$. Thus $Y_k$ in $\mathcal{G}_k$ and $Y$ in $\mathcal{G}_H$ attach to $\mathcal{L}_t$ through the same set of nodes, with $Y_k$ directly and $Y$ through $\mathbf{I}_t$ and the tail. By the sealing property of Step~1, every path from a variable in $\mathcal{L}_t$ to $Y$ in $(\mathcal{G}_H)_t'$ passes through $\mathbf{I}_t$.

\paragraph{Step 4: Transferring the sequential $\pi$-backdoor conditions.}

Since $\{\mathbf{Z}_t^k\}$ satisfies the sequential $\pi$-backdoor on $\mathcal{G}_k$, at timestep $k$ either:

\emph{Case A: $X_k \notin \mathrm{An}_{(\mathcal{G}_k)_k'}(Y_k)$.}

This means there is no directed path from $X_k$ to $Y_k$ in $(\mathcal{G}_k)_k'$. By the isomorphism of Step~2, there is no directed path from $X_t$ to any interface variable in $\mathbf{I}_t$ within $\mathcal{L}_t$. Since every directed path from $X_t$ to $Y$ in $(\mathcal{G}_H)_t'$ must pass through $\mathbf{I}_t$ (Step~3), and no directed path from $X_t$ reaches $\mathbf{I}_t$, we conclude $X_t \notin \mathrm{An}_{(\mathcal{G}_H)_t'}(Y)$.

\emph{Case B: $(X_k \perp\!\!\!\perp Y_k \mid \mathbf{Z}_k^k)_{((\mathcal{G}_k)_k')_{X_k}}$.}

We show $(X_t \perp\!\!\!\perp Y \mid \mathbf{Z}_t^H)_{((\mathcal{G}_H)_t')_{X_t}}$. Let $p$ be any path from $X_t$ to $Y$ in $((\mathcal{G}_H)_t')_{X_t}$.

Since outgoing edges from $X_t$ are deleted, $p$ must leave $X_t$ via an incoming edge from a parent or a node collider-connected to $X_t$, all of which lie in $\mathrm{Win}_t$ by the $\mathrm{pa}^+$ bound (Step~1). If any node of $p$ behind $X_t$ (i.e., before its first crossing into $\mathcal{T}_t$) occurred at a timestep $< t-k$, then $p$ would exit $\mathrm{Win}_t$ downward, and by the sealing property of Step~1 it would be blocked by $\mathbf{Z}_t^H$; in that case we are done. It therefore suffices to consider paths whose portion behind $X_t$, up to the first crossing into $\mathcal{T}_t$, is contained in $\mathcal{L}_t$.

Since $p$ connects $X_t$ to $Y$ and $Y \notin \mathcal{L}_t$, the path $p$ must at some point cross from $\mathcal{L}_t$ into $\mathcal{T}_t$. By Step~1, this crossing must go through an interface variable $I \in \mathbf{I}_t$.

Now consider the path $p$ as having two segments: the segment $p_L$ from $X_t$ to the first interface variable $I$ (contained in $\mathcal{L}_t \cup \{I\}$), and the segment $p_T$ from $I$ onward to $Y$ (contained in $\mathcal{T}_t$). We show $p$ is blocked by $\mathbf{Z}_t^H$ by showing $p_L$ is blocked.

Construct the corresponding path $p_L'$ in $((\mathcal{G}_k)_k')_{X_k}$: apply $\phi_t^{-1}$ to map $p_L$ from $\mathcal{L}_t$ to $\mathcal{G}_k$. The interface variable $I$ at timestep $t + j$ (for some $1 \leq j \leq k$) maps to a variable at timestep $k + j$ in $\mathcal{G}_k$. However, $\mathcal{G}_k$ only contains timesteps $\{0, \ldots, k\}$, so if $j \geq 1$, the variable $I$ maps to a timestep beyond $\mathcal{G}_k$.

We handle this as follows. By definition of $\mathbf{I}_t$, the interface variable $I$ is connected to $\mathcal{L}_t$ through its extended parents in $\mathcal{L}_t$: either a directed parent, or a node collider-connected to $I$ whose own parents lie in $\mathcal{L}_t$. By Step~3, $Y_k$ in the proxy attaches to the in-window nodes through the same extended-parent structure (under $\phi_t$). Therefore any path from $X_k$ that reaches an extended parent of $Y_k$ in $((\mathcal{G}_k)_k')_{X_k}$ corresponds to a path from $X_t$ that reaches an extended parent of $\mathbf{I}_t$ in $((\mathcal{G}_H)_t')_{X_t}$, and conversely.

Concretely, the path $p_L$ from $X_t$ to $I$ reaches $I$ through some node $V_q \in \mathrm{pa}^+(I) \cap \mathcal{L}_t$. The sub-path from $X_t$ to $V_q$ lies entirely in $\mathcal{L}_t$ (any step to a timestep $< t-k$ is covered by the sealing property of Step~1, under which $p$ is already blocked) and maps under $\phi_t^{-1}$ to a path from $X_k$ to $\phi_t^{-1}(V_q)$ in $((\mathcal{G}_k)_k')_{X_k}$. Because $\mathrm{pa}^+$ is defined as the parents of a C-component, the collider/non-collider status of $V_q$ along $p_L$ is preserved under $\phi_t^{-1}$: a node entered through a bidirected edge inside a collider-connected chain remains so in the proxy, and a directed parent remains a directed parent. Since $V_q$ is an extended parent of an interface variable, $\phi_t^{-1}(V_q)$ is an extended parent of $Y_k$ in $\mathcal{G}_k$ (Step~3), so the path from $X_k$ to $\phi_t^{-1}(V_q)$ extends to $Y_k$ through the corresponding parent or collider-connecting edge, forming a path $q$ from $X_k$ to $Y_k$ in $((\mathcal{G}_k)_k')_{X_k}$.

By the $d$-separation hypothesis, $q$ is blocked by $\mathbf{Z}_k^k$. The extension from $\phi_t^{-1}(V_q)$ to $Y_k$ introduces no open collider at $\phi_t^{-1}(V_q)$ that was not already present on $p_L$: if the extension makes $\phi_t^{-1}(V_q)$ a collider, the identical collider configuration appears at $V_q$ on the full-graph path $p$ (via the matching edge into the interface chain), so the two paths share collider status at this node and a block of one implies a block of the other. Hence the block of $q$ occurs on the sub-path from $X_k$ to $\phi_t^{-1}(V_q)$, at a node of $\mathbf{Z}_k^k$. Mapping back via $\phi_t$, the same block occurs on $p_L$ at the corresponding node of $\mathbf{Z}_t^H = \phi_t(\mathbf{Z}_k^k)$. Therefore $p$ is blocked by $\mathbf{Z}_t^H$.

\paragraph{Boundary timesteps ($t < k$).}

For $t < k$, the window is $\{0, \ldots, t\}$. The $\mathrm{before}(X_t)$ constraint in Algorithm~\ref{alg:feasible-pi-backdoor} (Line~6) ensures $\mathbf{Z}_t^k$ contains only variables at timesteps $\leq t - 1$. Since $t < k$, all such variables satisfy $\tau \geq 0 > t - k$, so the clipping in Lines~11--13 is vacuous: $\mathbf{Z}_t^H = \mathbf{Z}_t^k$. The subgraph of $\mathcal{G}_H$ restricted to $\{0, \ldots, t\}$ is identical to the subgraph of $\mathcal{G}_k$ on the same timesteps (a direct subgraph inclusion requiring no time-shift). The manipulated graph $(\mathcal{G}_H)_t'$ restricted to $\{0, \ldots, t\}$ matches $(\mathcal{G}_k)_t'$ restricted to $\{0, \ldots, t\}$, since the modifications at timesteps $> t$ (replacement of future action parents) do not alter the subgraph at $\{0, \ldots, t\}$. Therefore the $d$-separation or non-ancestry condition holding for $X_t$ in $(\mathcal{G}_k)_t'$ also holds in $(\mathcal{G}_H)_t'$.

\medskip

Since for every $t \in \{0, \ldots, H{-}1\}$ one of the two conditions of Definition~\ref{def:seq-pi-backdoor} holds, $\{\mathbf{Z}_t^H\}_{t=0}^{H-1}$ satisfies the sequential $\pi$-backdoor criterion for $(\mathcal{G}_H, \mathbf{X}, Y)$.
\end{proof}

\subsection{Necessity of Extended Parents Set}
\label{app:why-pa-plus}
Assumption~\ref{assm:bounded-confounding}(i) bounds the temporal reach of the extended parent set $\mathrm{pa}^+(\mathbf{V}_t)$ rather than that of the ordinary parent set $\mathrm{pa}(\mathbf{V}_t)$ or, equivalently, the maximal length of an individual edge. This is necessary, as an SCM may satisfy a naive edge-length bound yet violate the windowed criterion.

\begin{figure}[H]
    \centering
    \includegraphics[width=0.7\linewidth]{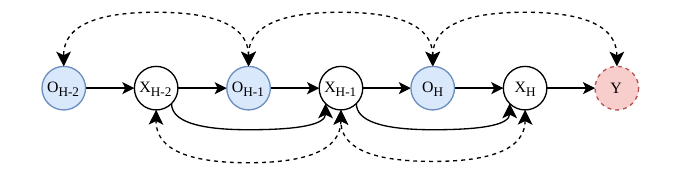}
    \caption{Causal graph for Example~\ref{ex:pa-plus-necessity}, shown over the final three timesteps. Every edge spans a single timestep, yet the open path $\rho$, of $X_H \leftrightarrow X_{H-1} \leftarrow X_{H-2} \to O_{H-1} \leftrightarrow O_H \leftrightarrow Y$, reaches the non-collider $X_{H-2}$ at lag $-2$, outside the width-$1$ window $\mathrm{Win}_H = \{H-1, H\}$.}
    \label{fig:ccomp-example}
\end{figure}

\begin{example}
\label{ex:pa-plus-necessity}
Consider an SCM unrolled over horizon $H$ with endogenous variables $\{O_t, X_t\}_{t=1}^{H}$ and terminal reward $Y$ (Figure~\ref{fig:ccomp-example}), with directed edges $O_t \to X_t$, $X_t \to O_{t+1}$, and $X_t \to X_{t+1}$, terminal edge $X_H \to Y$, and bidirected edges $O_t \leftrightarrow O_{t+1}$, $X_t \leftrightarrow X_{t+1}$, and $O_H \leftrightarrow Y$ from unobserved confounders, with only $Y \in \mathbf{V}^L$. Every directed and bidirected edge spans exactly one timestep, satisfying the condition that $\mathrm{pa}(V_t)$ lie within $k = 1$ of $V_t, V_t \in \mathbf{V}$.

Now examine action $X_H$ at the final step under the resulting window from Algorithm~\ref{alg:feasible-pi-backdoor}, $\mathrm{Win}_H = \{H-1, H\}$, and consider the path
\[
\rho:\quad X_H \leftrightarrow X_{H-1} \leftarrow X_{H-2} \to O_{H-1} \leftrightarrow O_H \leftrightarrow Y.
\]
On $\rho$, the nodes $X_{H-1}$, $O_{H-1}$, $O_H$ are colliders that are conditioned and thus opened. To block $\rho$ one must condition on $X_{H-2}$. But $X_{H-2}$ lies at timestep $H-2$ outside $\mathrm{Win}_H$, so a window of width $k=1$ omits it and $\rho$ remains open rendering the windowed adjustment set is invalid.

Though $X_{H-2}$ is only a single directed edge removed from the window, it is a parent of the C-component $O_{H-2}, O_{H-1}, O_H, Y$ due to $X_{H-2} \to O_{H-1}$. Hence $X_{H-2} \in \mathrm{pa}^+(\mathbf{V}_H)$, even though no single bidirected edge reaches it from the window. Assumption~\ref{assm:bounded-confounding}(i), which bounds $\mathrm{pa}^+$, is violated by this SCM at $k=1$ as it requires every element of $\mathrm{pa}^+(\mathbf{V}_H)$ including $X_{H-2}$ to occur no earlier than $H-k$, forcing $k \geq 2$. This blocks $\rho$ and deems the windowed adjustment set valid while satisfying Assumption~\ref{assm:bounded-confounding}. Therefore, by bounding $\mathrm{pa}^+(V)$ instead of $\mathrm{pa}(V)$, collider-connected paths that have causal influence on $V$ beyond the edge length of direct parents are considered. $\hfill \blacksquare$
\end{example}

\newpage

\section{Algorithm Pseudocode}
\label{app:pseudocode}

Algorithms~\ref{alg:causal-sqil} and~\ref{alg:causal-iqlearn} give the full training procedures for Causal SQIL and Causal IQ-Learn, respectively. Both algorithms share the same causal encoding layer (Algorithm~\ref{alg:feasible-pi-backdoor}) and SAC-style actor update, differing only in how the critic is trained. In Causal SQIL the critic minimizes a standard soft Bellman residual on binary-labeled transitions, whereas in Causal IQ-Learn it minimizes a chi-squared divergence objective on implicit rewards. All other components (entropy tuning, soft target updates, replay buffer management, and the causal adjustment algorithm) are identical.

\begin{algorithm}[H]
\caption{Causal SQIL}
\label{alg:causal-sqil}
\begin{algorithmic}[1]

\REQUIRE Expert demonstrations $\mathcal{D}_{\mathrm{exp}}$, causal graph $\mathcal{G}$, window size $k$, horizon $H$, discount $\gamma$, soft update rate $\tau$, batch size $B$.

\STATE Compute windowed adjustment sets $\{\mathbf{Z}_t^H\}_{t=0}^{H-1}$ and window specification $S$ via Algorithm~\ref{alg:feasible-pi-backdoor}.
\STATE Build encoder $\textsc{Encode}(\mathbf{s}, t)$ from $S$: concatenates values of $V \in \mathbf{V}^O \cap \mathbf{Z}_t^H$ at lags $\{0, -1, \dots, -k\}$ relative to $t$, zero-padding when $t < k$.

\STATE Initialize actor $\pi_\phi$, twin Q-networks $Q_{\theta_1}, Q_{\theta_2}$, target networks $Q_{\bar{\theta}_1} \leftarrow Q_{\theta_1}$, $Q_{\bar{\theta}_2} \leftarrow Q_{\theta_2}$.
\STATE Initialize entropy coefficient $\log \alpha \leftarrow 0$, target entropy $\bar{\mathcal{H}} \leftarrow -|\mathbf{A}|$.

\STATE $\mathcal{B}_{\mathrm{exp}} \leftarrow \varnothing$, \; $\mathcal{B}_{\pi} \leftarrow \varnothing$.
\FOR{each transition $(\mathbf{s}_t, \mathbf{x}_t, \mathbf{s}_{t+1}, d_t)$ in $\mathcal{D}_{\mathrm{exp}}$}
    \STATE $\mathcal{B}_{\mathrm{exp}} \leftarrow \mathcal{B}_{\mathrm{exp}} \cup \{(\textsc{Encode}(\mathbf{s}_t, t),\; \mathbf{x}_t,\; r{=}1,\; \textsc{Encode}(\mathbf{s}_{t+1}, t{+}1),\; d_t)\}$
\ENDFOR

\FOR{each episode $e = 1, 2, \dots$}
    \STATE Reset environment, observe $\mathbf{s}_0$.
    \FOR{$t = 0, \dots, H-1$}
        \STATE $\mathbf{z}_t \leftarrow \textsc{Encode}(\mathbf{s}_t, t)$
        \STATE Sample $\mathbf{x}_t \sim \pi_\phi(\cdot \mid \mathbf{z}_t)$, execute, observe $\mathbf{s}_{t+1}$, $d_t$.
        \STATE $\mathcal{B}_{\pi} \leftarrow \mathcal{B}_{\pi} \cup \{(\mathbf{z}_t,\; \mathbf{x}_t,\; r{=}0,\; \textsc{Encode}(\mathbf{s}_{t+1}, t{+}1),\; d_t)\}$
    \ENDFOR

    \FOR{each gradient step}
        \STATE Sample $\{(\mathbf{z}, \mathbf{x}, r, \mathbf{z}', d)\}_{i=1}^{B}$ with $B/2$ from $\mathcal{B}_{\mathrm{exp}}$ and $B/2$ from $\mathcal{B}_{\pi}$.

        \STATE \textit{// Critic update}
        \STATE $\mathbf{x}' \sim \pi_\phi(\cdot \mid \mathbf{z}')$
        \STATE $y \leftarrow r + \gamma (1 - d)\big(\min_{j} Q_{\bar{\theta}_j}(\mathbf{z}', \mathbf{x}') - \alpha \log \pi_\phi(\mathbf{x}' \mid \mathbf{z}')\big)$
        \STATE $\mathcal{L}_Q \leftarrow \frac{1}{2}\sum_{j=1}^{2} \big\| Q_{\theta_j}(\mathbf{z}, \mathbf{x}) - y \big\|^2$
        \STATE Update $\theta_1, \theta_2$ by $\nabla_{\theta} \mathcal{L}_Q$.

        \STATE \textit{// Actor update}
        \STATE $\tilde{\mathbf{x}} \sim \pi_\phi(\cdot \mid \mathbf{z})$ \quad (reparameterized)
        \STATE $\mathcal{L}_\pi \leftarrow \mathbb{E}\big[\alpha \log \pi_\phi(\tilde{\mathbf{x}} \mid \mathbf{z}) - \min_{j} Q_{\theta_j}(\mathbf{z}, \tilde{\mathbf{x}})\big]$
        \STATE Update $\phi$ by $\nabla_\phi \mathcal{L}_\pi$.

        \STATE \textit{// Entropy tuning}
        \STATE $\mathcal{L}_\alpha \leftarrow -\log \alpha \;\mathbb{E}\big[\log \pi_\phi(\tilde{\mathbf{x}} \mid \mathbf{z}) + \bar{\mathcal{H}}\big]$
        \STATE Update $\log \alpha$ by $\nabla \mathcal{L}_\alpha$, \; $\alpha \leftarrow \exp(\log \alpha)$.

        \STATE \textit{// Soft target update}
        \STATE $\bar{\theta}_j \leftarrow \tau\,\theta_j + (1-\tau)\,\bar{\theta}_j$ \quad for $j = 1, 2$.
    \ENDFOR
\ENDFOR

\STATE \textbf{return} $\pi_\phi$.

\end{algorithmic}
\end{algorithm}

\begin{algorithm}[H]
\caption{Causal IQ-Learn}
\label{alg:causal-iqlearn}
\begin{algorithmic}[1]

\REQUIRE Expert demonstrations $\mathcal{D}_{\mathrm{exp}}$, causal graph $\mathcal{G}$, window size $k$, horizon $H$, discount $\gamma$, soft update rate $\tau$, batch size $B$, number of value samples $K$.

\STATE Compute adjustment sets and build encoder $\textsc{Encode}$ as in Algorithm~\ref{alg:causal-sqil}, steps 1--2.

\STATE Initialize actor $\pi_\phi$, twin Q-networks $Q_{\theta_1}, Q_{\theta_2}$, target networks $Q_{\bar{\theta}_1} \leftarrow Q_{\theta_1}$, $Q_{\bar{\theta}_2} \leftarrow Q_{\theta_2}$.
\STATE Initialize entropy coefficient $\log \alpha \leftarrow 0$, target entropy $\bar{\mathcal{H}} \leftarrow -|\mathbf{A}|$.

\STATE $\mathcal{B}_{\mathrm{exp}} \leftarrow \varnothing$, \; $\mathcal{B}_{\pi} \leftarrow \varnothing$.
\STATE Encode expert demonstrations into $\mathcal{B}_{\mathrm{exp}}$ as $\{(\mathbf{z}_t, \mathbf{x}_t, \mathbf{z}_{t+1}, d_t)\}$ (no reward labels).

\FOR{each episode $e = 1, 2, \dots$}
    \STATE Roll out $\pi_\phi$ in the environment with causal encoding (as in Algorithm~\ref{alg:causal-sqil}, steps 9--13). Store transitions in $\mathcal{B}_{\pi}$.

    \FOR{each gradient step}
        \STATE Sample $\{(\mathbf{z}^e, \mathbf{x}^e, \mathbf{z}'^e, d^e)\}_{i=1}^{B/2}$ from $\mathcal{B}_{\mathrm{exp}}$ and $\{(\mathbf{z}^p, \mathbf{x}^p, \mathbf{z}'^p, d^p)\}_{i=1}^{B/2}$ from $\mathcal{B}_{\pi}$.
        \STATE Let $(\mathbf{z}, \mathbf{x}, \mathbf{z}', d)$ denote the concatenation of expert and policy batches.

        \STATE \textit{// Compute soft state value via Monte Carlo}
        \STATE $\{\tilde{\mathbf{x}}_m\}_{m=1}^{K} \sim \pi_\phi(\cdot \mid \mathbf{z}')$
        \STATE $\bar{V}(\mathbf{z}') \leftarrow \frac{1}{K} \sum_{m=1}^{K} \big[\min_{j} Q_{\bar{\theta}_j}(\mathbf{z}', \tilde{\mathbf{x}}_m) - \alpha \log \pi_\phi(\tilde{\mathbf{x}}_m \mid \mathbf{z}')\big]$

        \STATE \textit{// Implicit reward}
        \STATE $\hat{r}_j(\mathbf{z}, \mathbf{x}) \leftarrow Q_{\theta_j}(\mathbf{z}, \mathbf{x}) - \gamma(1 - d)\,\bar{V}(\mathbf{z}')$ \quad for $j = 1, 2$

        \STATE \textit{// Critic update (chi-squared divergence)}
        \STATE $\mathcal{L}_{Q_j} \leftarrow -\mathbb{E}_{\mathrm{exp}}\big[\hat{r}_j(\mathbf{z}^e, \mathbf{x}^e)\big] + \frac{1}{2}\,\mathbb{E}_{\mathrm{all}}\big[\hat{r}_j(\mathbf{z}, \mathbf{x})^2\big]$ \quad for $j = 1, 2$
        \STATE Update $\theta_1, \theta_2$ by $\nabla_{\theta} \mathcal{L}_{Q}$.

        \STATE \textit{// Actor update (identical to Algorithm~\ref{alg:causal-sqil}, steps 21--24)}
        \STATE $\tilde{\mathbf{x}} \sim \pi_\phi(\cdot \mid \mathbf{z})$, \; $\mathcal{L}_\pi \leftarrow \mathbb{E}\big[\alpha \log \pi_\phi(\tilde{\mathbf{x}} \mid \mathbf{z}) - \min_{j} Q_{\theta_j}(\mathbf{z}, \tilde{\mathbf{x}})\big]$
        \STATE Update $\phi$ by $\nabla_\phi \mathcal{L}_\pi$.

        \STATE \textit{// Entropy tuning and soft target update (identical to Algorithm~\ref{alg:causal-sqil}, steps 26--29)}
        \STATE Update $\log \alpha$, $\alpha$, and $\bar{\theta}_1, \bar{\theta}_2$.
    \ENDFOR
\ENDFOR

\STATE \textbf{return} $\pi_\phi$.

\end{algorithmic}
\end{algorithm}

\newpage

For completeness, we reproduce the \textsc{FindOX} algorithm of \citet{kumor2021sequential} (Algorithm~1 in that paper), which returns the maximal set $\mathbf{O}_X \subseteq \mathbf{V}^O$ from which sequential $\pi$-backdoor admissible sets can be constructed. Given the causal diagram $\mathcal{G}$, action set $\mathbf{X}$, and target $Y$, \textsc{FindOX} iteratively grows $\mathbf{O}_X$ by checking, for each observed node, whether a valid backdoor adjustment exists that would make that node a non-ancestor of $Y$ in the manipulated graph. A sequential $\pi$-backdoor exists for $(\mathcal{G}, \mathbf{X}, Y)$ if and only if $\mathbf{X} \subseteq \mathbf{O}_X$ \citep[Theorem~3.1]{kumor2021sequential}. Once $\mathbf{O}_X$ is obtained, the per-action adjustment sets $\mathbf{Z}_t$ are constructed from the Markov boundary of $\mathbf{O}_X$ in $\mathcal{G}^Y_{\mathbf{X}'}$ (where $\mathbf{X}' = \mathbf{O}_X \cap \mathbf{X}$), intersected with $\mathrm{before}(\mathbf{X}_t)$. Our windowed adjustment procedure (Algorithm~\ref{alg:feasible-pi-backdoor}) invokes \textsc{FindOX} on a short-horizon proxy graph $\mathcal{G}_k$ rather than the full-horizon graph $\mathcal{G}_H$.

\begin{algorithm}[H]
\caption{\textsc{FindOX} \citep{kumor2021sequential}: Find largest valid $\mathbf{O}_X$ in ancestral graph of $Y$}
\label{alg:findox}
\begin{algorithmic}[1]

\REQUIRE Causal diagram $\mathcal{G}$, action set $\mathbf{X}$, target $Y$.

\STATE \textbf{function} \textsc{HasValidAdjustment}($\mathcal{G}, \mathbf{O}^X, O_i, X_i$)
\STATE \quad $C \leftarrow$ the C-component of $O_i$ in $\mathcal{G}^Y$
\STATE \quad $\mathcal{G}_C \leftarrow$ the subgraph of $\mathcal{G}^Y$ containing only $\mathrm{Pa}^+(C)$ and intermediate latent variables
\STATE \quad $\mathbf{O}^C \leftarrow C \setminus (\mathbf{O}^X \cup \{O_i\})$ \hfill \textit{// Elements of C-component that might be ancestors of $Y$ in $\mathcal{G}'_i$}
\STATE \quad \textbf{return} $(O_i \perp\!\!\!\perp \mathbf{O}^C \mid \mathbf{O}^C \cap \mathrm{before}(X_i))$ in $\mathcal{G}_C$

\STATE

\STATE \textbf{function} \textsc{FindOX}($\mathcal{G}, \mathbf{X}, Y$)
\STATE \quad $\mathscr{O}^X \leftarrow$ empty map from elements of $\mathbf{V}^O$ to elements of $\mathbf{X}$
\STATE \quad \textbf{repeat}
\STATE \quad\quad \textbf{for} $O_i \in \mathbf{V}^O$ of $\mathcal{G}^Y$ (ancestral graph of $Y$) in reverse temporal order \textbf{do}
\STATE \quad\quad\quad \textbf{if} $|\mathrm{ch}^+(O_i)| > 0$ and $\mathrm{ch}^+(O_i) \subseteq \mathrm{keys}(\mathscr{O}^X)$ \textbf{then}
\STATE \quad\quad\quad\quad $X_i \leftarrow$ earliest element of $\mathscr{O}^X[\mathrm{ch}^+(O_i)]$ in temporal order
\STATE \quad\quad\quad\quad \textbf{if} \textsc{HasValidAdjustment}($\mathcal{G}, \mathrm{keys}(\mathscr{O}^X), O_i, X_i$) \textbf{then}
\STATE \quad\quad\quad\quad\quad $\mathscr{O}^X[O_i] \leftarrow X_i$
\STATE \quad\quad\quad \textbf{else if} $O_i \in \mathbf{X}$ and \textsc{HasValidAdjustment}($\mathcal{G}, \mathrm{keys}(\mathscr{O}^X), O_i, O_i$) \textbf{then}
\STATE \quad\quad\quad\quad $\mathscr{O}^X[O_i] \leftarrow O_i$
\STATE \quad \textbf{while} $|\mathscr{O}^X|$ changed in most recent pass
\STATE \quad \textbf{return} $\mathrm{keys}(\mathscr{O}^X)$

\end{algorithmic}
\end{algorithm}

\newpage

\section{Environment Details}
\label{app:envs}

We describe each confounded environment in detail, including the causal structure, the confounding mechanism, and the observation partition. Table~\ref{tab:env-summary} summarizes the key dimensions and setup.

\paragraph{Confounded AntMaze.}
The base task is goal-conditioned navigation in a windy maze using an 8-DoF ant robot as described in Example~\ref{ex:antmaze-scm}. The imitator does not observe $\mathbf{O}$, or torso orientation; to compensate, a 2D compass reading $\mathbf{W}$ is added providing a noisy surrogate for heading. The latent wind field $\mathbf{U}$ follows a piecewise-constant gust process and affects both the dynamics and the reward function. The compass is prone to distributional shift between expert and imitator environments due to changing influence from $\mathbf{U}$. Causally unaware methods that condition on $\mathbf{W}$ conflate the wind's influence on the compass with the ant's true heading, learning policies that turn into walls when the wind changes direction. Causal methods exclude $\mathbf{W}$ from their adjustment set and learn to navigate the maze using position, joint angles, and velocities alone, without explicit orientation information.

\begin{figure}[H]
    \centering
    \includegraphics[width=0.7\linewidth]{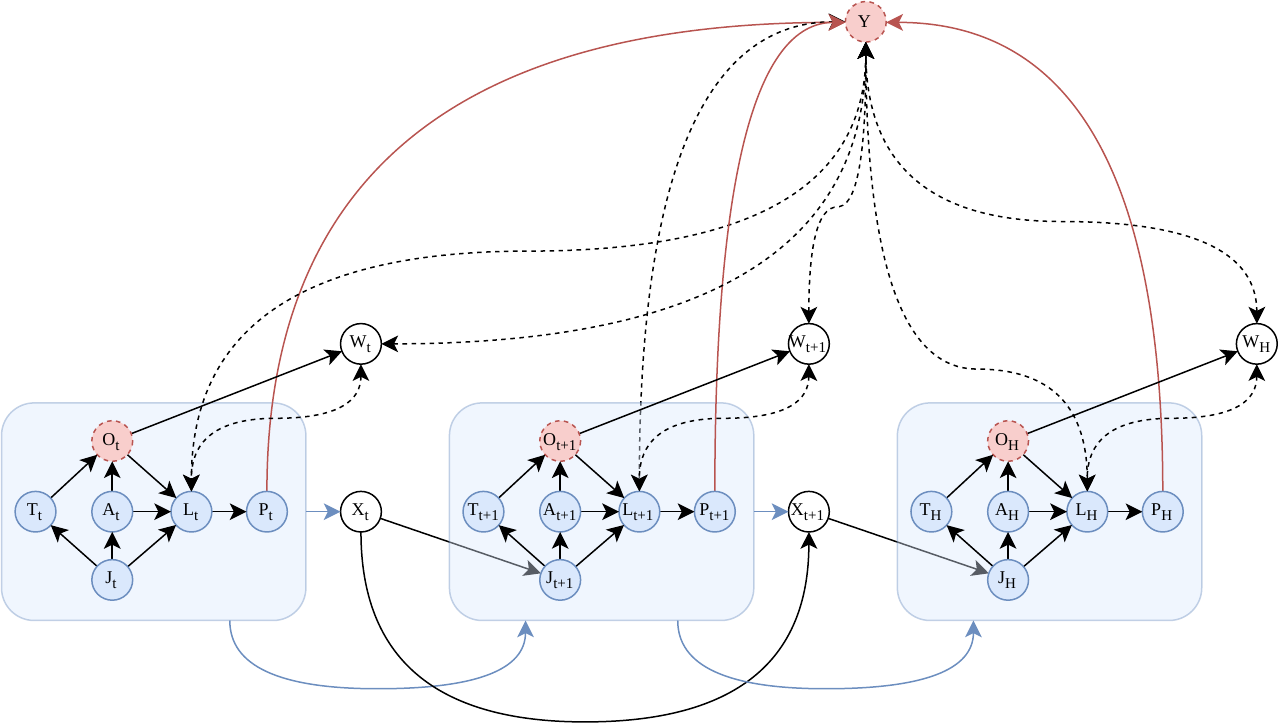}
    \caption{Confounded AntMaze. $\mathbf{P}$ is global position, $\mathbf{L}$ is torso linear velocity, $\mathbf{O}$ is torso orientation, $\mathbf{A}$ is joint angles, $\mathbf{T}$ is torso angular velocity, $\mathbf{J}$ is joint angular velocities, $\mathbf{U}$ is the latent wind field, $\mathbf{W}$ is the compass, $\mathbf{X}$ is joint torques, and $\mathbf{Y}$ is the latent terminal reward.}
    \label{fig:ant-causal-graph}
\end{figure}

\begin{figure}[H]
    \centering
    \includegraphics[width=0.5\linewidth]{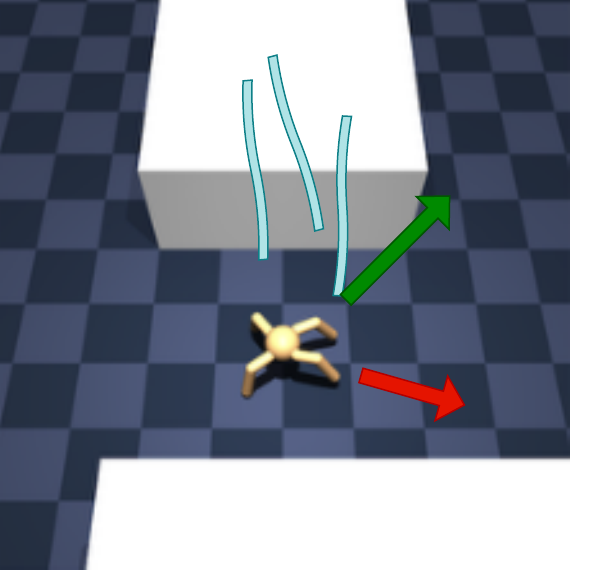}
    \caption{Visualization for Confounded AntMaze, demonstrating the effect of latent winds $\mathbf{U}$ (blue) on the heading compass $\mathbf{W}$ (red) despite the original heading $\mathbf{O}$ (green) being completely different.}
    \label{fig:visual-antmaze}
\end{figure}

\paragraph{Confounded HumanoidMaze.}
The base task is goal-conditioned navigation in a maze using a 21-DoF humanoid robot. The imitator does not observe $\mathbf{C}$, the 3D center-of-mass velocity; to compensate, a 2D ground-vibration reading $\mathbf{W}$ is added that provides a noisy surrogate for locomotion velocity. The latent seismic tremor $\mathbf{U}$ follows a piecewise-constant impulse process and affects both the dynamics (applying external force to the torso) and the reward function. The vibration sensor is prone to distributional shift between expert and imitator environments due to changing influence from $\mathbf{U}$. Causally unaware methods that condition on $\mathbf{W}$ conflate the tremor's influence on the vibration reading with the humanoid's true velocity, learning policies that stumble or over-correct when the tremor changes direction. Causal methods exclude $\mathbf{W}$ from their adjustment set and must learn to navigate the maze using position, joint angles, head height, extremity positions, torso orientation, and joint velocities alone, without explicit velocity information. The confounding effect of the seismic tremor is less disruptive than that of the wind field in AntMaze.

\begin{figure}[H]
    \centering
    \includegraphics[width=0.7\linewidth]{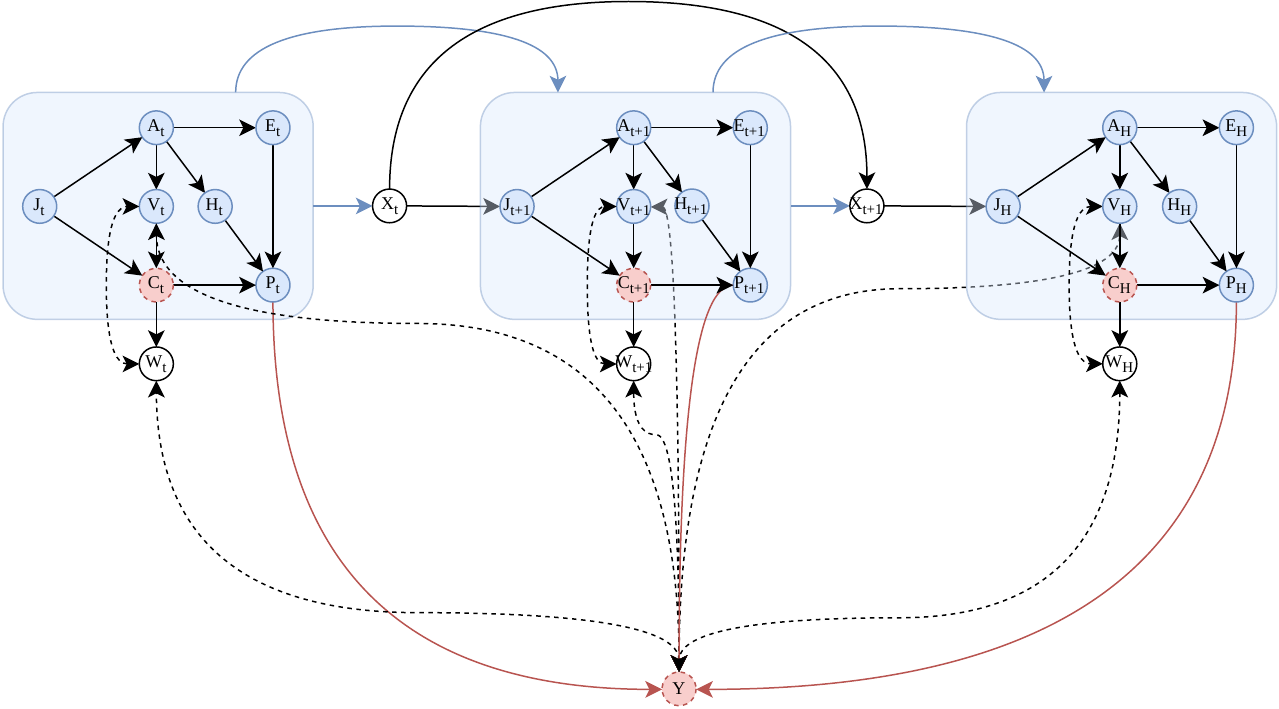}
    \caption{Confounded HumanoidMaze. $\mathbf{P}$ is global position, $\mathbf{A}$ is joint angles, $\mathbf{H}$ is head height, $\mathbf{E}$ is extremities in torso frame, $\mathbf{V}$ is torso vertical, $\mathbf{C}$ is center-of-mass gravity, $\mathbf{J}$ is joint velocities, $\mathbf{U}$ is the latent seismic tremor, $\mathbf{W}$ is the vibration sensor, $\mathbf{X}$ is joint torques, and $\mathbf{Y}$ is the latent reward.}
    \label{fig:humanoid-causal-graph}
\end{figure}

\begin{figure}[H]
    \centering
    \includegraphics[width=0.5\linewidth]{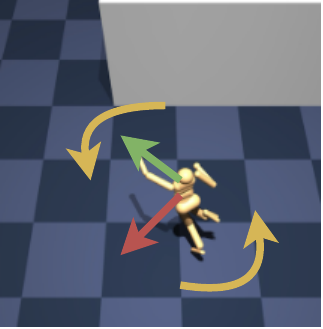}
    \caption{Visualization for Confounded HumanoidMaze, demonstrating the effect of latent seismic tremors $\mathbf{U}$ (yellow) on the perceived ground vibration sensor $\mathbf{W}$ (red) despite the true center-of-mass velocity $\mathbf{C}$ (green) pointing in a completely different direction.}
    \label{fig:env-humanoidmaze}
\end{figure}

\begin{table}[H]
\centering
\begin{tabular}{lccccc}
\hline
\textbf{Environment} & $H$ & $|\mathbf{V}|$ & $|\mathbf{X}|$ & \textbf{Latent Confounder} & \textbf{Collider} \\
\hline
AntMaze  & 1000 & 35 & 8  & Wind field & Compass \\
HumanoidMaze & 2000 & 89 & 21 & Seismic tremors & Vibration \\
\hline
\end{tabular}
\caption{Summary of confounded environments. $|\mathbf{V}|$ denotes the imitator's observation dimensionality (excluding hidden variables). $|\mathbf{X}|$ is the action dimensionality.}
\label{tab:env-summary}
\end{table}

\section{Hyperparameters and Training Details}
\label{app:hparams}

Table~\ref{tab:hparams} summarizes the hyperparameters for all algorithms. Causal and non-causal variants of each algorithm use identical hyperparameters and network architectures; the only difference is the input representation (causally-adjusted $\mathbf{Z}_t$ vs.\ full observation $\mathbf{V}^O_t$). Where hyperparameters differ between environment families we write $\mathcal{A}$ (AntMaze) and $\mathcal{H}$ (HumanoidMaze). All tasks used 3 seeds. Note: due to limited computational resources, HumanoidMaze-Large evaluations consisted of $100$ episodes rather than the other tasks' $1000$ episodes, and non-causal SQIL for HumanoidMaze-Large was trained on only $1M$ timesteps instead of $2M$; this decision was informed by the fact that every algorithm in every task reached its best checkpoint before the $1M$ timestep mark during training.

\begin{table}[H]
\centering
\caption{Hyperparameters for all algorithms. Causal and non-causal variants share identical settings. ``---'' denotes a hyperparameter not applicable to that algorithm. $\mathcal{A}$\,/\,$\mathcal{H}$ distinguishes AntMaze and HumanoidMaze values where they differ.}
\label{tab:hparams}
\resizebox{\textwidth}{!}{
\begin{tabular}{lcccc}
\hline
\textbf{Hyperparameter} & \textbf{BC} & \textbf{GAIL} & \textbf{SQIL} & \textbf{IQ-Learn} \\
\hline
\multicolumn{5}{l}{\textit{Network Architecture}} \\
Hidden dimension & 256 & 256 & 256 & 256 \\
Actor residual blocks & 4 & 3 & 3 & 3 \\
Actor dropout & 0.0 & 0.05 & 0.05 & 0.05 \\
Layer normalization & $\checkmark$ & $\checkmark$ & $\checkmark$ & $\checkmark$ \\
Activation & SiLU & SiLU & SiLU & SiLU \\
Output squashing & Tanh & Tanh & Tanh & Tanh \\
\hline
\multicolumn{5}{l}{\textit{Optimization}} \\
Optimizer & Adam & Adam & Adam & Adam \\
Actor learning rate & $3 \times 10^{-4}$ & $1 \times 10^{-4}$ & $3 \times 10^{-4}$ & $3 \times 10^{-4}$ \\
Critic learning rate & --- & $3 \times 10^{-4}$ & $3 \times 10^{-4}$ & $3 \times 10^{-4}$ \\
Batch size & 2048 & 1024 & 256 & 256 \\
Discount $\gamma$ & --- & 0.99 & 0.99 & 0.99 \\
Max gradient norm & --- & 0.5 & 1.0 & 1.0 \\
\hline
\multicolumn{5}{l}{\textit{BC-Specific}} \\
Loss function & Huber & --- & --- & --- \\
Training epochs & $\mathcal{A}$: 100\;/\;$\mathcal{H}$-Med: 100\;/\;$\mathcal{H}$-Large: 200 & --- & --- & --- \\
Early stopping patience & $\mathcal{A}$: 15\;/\;$\mathcal{H}$-Med: 30\;/\;$\mathcal{H}$-Large: 15 & --- & --- & --- \\
Validation fraction & 0.2 & --- & --- & --- \\
\hline
\multicolumn{5}{l}{\textit{GAIL-Specific (PPO + Discriminator)}} \\
PPO clip ratio $\epsilon$ & --- & 0.2 & --- & --- \\
GAE $\lambda$ & --- & 0.95 & --- & --- \\
PPO epochs per round & --- & $\mathcal{A}$: 4\;/\;$\mathcal{H}$: 2 & --- & --- \\
PPO entropy coefficient & --- & $10^{-2}$ & --- & --- \\
Value loss coefficient & --- & 0.5 & --- & --- \\
Normalize advantages & --- & $\checkmark$ & --- & --- \\
Discriminator learning rate & --- & $3 \times 10^{-4}$ & --- & --- \\
Discriminator dropout & --- & 0.2 & --- & --- \\
Discriminator updates per round & --- & 2 & --- & --- \\
Discriminator minibatch size & --- & 1024 & --- & --- \\
Gradient penalty $\lambda_{\mathrm{GP}}$ & --- & 5.0 & --- & --- \\
Episodes per round & --- & $\mathcal{A}$: 20\;/\;$\mathcal{H}$: 10 & --- & --- \\
Total training rounds & --- & $\mathcal{A}$: 500\;/\;$\mathcal{H}$: 200 & --- & --- \\
Disc.\ LR scheduler & --- & StepLR($100$, $0.5$) & --- & --- \\
\hline
\multicolumn{5}{l}{\textit{SQIL / IQ-Learn (SAC-Based)}} \\
Soft update rate $\tau$ & --- & --- & 0.005 & 0.005 \\
Entropy coefficient $\alpha$ & --- & --- & auto-tuned & auto-tuned \\
Entropy learning rate & --- & --- & $3 \times 10^{-4}$ & $3 \times 10^{-4}$ \\
Target entropy $\bar{\mathcal{H}}$ & --- & --- & $-|\mathbf{A}|$ & $-|\mathbf{A}|$ \\
$\alpha$ clamp range & --- & --- & $[e^{-\log 1000},\; e^{-\log 10}]$ & $[e^{-\log 1000},\; e^{-\log 10}]$ \\
Replay buffer capacity & --- & --- & $10^{6}$ & $10^{6}$ \\
Expert buffer ratio & --- & --- & 0.5 & 0.5 \\
Random exploration steps & --- & --- & 5000 & 5000 \\
Update-to-data ratio & --- & --- & $\mathcal{A}$: 0.25\;/\;$\mathcal{H}$: 0.5 & $\mathcal{A}$: 0.25\;/\;$\mathcal{H}$: 0.5 \\
Total timesteps & --- & --- & $2 \times 10^{6}$ & $2 \times 10^{6}$ \\
V estimation & --- & --- & --- & $\mathcal{A}$: MC ($K{=}16$)\;/\;$\mathcal{H}$: single-sample \\
Critic LR scheduler & --- & --- & CosineAnnealing & CosineAnnealing \\
\hline
\multicolumn{5}{l}{\textit{Shared}} \\
Lookback window $k$ & $\mathcal{A}$: 10\;/\;$\mathcal{H}$: 2 & $\mathcal{A}$: 10\;/\;$\mathcal{H}$: 2 & $\mathcal{A}$: 10\;/\;$\mathcal{H}$: 2 & $\mathcal{A}$: 10\;/\;$\mathcal{H}$: 2 \\
Max episode steps & $\mathcal{A}$: 1000\;/\;$\mathcal{H}$: 2000 & $\mathcal{A}$: 1000\;/\;$\mathcal{H}$: 2000 & $\mathcal{A}$: 1000\;/\;$\mathcal{H}$: 2000 & $\mathcal{A}$: 1000\;/\;$\mathcal{H}$: 2000 \\
\hline
\end{tabular}
}
\end{table}

\paragraph{Expert construction.}
Expert policies for the maze tasks are constructed using offline-to-online RL. For each environment, we begin by training a goal-conditioned behavioral cloning (BC) policy on provided demonstrations from the base OGBench dataset. This policy provides an initialization that captures the global structure of the task, but it is not yet ready for the confounders introduced in the modified environment. To obtain an expert that reflects performance under the confounded dynamics, we then fine-tune this BC policy through a period of off-policy actor--critic training using TD3 \citep{fujimoto2018td3}. Reward shaping is added optionally during fine-tuning to compensate for the sparse reward signals. The resulting expert policy is capable of operating effectively under the latent disturbances, partial observability, and altered transition dynamics of the confounded environment.

\paragraph{Network architecture.}
All algorithms share a residual MLP backbone for the actor. The actor network maps the causally-adjusted encoding $\mathbf{z}_t$ to actions through a linear projection into hidden dimension $256$, followed by residual blocks, and a final linear output with tanh squashing to the action bounds. Each residual block consists of LayerNorm, SiLU, a linear layer, a second LayerNorm, SiLU, dropout, and a second linear layer, with a skip connection from input to output. For GAIL, SQIL, and IQ-Learn, the actor outputs the mean of a squashed Gaussian with a state-independent learnable log-variance; for BC, the actor outputs a deterministic action. The Q-networks (SQIL and IQ-Learn) use the same residual MLP architecture, taking concatenated $(\mathbf{z}_t, \mathbf{x}_t)$ as input and producing a scalar output. GAIL's value network (critic) uses a simpler three-layer MLP with ReLU activations, taking $\mathbf{z}_t$ as input. GAIL's discriminator similarly uses a three-layer MLP with ReLU activations, skip connections between layers, higher dropout ($0.2$), and binary cross-entropy loss to classify $(\mathbf{z}_t, \mathbf{x}_t)$ pairs as expert or policy-generated.

\paragraph{SQIL training details.}
Causal SQIL uses SAC with twin Q-networks and soft target updates ($\tau = 0.005$). The entropy coefficient $\alpha$ is automatically tuned toward a target entropy of $-|\mathbf{X}|$ and clamped to $[\exp(-\log 1000),\; \exp(-\log 10)] \approx [0.001, 0.1]$. The replay buffer is split equally: $50\%$ capacity for expert transitions (labeled $r = 1$) and $50\%$ for policy transitions ($r = 0$), with each training batch sampled $50/50$ from both halves. The first $5{,}000$ timesteps use random actions for exploration before policy rollouts begin. Training runs for $2 \times 10^{6}$ environment steps. The update-to-data (UTD) ratio is $0.25$ for AntMaze tasks (one gradient step per four environment steps) and $0.5$ for HumanoidMaze tasks. A cosine annealing schedule is applied to the critic learning rate.

\paragraph{IQ-Learn training details.}
Causal IQ-Learn shares the SAC actor update and twin Q-network architecture with Causal SQIL. The key difference is the critic loss: instead of minimizing a soft Bellman residual on binary-labeled transitions, IQ-Learn minimizes the chi-squared divergence between the implicit reward distribution under the expert and the combined (expert $+$ policy) data. For AntMaze tasks, the soft state value $V(\mathbf{z}')$ is estimated via $K = 16$ Monte Carlo samples from the current policy; for HumanoidMaze tasks, a single-sample estimate is used (matching the standard SAC target computation), which avoids the high variance of multi-sample estimates in the higher-dimensional action space. The entropy coefficient $\alpha$ is auto-tuned and clamped to $[\exp(-\log 1000),\; \exp(-\log 10)] \approx [0.001, 0.1]$ to prevent entropy collapse or explosion. IQ-Learn uses the same UTD ratios as SQIL ($0.25$ for AntMaze, $0.5$ for HumanoidMaze).

\paragraph{GAIL training details.}
Causal GAIL alternates between on-policy rollouts and discriminator--policy updates. Each round collects $20$ episodes for AntMaze ($10$ for HumanoidMaze), up to $H$ steps each, using the current policy, computes advantages via GAE ($\lambda = 0.95$, $\gamma = 0.99$), and performs PPO updates with clipping ratio $\epsilon = 0.2$, entropy regularization coefficient $10^{-2}$, and value loss coefficient $0.5$. The number of PPO epochs per round is $4$ for AntMaze and $2$ for HumanoidMaze. The discriminator is updated $2$ times per round on minibatches of $1024$ with gradient penalty ($\lambda_{\mathrm{GP}} = 5.0$). The discriminator learning rate follows a StepLR schedule, halving every $100$ rounds. Training runs for $500$ rounds on AntMaze and $200$ rounds on HumanoidMaze.

\paragraph{BC training details.}
Causal BC performs supervised learning on expert demonstrations using Huber loss, optimized with Adam at learning rate $3 \times 10^{-4}$. The actor uses $4$ residual blocks with no dropout, trained for up to $100$ epochs (200 for HumanoidMaze-Large) with early stopping on a held-out $20\%$ validation split. The early stopping patience is $15$ for all tasks except HumanoidMaze-Medium, which uses a patience of $30$. Training uses large batches of $2048$ to reduce variance in the gradient estimates. BC is the only algorithm that does not interact with the environment during training; it learns entirely from the static expert dataset.

\section{Additional Results}
\label{app:results}

\subsection{Raw Evaluation Data}
\label{app:raw-data}

\begin{table}[H]
\centering
\caption{Evaluation results on confounded tasks without normalization. Best non-expert result per task in \textbf{bold}.}
\label{tab:raw-results}
\resizebox{\textwidth}{!}{
\begin{tabular}{l|cc|cc|cc|cc}
\hline
& \multicolumn{2}{c}{Ant-Med} & \multicolumn{2}{c}{Ant-Large} & \multicolumn{2}{c}{Hum-Med} & \multicolumn{2}{c}{Hum-Large} \\
\cline{2-3} \cline{4-5} \cline{6-7} \cline{8-9}
\textbf{Algorithm} & $\mathbb{E}[Y]$ & SR\ (\%) & $\mathbb{E}[Y]$ & SR\ (\%) & $\mathbb{E}[Y]$ & SR\ (\%) & $\mathbb{E}[Y]$ & SR\ (\%) \\
\hline
Expert & $-104.0 {\scriptstyle \pm 87.5}$ & $87.6 {\scriptstyle \pm 1.0}$ & $-255.2 {\scriptstyle \pm 122.8}$ & $55.9 {\scriptstyle \pm 1.6}$ & $-522.9 {\scriptstyle \pm 218.4}$ & $33.8 {\scriptstyle \pm 1.5}$ & $-840.6 {\scriptstyle \pm 212.9}$ & $7.0 {\scriptstyle \pm 0.8}$ \\
C-BC & $-125.2 {\scriptstyle \pm 100.9}$ & $77.9 {\scriptstyle \pm 1.3}$ & $-284.9 {\scriptstyle \pm 121.6}$ & $39.8 {\scriptstyle \pm 1.5}$ & $-655.7 {\scriptstyle \pm 172.7}$ & $10.4 {\scriptstyle \pm 1.0}$ & $-858.6 {\scriptstyle \pm 227.9}$ & $\mathbf{8.0} {\scriptstyle \pm 2.7}$ \\
C-GAIL & $-135.8 {\scriptstyle \pm 113.6}$ & $74.1 {\scriptstyle \pm 1.4}$ & $-328.5 {\scriptstyle \pm 122.6}$ & $7.3 {\scriptstyle \pm 0.8}$ & $-747.2 {\scriptstyle \pm 161.9}$ & $0.0 {\scriptstyle \pm 0.0}$ & $-976.3 {\scriptstyle \pm 210.2}$ & $0.0 {\scriptstyle \pm 0.0}$ \\
C-SQIL & $\mathbf{-99.6} {\scriptstyle \pm 85.7}$ & $\mathbf{90.7} {\scriptstyle \pm 0.9}$ & $-279.9 {\scriptstyle \pm 125.9}$ & $45.0 {\scriptstyle \pm 1.6}$ & $\mathbf{-555.5} {\scriptstyle \pm 217.2}$ & $\mathbf{24.7} {\scriptstyle \pm 1.4}$ & $\mathbf{-837.7} {\scriptstyle \pm 226.0}$ & $\mathbf{8.0} {\scriptstyle \pm 2.7}$ \\
C-IQL & $-118.2 {\scriptstyle \pm 101.3}$ & $84.3 {\scriptstyle \pm 1.15}$ & $\mathbf{-258.6} {\scriptstyle \pm 127.9}$ & $\mathbf{58.9} {\scriptstyle \pm 1.6}$ & $-588.8 {\scriptstyle \pm 203.8}$ & $19.1 {\scriptstyle \pm 1.2}$ & $-887.1 {\scriptstyle \pm 220.5}$ & $3.0 {\scriptstyle \pm 1.7}$ \\
BC & $-375.3 {\scriptstyle \pm 119.2}$ & $0.0 {\scriptstyle \pm 0.0}$ & $-470.2 {\scriptstyle \pm 150.2}$ & $0.0 {\scriptstyle \pm 0.0}$ & $-678.4 {\scriptstyle \pm 169.1}$ & $5.4 {\scriptstyle \pm 0.7}$ & $-884.1 {\scriptstyle \pm 204.6}$ & $5.0 {\scriptstyle \pm 2.2}$ \\
GAIL & $-364.9 {\scriptstyle \pm 112.2}$ & $0.0 {\scriptstyle \pm 0.0}$ & $-398.4 {\scriptstyle \pm 132.9}$ & $0.0 {\scriptstyle \pm 0.0}$ & $-747.7 {\scriptstyle \pm 162.5}$ & $0.0 {\scriptstyle \pm 0.0}$ & $-977.0 {\scriptstyle \pm 208.9}$ & $0.0 {\scriptstyle \pm 0.0}$ \\
SQIL & $-370.7 {\scriptstyle \pm 119.1}$ & $0.0 {\scriptstyle \pm 0.0}$ & $-484.2 {\scriptstyle \pm 151.5}$ & $0.0 {\scriptstyle \pm 0.0}$ & $-706.2 {\scriptstyle \pm 157.0}$ & $0.1 {\scriptstyle \pm 0.1}$ & $-896.3 {\scriptstyle \pm 194.9}$ & $0.0 {\scriptstyle \pm 0.0}$ \\
IQL & $-372.5 {\scriptstyle \pm 120.8}$ & $0.0 {\scriptstyle \pm 0.0}$ & $-468.1 {\scriptstyle \pm 153.5}$ & $0.0 {\scriptstyle \pm 0.0}$ & $-697.6 {\scriptstyle \pm 169.7}$ & $2.4 {\scriptstyle \pm 0.5}$ & $-906.2 {\scriptstyle \pm 193.3}$ & $0.0 {\scriptstyle \pm 0.0}$ \\
\hline
\end{tabular}
}
\end{table}

\paragraph{Remark on expert success rates.}
We note that the raw expert success rates in Table~\ref{tab:raw-results} are substantially lower than those typically reported in standard MuJoCo benchmarks. This is by design: our confounded environments are partially observable even to the expert, as the exogenous disturbances $\mathbf{U}$ (represented by bidirected arrows in the causal diagrams) are unobserved by all agents, including the expert. The unpredictable external forces generated by $\mathbf{U}$ make consistent task completion inherently difficult regardless of the agent's sensory access or learning algorithm. In this sense, the environments are designed to reflect realistic deployment conditions in which no agent has full knowledge of the environment dynamics, rather than the fully observed settings common in standard benchmarks. Nevertheless, the expert has strictly more information than the imitator (observing $\mathbf{V}^O \cup \mathbf{V}^L$ rather than $\mathbf{V}^O$ alone) and serves as a meaningful upper bound on what can be achieved given the imitator's observation set. The normalized metrics in Table~\ref{tab:main-results} should therefore be interpreted relative to this upper bound: an algorithm recovering, say, $80\%$ of the expert's success rate is operating near the frontier of what is achievable given the imitator's partial observability, not at $80\%$ of a trivially solvable task. Cases in which a causal method exceeds $100\%$ normalized performance (e.g., Causal SQIL on AntMaze-Medium, Causal IQ-Learn on AntMaze-Large) indicate that the imitator's learned policy is slightly more robust to the stochastic disturbances than the expert policy obtained via offline-to-online RL, likely because the off-policy Q-learning objective implicitly averages over the disturbance distribution encountered during training rather than committing to the point estimates used during expert fine-tuning.

\subsection{Episode Lengths}
\label{app:episode-lengths}

We report the mean episode length for successfully solved episodes in Table~\ref{tab:episode-lengths}. Shorter episodes indicate more efficient navigation. Only algorithms with nonzero success rates are included.

\begin{table}[H]
\centering
\caption{Mean episode length (steps) for successfully solved episodes ($\pm$ std). Only algorithms with nonzero success rates are shown. Best result per task in \textbf{bold}.}
\label{tab:episode-lengths}
\resizebox{\textwidth}{!}{
\begin{tabular}{lcccc}
\hline
\textbf{Algorithm} & \textbf{AntMaze-Med} & \textbf{AntMaze-Large} & \textbf{HumanoidMaze-Med} & \textbf{HumanoidMaze-Large} \\
\hline
Expert & $301 {\scriptstyle \pm 57}$ & $648 {\scriptstyle \pm 99}$ & $976 {\scriptstyle \pm 410}$ & $1463 {\scriptstyle \pm 366}$ \\
C-BC & $318 {\scriptstyle \pm 79}$ & $683 {\scriptstyle \pm 116}$ & $1356 {\scriptstyle \pm 344}$ & $1477 {\scriptstyle \pm 421}$ \\
C-GAIL & $309 {\scriptstyle \pm 75}$ & $709 {\scriptstyle \pm 134}$ & --- & --- \\
C-SQIL & $\mathbf{299} {\scriptstyle \pm 62}$ & $695 {\scriptstyle \pm 119}$ & $\mathbf{912} {\scriptstyle \pm 425}$ & $\mathbf{1251} {\scriptstyle \pm 316}$ \\
C-IQL & $321 {\scriptstyle \pm 90}$ & $\mathbf{636} {\scriptstyle \pm 108}$ & $1024 {\scriptstyle \pm 443}$ & $1848 {\scriptstyle \pm 84}$ \\
BC & --- & --- & $1364 {\scriptstyle \pm 320}$ & $1663 {\scriptstyle \pm 237}$ \\
GAIL & --- & --- & --- & --- \\
SQIL & --- & --- & $1800 {\scriptstyle \pm 0}$ & --- \\
IQL & --- & --- & $1085 {\scriptstyle \pm 396}$ & --- \\
\hline
\end{tabular}
}
\end{table}

\subsection{Wall-Clock Runtime of Full vs.\ Windowed Adjustment}
\label{app:wallclock}

To empirically motivate the windowed approximation of Algorithm~\ref{alg:feasible-pi-backdoor}, we measure the wall-clock time required to compute the per-timestep adjustment sets $\mathbf{Z}_t^k$ under different horizons up to $H=2000$. Runtime for $H=1000, 2000$ reflects full sequential $\pi$-backdoor computation for environments such as AntMaze and HumanoidMaze, while the runtime for smaller $H$ (i.e. $H=2, 10$) is equivalent to that of Algorithm~\ref{alg:feasible-pi-backdoor} which uses $k=2, 10$ respectively in our experiments due to windowed adjustment.

\begin{figure}[H]
    \centering
    \includegraphics[width=0.65\linewidth]{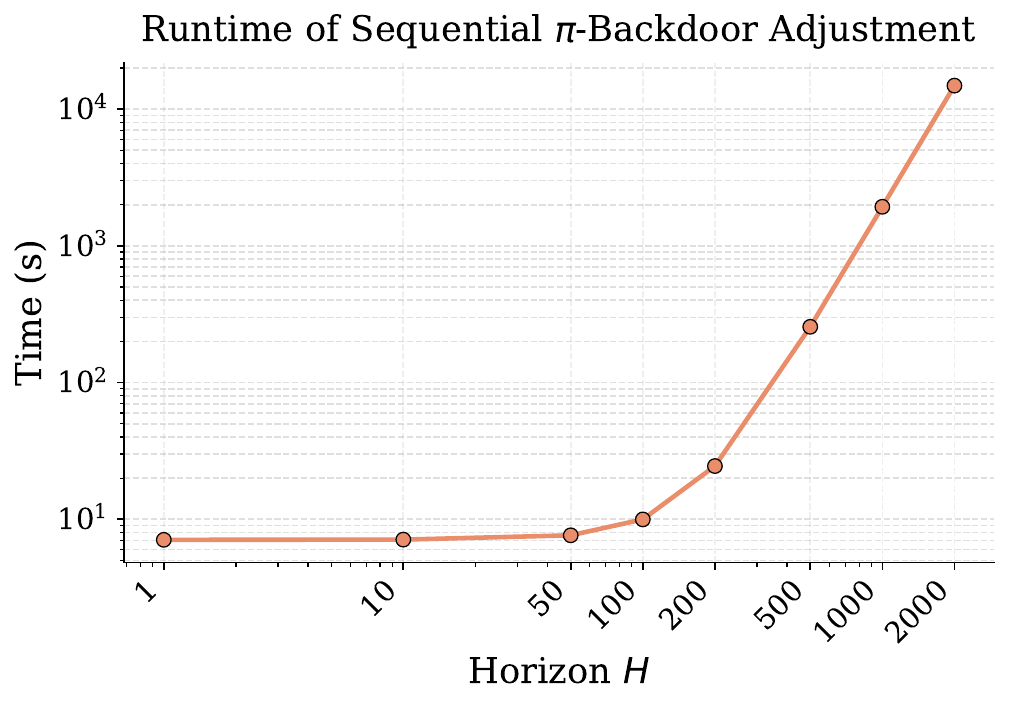}
    \caption{Runtime to compute adjustment sets as a function of horizon $H$ (log--log scale). Algorithm~\ref{alg:feasible-pi-backdoor} ($H=2$ for HumanoidMaze, $H=10$ for AntMaze; Appendix~\ref{app:hparams}) completes in a few seconds, while full-horizon adjustment ($H=1000, 2000$) takes over four hours.}
    \label{fig:wallclock}
\end{figure}

The runtime is roughly constant for small $H$ but grows super-linearly once $H$ exceeds a few hundred, increasing from $7.1$s at $H=1$ to $14{,}854$s ($4.1$ hours) at $H=2000$. By restricting the sequential $\pi$-backdoor computation to a fixed-size window (Theorem~\ref{thm:windowed-adjustment}), Algorithm~\ref{alg:feasible-pi-backdoor} reduces this cost to $O(k)$ independent of $H$, making causal adjustment tractable for the long-horizon tasks in this work.

\subsection{Ablations over Horizon Length and Confounding Strength}
\label{app:ablations-hc}

The main experiments (Table~\ref{tab:main-results}) fix the episode horizon at the default ($H=1000$ for AntMaze, $H=2000$ for HumanoidMaze) and the confounding strength at $100\%$. To test the effect of these settings in higher granularity, we sweep the horizon $H \in \{500, 1000, 1500\}$ (Table~\ref{tab:ablation-horizon}) and the confounding strength $\in \{15\%, 50\%, 85\%\}$ (Table~\ref{tab:ablation-confound}) on both AntMaze tasks (in which confounding strength is represented by heading compass diversion from the true heading) for all eight algorithms; due to limited computational resources, the HumanoidMaze tasks are excluded from these ablations. Both sweeps utilize the same expert as the main experiments, though for $H \in \{500, 1500\}$ expert demonstrations are generated with the appropriate horizon. Note that at $0\%$ confounding strength the compass is an uncorrupted (if noisy) heading sensor, so a non-causal imitator would incur no confounding bias and could even benefit from the additional signal, as seen in the $15\%$-confounded columns. The $H=1000$ and $100\%$-confounded columns reproduce the results reported in Table~\ref{tab:main-results} and are listed for reference. $\mathbb{E}[Y]$ is normalized in the same manner as in Table~\ref{tab:main-results}. Because the reward penalizes every step taken before reaching the goal, it reflects solving speed; thus, an episode solved in a given number of steps yields a larger normalized value under a longer horizon (failed episodes accrue penalties for all $H$ steps), while the anchor changes with the strength setting per configuration. Normalized $\mathbb{E}[Y]$ therefore compares algorithms within a configuration and is generally not comparable across horizons or across strengths.

\begin{table}[H]
\centering
\begin{subtable}{\textwidth}
\centering
\resizebox{\textwidth}{!}{
\begin{tabular}{l|cc|cc|cc}
\hline
& \multicolumn{2}{c}{$H=500$} & \multicolumn{2}{c}{$H=1000$} & \multicolumn{2}{c}{$H=1500$} \\
\cline{2-3} \cline{4-5} \cline{6-7}
\textbf{Algorithm} & $\mathbb{E}[Y]$ & SR\ (\%) & $\mathbb{E}[Y]$ & SR\ (\%) & $\mathbb{E}[Y]$ & SR\ (\%) \\
\hline
C-BC & $90.4 {\scriptstyle \pm 62.4}$ & $78.2 {\scriptstyle \pm 1.3}$ & $250.1 {\scriptstyle \pm 100.9}$ & $77.9 {\scriptstyle \pm 1.3}$ & $406.5 {\scriptstyle \pm 153.1}$ & $81.3 {\scriptstyle \pm 1.2}$ \\
C-GAIL & $76.5 {\scriptstyle \pm 68.9}$ & $64.5 {\scriptstyle \pm 1.4}$ & $239.5 {\scriptstyle \pm 113.6}$ & $74.1 {\scriptstyle \pm 1.4}$ & $400.1 {\scriptstyle \pm 159.4}$ & $80.2 {\scriptstyle \pm 1.3}$ \\
C-SQIL & $98.3 {\scriptstyle \pm 59.3}$ & $\mathbf{87.5} {\scriptstyle \pm 1.0}$ & $\mathbf{275.8} {\scriptstyle \pm 85.7}$ & $\mathbf{90.7} {\scriptstyle \pm 0.9}$ & $431.0 {\scriptstyle \pm 135.8}$ & $87.0 {\scriptstyle \pm 1.1}$ \\
C-IQL & $\mathbf{99.1} {\scriptstyle \pm 59.0}$ & $86.3 {\scriptstyle \pm 1.1}$ & $257.1 {\scriptstyle \pm 101.3}$ & $84.3 {\scriptstyle \pm 1.1}$ & $\mathbf{431.7} {\scriptstyle \pm 134.0}$ & $\mathbf{87.2} {\scriptstyle \pm 1.1}$ \\
BC & $17.1 {\scriptstyle \pm 77.7}$ & $0.0 {\scriptstyle \pm 0.0}$ & $0.0 {\scriptstyle \pm 119.2}$ & $0.0 {\scriptstyle \pm 0.0}$ & $82.3 {\scriptstyle \pm 132.8}$ & $0.5 {\scriptstyle \pm 0.1}$ \\
GAIL & $17.1 {\scriptstyle \pm 77.3}$ & $0.0 {\scriptstyle \pm 0.0}$ & $10.4 {\scriptstyle \pm 112.2}$ & $0.0 {\scriptstyle \pm 0.0}$ & $199.0 {\scriptstyle \pm 170.4}$ & $26.5 {\scriptstyle \pm 1.0}$ \\
SQIL & $6.0 {\scriptstyle \pm 81.9}$ & $0.0 {\scriptstyle \pm 0.0}$ & $4.7 {\scriptstyle \pm 119.1}$ & $0.0 {\scriptstyle \pm 0.0}$ & $2.3 {\scriptstyle \pm 144.4}$ & $0.0 {\scriptstyle \pm 0.0}$ \\
IQL & $0.0 {\scriptstyle \pm 84.3}$ & $0.0 {\scriptstyle \pm 0.0}$ & $2.8 {\scriptstyle \pm 120.8}$ & $0.0 {\scriptstyle \pm 0.0}$ & $0.0 {\scriptstyle \pm 145.2}$ & $0.0 {\scriptstyle \pm 0.0}$ \\
\hline
\end{tabular}
}
\caption{AntMaze-Medium}
\label{tab:ablation-horizon-medium}
\end{subtable}
\vspace{0.75em}
\begin{subtable}{\textwidth}
\centering
\resizebox{\textwidth}{!}{
\begin{tabular}{l|cc|cc|cc}
\hline
& \multicolumn{2}{c}{$H=500$} & \multicolumn{2}{c}{$H=1000$} & \multicolumn{2}{c}{$H=1500$} \\
\cline{2-3} \cline{4-5} \cline{6-7}
\textbf{Algorithm} & $\mathbb{E}[Y]$ & SR\ (\%) & $\mathbb{E}[Y]$ & SR\ (\%) & $\mathbb{E}[Y]$ & SR\ (\%) \\
\hline
C-BC & $61.2 {\scriptstyle \pm 83.9}$ & $\mathbf{0.7} {\scriptstyle \pm 0.3}$ & $199.3 {\scriptstyle \pm 121.6}$ & $39.8 {\scriptstyle \pm 1.5}$ & $271.6 {\scriptstyle \pm 179.6}$ & $25.3 {\scriptstyle \pm 1.4}$ \\
C-GAIL & $49.1 {\scriptstyle \pm 89.4}$ & $0.0 {\scriptstyle \pm 0.0}$ & $155.6 {\scriptstyle \pm 122.6}$ & $7.3 {\scriptstyle \pm 0.8}$ & $257.8 {\scriptstyle \pm 167.3}$ & $16.2 {\scriptstyle \pm 0.9}$ \\
C-SQIL & $63.3 {\scriptstyle \pm 84.2}$ & $0.4 {\scriptstyle \pm 0.1}$ & $204.3 {\scriptstyle \pm 125.9}$ & $45.0 {\scriptstyle \pm 1.6}$ & $\mathbf{330.4} {\scriptstyle \pm 194.4}$ & $\mathbf{46.7} {\scriptstyle \pm 1.6}$ \\
C-IQL & $\mathbf{65.4} {\scriptstyle \pm 82.6}$ & $0.1 {\scriptstyle \pm 0.1}$ & $\mathbf{225.6} {\scriptstyle \pm 127.9}$ & $\mathbf{58.9} {\scriptstyle \pm 1.6}$ & $322.1 {\scriptstyle \pm 181.6}$ & $32.4 {\scriptstyle \pm 1.4}$ \\
BC & $5.6 {\scriptstyle \pm 105.9}$ & $0.0 {\scriptstyle \pm 0.0}$ & $14.0 {\scriptstyle \pm 150.2}$ & $0.0 {\scriptstyle \pm 0.0}$ & $8.0 {\scriptstyle \pm 184.6}$ & $0.0 {\scriptstyle \pm 0.0}$ \\
GAIL & $15.4 {\scriptstyle \pm 102.1}$ & $0.0 {\scriptstyle \pm 0.0}$ & $85.8 {\scriptstyle \pm 132.9}$ & $0.0 {\scriptstyle \pm 0.0}$ & $58.2 {\scriptstyle \pm 174.0}$ & $0.0 {\scriptstyle \pm 0.0}$ \\
SQIL & $1.3 {\scriptstyle \pm 107.5}$ & $0.0 {\scriptstyle \pm 0.0}$ & $0.0 {\scriptstyle \pm 151.5}$ & $0.0 {\scriptstyle \pm 0.0}$ & $1.1 {\scriptstyle \pm 186.3}$ & $0.0 {\scriptstyle \pm 0.0}$ \\
IQL & $0.0 {\scriptstyle \pm 108.1}$ & $0.0 {\scriptstyle \pm 0.0}$ & $16.1 {\scriptstyle \pm 153.5}$ & $0.0 {\scriptstyle \pm 0.0}$ & $0.0 {\scriptstyle \pm 186.6}$ & $0.0 {\scriptstyle \pm 0.0}$ \\
\hline
\end{tabular}
}
\caption{AntMaze-Large}
\label{tab:ablation-horizon-large}
\end{subtable}
\caption{Ablation over horizon length $H \in \{500, 1000, 1500\}$ on Confounded AntMaze, at the default $100\%$ confounding strength. Best result per column in \textbf{bold}. Normalized $\mathbb{E}[Y]$ is not comparable across horizons (see text).}
\label{tab:ablation-horizon}
\end{table}

The prevalence of scalable causal algorithms over non-causal and less scalable algorithms is stable across horizons. On AntMaze-Medium, every causal method attains $64$--$91\%$ success at all three horizons while every non-causal method remains near $0\%$, with a single exception discussed below, and on AntMaze-Large the same ordering holds at every $H$ under normalized $\mathbb{E}[Y]$. At $H=500$ on AntMaze-Large, success rates saturate near zero for all methods including the causal ones; this is likely due to successful expert traversals requiring $648 \pm 99$ steps (Table~\ref{tab:episode-lengths}). Normalized $\mathbb{E}[Y]$ is more informative in this case, with causal imitators' higher values indicating that they progress further toward the goal than non-causal imitators even under an infeasible time budget. The clearest departure from the overall pattern is non-causal GAIL at $H=1500$ on Medium ($26.5\%$). Two factors plausibly contribute. First, GAIL's training budget is specified in rounds and episodes (Appendix~\ref{app:hparams}), so its total environment interaction scales linearly with $H$ and is largest in this configuration, whereas SQIL and IQ-Learn are capped at $2 \times 10^{6}$ steps throughout. Second, GAIL is the only on-policy algorithm among the non-causal baselines, and on-policy divergence minimization has been shown to partially recover from errors induced by unobserved contexts where off-policy conditional fitting latches onto them \citep{swamy2022sequence}. Even so, this partial recovery remains far below every causal method at the same horizon ($80.2$--$87.2\%$).

\begin{table}[H]
\centering
\begin{subtable}{\textwidth}
\centering
\resizebox{\textwidth}{!}{
\begin{tabular}{l|cc|cc|cc|cc}
\hline
& \multicolumn{2}{c}{$15\%$} & \multicolumn{2}{c}{$50\%$} & \multicolumn{2}{c}{$85\%$} & \multicolumn{2}{c}{$100\%$} \\
\cline{2-3} \cline{4-5} \cline{6-7} \cline{8-9}
\textbf{Algorithm} & $\mathbb{E}[Y]$ & SR\ (\%) & $\mathbb{E}[Y]$ & SR\ (\%) & $\mathbb{E}[Y]$ & SR\ (\%) & $\mathbb{E}[Y]$ & SR\ (\%) \\
\hline
C-BC & $59.1 {\scriptstyle \pm 99.5}$ & $80.6 {\scriptstyle \pm 1.2}$ & $246.8 {\scriptstyle \pm 98.8}$ & $81.1 {\scriptstyle \pm 1.2}$ & $248.5 {\scriptstyle \pm 99.3}$ & $81.6 {\scriptstyle \pm 1.2}$ & $250.1 {\scriptstyle \pm 100.9}$ & $77.9 {\scriptstyle \pm 1.3}$ \\
C-GAIL & $45.5 {\scriptstyle \pm 115.2}$ & $76.6 {\scriptstyle \pm 1.3}$ & $254.2 {\scriptstyle \pm 102.4}$ & $83.4 {\scriptstyle \pm 1.2}$ & $240.0 {\scriptstyle \pm 110.0}$ & $79.5 {\scriptstyle \pm 1.3}$ & $239.5 {\scriptstyle \pm 113.6}$ & $74.1 {\scriptstyle \pm 1.4}$ \\
C-SQIL & $\mathbf{73.2} {\scriptstyle \pm 97.9}$ & $\mathbf{86.5} {\scriptstyle \pm 1.1}$ & $\mathbf{267.2} {\scriptstyle \pm 90.1}$ & $\mathbf{87.8} {\scriptstyle \pm 1.0}$ & $256.3 {\scriptstyle \pm 99.4}$ & $82.6 {\scriptstyle \pm 1.2}$ & $\mathbf{275.8} {\scriptstyle \pm 85.7}$ & $\mathbf{90.7} {\scriptstyle \pm 0.9}$ \\
C-IQL & $67.5 {\scriptstyle \pm 102.3}$ & $82.5 {\scriptstyle \pm 1.2}$ & $256.1 {\scriptstyle \pm 102.2}$ & $83.4 {\scriptstyle \pm 1.2}$ & $\mathbf{257.8} {\scriptstyle \pm 97.8}$ & $\mathbf{84.6} {\scriptstyle \pm 1.1}$ & $257.1 {\scriptstyle \pm 101.3}$ & $84.3 {\scriptstyle \pm 1.1}$ \\
BC & $0.0 {\scriptstyle \pm 108.9}$ & $47.7 {\scriptstyle \pm 1.4}$ & $136.5 {\scriptstyle \pm 113.7}$ & $30.6 {\scriptstyle \pm 1.4}$ & $29.9 {\scriptstyle \pm 110.5}$ & $0.0 {\scriptstyle \pm 0.0}$ & $0.0 {\scriptstyle \pm 119.2}$ & $0.0 {\scriptstyle \pm 0.0}$ \\
GAIL & $61.1 {\scriptstyle \pm 104.3}$ & $80.5 {\scriptstyle \pm 1.2}$ & $152.9 {\scriptstyle \pm 128.5}$ & $40.0 {\scriptstyle \pm 1.5}$ & $150.3 {\scriptstyle \pm 130.9}$ & $38.4 {\scriptstyle \pm 1.4}$ & $10.4 {\scriptstyle \pm 112.2}$ & $0.0 {\scriptstyle \pm 0.0}$ \\
SQIL & $62.6 {\scriptstyle \pm 102.8}$ & $80.2 {\scriptstyle \pm 1.3}$ & $0.0 {\scriptstyle \pm 118.3}$ & $0.0 {\scriptstyle \pm 0.0}$ & $0.0 {\scriptstyle \pm 118.8}$ & $0.0 {\scriptstyle \pm 0.0}$ & $4.7 {\scriptstyle \pm 119.1}$ & $0.0 {\scriptstyle \pm 0.0}$ \\
IQL & $64.8 {\scriptstyle \pm 99.1}$ & $80.7 {\scriptstyle \pm 1.2}$ & $1.0 {\scriptstyle \pm 118.2}$ & $0.0 {\scriptstyle \pm 0.0}$ & $1.8 {\scriptstyle \pm 118.1}$ & $0.0 {\scriptstyle \pm 0.0}$ & $2.8 {\scriptstyle \pm 120.8}$ & $0.0 {\scriptstyle \pm 0.0}$ \\
\hline
\end{tabular}
}
\caption{AntMaze-Medium}
\label{tab:ablation-confound-medium}
\end{subtable}
\vspace{0.75em}
\begin{subtable}{\textwidth}
\centering
\resizebox{\textwidth}{!}{
\begin{tabular}{l|cc|cc|cc|cc}
\hline
& \multicolumn{2}{c}{$15\%$} & \multicolumn{2}{c}{$50\%$} & \multicolumn{2}{c}{$85\%$} & \multicolumn{2}{c}{$100\%$} \\
\cline{2-3} \cline{4-5} \cline{6-7} \cline{8-9}
\textbf{Algorithm} & $\mathbb{E}[Y]$ & SR\ (\%) & $\mathbb{E}[Y]$ & SR\ (\%) & $\mathbb{E}[Y]$ & SR\ (\%) & $\mathbb{E}[Y]$ & SR\ (\%) \\
\hline
C-BC & $43.7 {\scriptstyle \pm 128.9}$ & $41.5 {\scriptstyle \pm 1.5}$ & $189.1 {\scriptstyle \pm 127.3}$ & $39.9 {\scriptstyle \pm 1.5}$ & $191.2 {\scriptstyle \pm 128.5}$ & $40.6 {\scriptstyle \pm 1.5}$ & $199.3 {\scriptstyle \pm 121.6}$ & $39.8 {\scriptstyle \pm 1.5}$ \\
C-GAIL & $2.9 {\scriptstyle \pm 129.6}$ & $15.5 {\scriptstyle \pm 0.8}$ & $154.2 {\scriptstyle \pm 138.8}$ & $25.9 {\scriptstyle \pm 1.2}$ & $142.5 {\scriptstyle \pm 129.3}$ & $10.1 {\scriptstyle \pm 0.9}$ & $155.6 {\scriptstyle \pm 122.6}$ & $7.3 {\scriptstyle \pm 0.8}$ \\
C-SQIL & $\mathbf{57.3} {\scriptstyle \pm 131.0}$ & $\mathbf{51.8} {\scriptstyle \pm 1.6}$ & $207.8 {\scriptstyle \pm 136.1}$ & $\mathbf{54.6} {\scriptstyle \pm 1.6}$ & $201.7 {\scriptstyle \pm 127.8}$ & $49.0 {\scriptstyle \pm 1.6}$ & $204.3 {\scriptstyle \pm 125.9}$ & $45.0 {\scriptstyle \pm 1.6}$ \\
C-IQL & $52.0 {\scriptstyle \pm 131.0}$ & $46.8 {\scriptstyle \pm 1.6}$ & $\mathbf{212.3} {\scriptstyle \pm 131.7}$ & $52.7 {\scriptstyle \pm 1.5}$ & $\mathbf{204.1} {\scriptstyle \pm 129.8}$ & $\mathbf{51.4} {\scriptstyle \pm 1.6}$ & $\mathbf{225.6} {\scriptstyle \pm 127.9}$ & $\mathbf{58.9} {\scriptstyle \pm 1.6}$ \\
BC & $13.1 {\scriptstyle \pm 122.9}$ & $18.4 {\scriptstyle \pm 1.2}$ & $79.0 {\scriptstyle \pm 134.3}$ & $0.0 {\scriptstyle \pm 0.0}$ & $14.9 {\scriptstyle \pm 148.2}$ & $0.0 {\scriptstyle \pm 0.0}$ & $14.0 {\scriptstyle \pm 150.2}$ & $0.0 {\scriptstyle \pm 0.0}$ \\
GAIL & $0.0 {\scriptstyle \pm 127.6}$ & $11.9 {\scriptstyle \pm 1.0}$ & $120.0 {\scriptstyle \pm 124.0}$ & $6.5 {\scriptstyle \pm 0.4}$ & $12.5 {\scriptstyle \pm 148.2}$ & $0.0 {\scriptstyle \pm 0.0}$ & $85.8 {\scriptstyle \pm 132.9}$ & $0.0 {\scriptstyle \pm 0.0}$ \\
SQIL & $23.7 {\scriptstyle \pm 121.3}$ & $14.9 {\scriptstyle \pm 1.0}$ & $1.7 {\scriptstyle \pm 151.1}$ & $0.0 {\scriptstyle \pm 0.0}$ & $0.0 {\scriptstyle \pm 152.1}$ & $0.0 {\scriptstyle \pm 0.0}$ & $0.0 {\scriptstyle \pm 151.5}$ & $0.0 {\scriptstyle \pm 0.0}$ \\
IQL & $25.0 {\scriptstyle \pm 121.3}$ & $21.2 {\scriptstyle \pm 1.0}$ & $0.0 {\scriptstyle \pm 151.8}$ & $0.0 {\scriptstyle \pm 0.0}$ & $0.0 {\scriptstyle \pm 151.9}$ & $0.0 {\scriptstyle \pm 0.0}$ & $16.1 {\scriptstyle \pm 153.5}$ & $0.0 {\scriptstyle \pm 0.0}$ \\
\hline
\end{tabular}
}
\caption{AntMaze-Large}
\label{tab:ablation-confound-large}
\end{subtable}
\caption{Ablation over confounding strength $\in \{15\%, 50\%, 85\%, 100\%\}$ on Confounded AntMaze, at the default horizon $H=1000$. Best result per column in \textbf{bold}. Normalized $\mathbb{E}[Y]$ is not comparable across strengths (see text).}
\label{tab:ablation-confound}
\end{table}

The confounding-strength sweep isolates the mechanism behind the causal methods' advantage. Because causal variants exclude the compass $\mathbf{W}$ from their adjustment sets, and the strength parameter scales only the compass corruption, the learning problem faced by a causal imitator is identical at every strength; the causal rows (e.g., Causal IQ-Learn on Large: $46.8$--$58.9\%$) therefore provide a scale against which the behavior of non-causal methods can be assessed (e.g. the higher variation of Causal GAIL on Large, $7.3$--$25.9\%$, is consistent with its training instability at scale discussed in Section~\ref{sec:algos}). Their smaller normalized $\mathbb{E}[Y]$ values in the $15\%$ rows reflect the improved worst-performing anchor at low strength rather than any degradation of the causal methods, whose success rates remain consistent. Non-causal methods, conversely, degrade monotonically as strength increases, and they collapse in a consistent order: SQIL and IQ-Learn reach $0\%$ success by $50\%$ strength on both tasks, BC follows ($0\%$ by $85\%$ on Medium and by $50\%$ on Large), and GAIL persists longest ($38.4\%$ at $85\%$ strength on Medium; $6.5\%$ at $50\%$ on Large). This phenomenon implies how directly each objective commits to the confounded conditional: BC fits $P(x_t \mid \mathbf{v}^O_t)$ pointwise; SQIL and IQ-Learn additionally propagate it through bootstrapped value targets across the horizon, amplifying the corrupted association; and GAIL, the only on-policy method, matches occupancy measures under its own rollouts rather than fitting the conditional directly, consistent with the greater robustness of on-policy divergence minimization to unobserved contexts \citep{swamy2022sequence}, although GAIL's larger interaction budget (Appendix~\ref{app:hparams}) cannot be fully separated from this effect.

At the lowest strength, conditioning on $\mathbf{W}$ becomes nearly costless and sometimes net-beneficial: $15\%$ confounding leaves the compass a mostly faithful heading sensor even after the runtime shift, and non-causal GAIL ($80.5\%$) exceeds Causal GAIL ($76.6\%$) on Medium while non-causal SQIL and IQ-Learn recover to within a few points of their causal counterparts ($80.2$ against $86.5\%$ and $80.7$ against $82.5\%$). This substantiates the observation of Figure~\ref{fig:training-curves-main} that the spurious proxy carries genuine predictive signal, and is consistent with the discussion in Appendix~\ref{app:related-work}: the bias resides in the conditional over the imitator's observations, and its cost is governed by how far that conditional moves at runtime. Excluding $\mathbf{W}$ sacrifices this signal for invariance, a trade that is repaid as strength grows. The favorable regime for non-causal methods does not extend to the harder task, however: on Large at $15\%$, Causal SQIL ($51.8\%$) still more than doubles the best non-causal result ($21.2\%$), as even mild per-step miscalibration compounds over the longer traversal. Relatedly, non-causal BC is anomalously weak at $15\%$ on Medium ($47.7\%$ against roughly $80\%$ for the other non-causal methods): BC is the only algorithm that never interacts with the environment during training (Appendix~\ref{app:hparams}) and thus has no mechanism to recover once the mildly miscalibrated compass drifts it off the demonstration support, an interaction between confounding bias and compounding error, the two causes of distribution shift discussed in Appendix~\ref{app:related-work}. Finally, the low-strength rows provide further evidence for non-causal methods' comparatively better HumanoidMaze results being an effect of weaker confounding influence.

\section{Limitations}
\label{sec:limitations}

\paragraph{Known causal graph.}
Our framework assumes that the causal diagram $\mathcal{G}$ is specified \textit{a priori}. In practice, the graph must be elicited from domain expertise or estimated from data. Misspecification of $\mathcal{G}$, such as missing an edge from a confounder to the action, can lead to invalid adjustment sets and biased policies. Integrating causal discovery from observational data, either as a preprocessing step or jointly with policy learning, remains an important open problem.

\paragraph{Bounded temporal influence of confounders.}
The windowed approximation in Algorithm~\ref{alg:feasible-pi-backdoor} assumes that the influence of any single confounder realization decays within $k$ timesteps. This is justified by the MuJoCo physics of the environments we consider, but environments with long-range latent dependencies (e.g. persistent hidden goals, slowly drifting dynamics, or latent agent intentions in multi-agent settings) would violate this assumption. In such cases, alternative approximation strategies (e.g., hierarchical windows or attention-based aggregation over history) or exact methods on compressed representations would be required.

\paragraph{No exploitation of reward structure.}
Our algorithms treat the reward as entirely latent and make no parametric assumptions about its form. Prior work on partial identification for CIL \citep{ruan2024partial} has shown that incorporating reward priors can tighten bounds on the imitating policy and even enable the imitator to surpass expert performance. Combining Causal IQ-Learn with such reward priors is a natural extension that we leave to future work.

\paragraph{Sensitivity to expert optimality.}
While Causal SQIL and Causal IQ-Learn successfully address the compounding error and credit assignment issues seen in Causal BC and Causal GAIL, they are more sensitive than BC to the quality and optimality of expert demonstrations. Because these off-policy algorithms utilize demonstrations to ground an implicit reward signal or a value function via soft Q-learning objectives, they are more susceptible to noise in high-dimensional state-action spaces where the expert signal may be weak. This dependency is clearly demonstrated in the HumanoidMaze-Large task (Table~\ref{tab:main-results}), where the expert itself is sub-optimal, achieving a success rate of only $7.0\%$. While Causal BC maintains a success rate of $8.0\%$, Causal IQ-Learn collapses to $3.0\%$ and Causal SQIL matches at $8.0\%$ despite both algorithms outperforming BC in other tasks where the expert is more consistent. These results suggest that when the expert is sub-optimal, the off-policy grounding of value functions is more easily corrupted than the supervised cloning objective, indicating that the scalability benefits of our proposed methods are most reliably realized when provided with a high-quality expert signal.

\end{document}